\documentclass[11pt]{article}

\usepackage[preprint]{acl}
\usepackage{adjustbox}
\usepackage[table]{xcolor}
\usepackage{times}
\usepackage{latexsym}
\usepackage[T1]{fontenc}

\usepackage[utf8]{inputenc}

\usepackage{microtype}

\usepackage{inconsolata}
\usepackage{algorithm}
\usepackage{algorithmic}

\usepackage{booktabs}     
\usepackage{tabularx}     
\usepackage{array}        
\usepackage{ragged2e}     
\usepackage{makecell}     

\usepackage{xcolor}       
\definecolor{myblue}{RGB}{85,38,121} 
\usepackage{subcaption}

\usepackage[most]{tcolorbox} 

\usepackage{lipsum} 
\usepackage{enumitem} 

\usepackage{multirow}   
\usepackage{bm}

\usepackage{fix-cm}
\usepackage{tabularx}
\usepackage{array}

\newcolumntype{Y}{>{\centering\arraybackslash}X}
\newcolumntype{M}[1]{>{\centering\arraybackslash}p{#1}}
\newcolumntype{L}[1]{>{\raggedright\arraybackslash}p{#1}}
\usepackage{amsthm}
\usepackage{graphicx}
\usepackage{amssymb}
\newtheorem{lemma}{Lemma}
\newtheorem{theorem}{Theorem}

\newtheorem{proposition}{Proposition}
\newtheorem{corollary}{Corollary}

\definecolor{metablue}{RGB}{231,243,255}
\newcommand{\best}[1]{\textcolor{red}{#1}}
\newcommand{\second}[1]{\underline{#1}}
\usepackage{amsmath}
\usepackage{mathtools}

\title{Revisiting Ripple Effects in Knowledge Editing through Pressure-Aware Joint Neighborhood Optimization}

\author{
  \textbf{Haoben Huang\textsuperscript{1,2,$\dagger$,$\ddagger$}},
  \textbf{Shuxin Liu\textsuperscript{1,$\dagger$}},
  \textbf{Ou Wu\textsuperscript{1,$\ast$}},
  \textbf{Di Gao\textsuperscript{1}}
\\
\\
  \textsuperscript{1}Hangzhou Institute for Advanced Study, \\University of Chinese Academy of Sciences, Hangzhou, China,\\
  \textsuperscript{2}College of Information Engineering, Zhejiang University of Technology, Hangzhou, China\\
  \vspace{0.1cm}
  \texttt{liushuxin25@ucas.ac.cn},
\texttt{wuou@ucas.ac.cn}
}

\begin{document}

\maketitle

\maketitle
\begingroup
\renewcommand{\thefootnote}{\fnsymbol{footnote}}
\footnotetext[2]{Equal contributions.}
\footnotetext[1]{Corresponding Author.}
\footnotetext[3]{Work done at HIAS, UCAS.}

\endgroup

\begin{abstract}
Single-edit updates in large language models can trigger ripple effects across local knowledge neighborhoods: desirable propagation to related facts and unintended perturbation of preserved ones. Existing methods address these two effects separately, without explicitly modeling their coupling. We challenge this separation through an analysis of ripple responses across typical baselines, identifying two coupled design pressures: editable-side coordination and preserved-side leakage. We propose \textbf{Joint Neighborhood Optimization (JNO)}, a new knowledge-editing framework to formalize and jointly address both pressures at the target-planning stage. JNO instantiates this principle through \textbf{Pressure-Aware Coordination (PAC)}, which jointly optimizes neighborhood target representations under coupled constraints, and a semantic pre-execution gate that rejects high-risk target plans before parameter execution. Experiments on RippleEdits show JNO improves propagation and preservation metrics by at least 7.0\% while preserving cross-backbone editing stability.
Code at: \url{https://anonymous.4open.science/r/JNO-F4CB}.
\end{abstract}

\section{Introduction}
Large Language Models (LLMs) encode factual knowledge in parameters~\cite{yao-etal-2023-editing,petroni2019language,roberts2020much,rogers2020primer}. 
As knowledge ages, full retraining is often prohibitively expensive~\cite{strubell2019energy,kaplan2020scaling,hoffmann2022training}. 
Knowledge editing (KE) addresses this by revising specific facts without costly retraining~\cite{de2021editing,mitchellfast}. 
Locate-then-edit methods, such as ROME~\cite{meng2022locating} and MEMIT~\cite{mengmass}, identify factual memories in feed-forward layers and modify local weights. 
However, real-world edits are rarely isolated: as illustrated in Fig.~\ref{fig:example}, modifying one fact may require updating logically related neighbors (e.g., team-dependent attributes) and preserving co-located but independent facts (e.g., subject's nationality), inducing \emph{ripple effects} over its local knowledge neighborhood~\cite{cohen2024evaluating}.

\begin{figure}[t]
    \centering
    \vspace{-0.15in}
\includegraphics[width=1.0\linewidth]{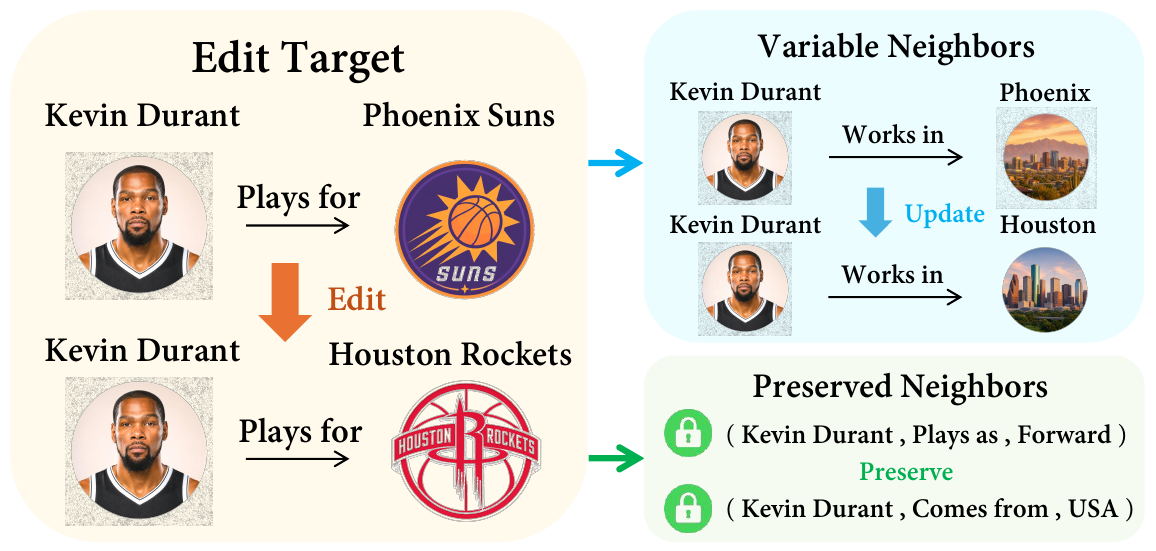}
    \vspace{-0.25in}
    \caption{Ripple effects in neighborhood-aware KE.
    }
    \label{fig:example}
    \vspace{-0.25in}
\end{figure}

These effects can be divided into \emph{positive ripple effects}, where variable facts should propagate consistently with the edit, and \emph{negative ripple effects}, where preserved ones should remain unchanged but are unintentionally perturbed~\cite{wang2025missing}. Existing methods address these effects through two largely separate lines. Propagation-oriented methods promote desirable propagation to related facts through knowledge graphs, memory association, or logical rules~\cite{zhang2024knowledge,park2025make,dong2025chainedit}, while preservation-oriented methods suppress harmful perturbation via neighboring-answer interference, same-subject preservation, knowledge conflict, and hidden-space ripple~\cite{ma2024neighboring,zhang2026disentangling,liunveiling,wang2024efficiently,duan2025related,dong2025memit}. 
However, this separation leaves a coupling unmodeled: propagation and preservation are two responses of the parameter update, so their targets cannot always be planned independently.

To this end, we revisit ripple effects through a quantitative analysis of how ripple responses depend on local neighborhood structure across edits. 
Our analysis identifies three factors and distills their patterns into two coupled pressures. Editable-side coordination pressure arises because related neighbors should move consistently, yet overly coupled or semantically mismatched targets can be hard to realize jointly. Preserved-side leakage pressure captures that dense clustering or edit proximity makes modest target shifts corrupt preserved facts.

To address these pressures, we propose \textbf{Joint Neighborhood Optimization (JNO)}, a pressure-aware editing framework that constructs structured neighborhoods and applies \textbf{Pressure-Aware Coordination (PAC)} to jointly optimize target representations under a pressure-aware objective, using pairwise coordination weights to coordinate variable-neighbor representations and preserved-side leakage weights to control leakage over preserved ones. 
For neighborhoods whose optimized targets fail a semantic-loss check, a pre-execution gate abstains rather than committing an unsafe edit.

Our contributions are summarized as follows:

\begin{itemize}[leftmargin=*,nosep]
\item We revisit ripple effects through a quantitative analysis, revealing two pressures: editable-side coordination and preserved-side leakage.
\item We propose JNO, a pressure-aware editing framework that constructs structured neighborhoods and uses PAC to jointly optimize representations. We further introduce a semantic pre-execution gate abstaining from semantically failed edits.
\item We evaluate JNO on KE benchmarks across Qwen2.5-1.5B-Instruct, GPT-J, and LLaMA3. Results show that JNO improves propagation and preservation metrics on RippleEdits by at least 7.0\% while preserving backbone stability.
\end{itemize}

\section{Related Work}
\subsection{Neighborhood-aware Knowledge Editing}

A challenge in KE is updating target facts while controlling their effects on neighborhoods. Locate-and-edit methods such as ROME and MEMIT infer targets for edited facts~\cite{meng2022locating,mengmass}, but RippleEdits shows that target rewriting does not ensure consistent updates to related knowledge~\cite{cohen2024evaluating}. This motivates neighborhood-aware editing: GLAME, MAKE, ChainEdit, and PropMEND promote propagation through knowledge graphs, internal memory, logical rules, or meta-trained updates~\cite{zhang2024knowledge,park2025make,dong2025chainedit,liu2025propmend}; PEAK, SSS, and DiKE mitigate perturbation to neighboring answers, irrelevant neighborhood examples, or same-subject knowledge~\cite{ma2024neighboring,wang-etal-2024-sss,zhang2026disentangling}. JNO jointly infers executable value targets over neighborhoods containing both editable and preserved samples.

\subsection{Selective Execution and Gating}
Selective decision making has been studied through reject-option and selective-prediction frameworks, which abstain on uncertain inputs to trade coverage for lower risk~\cite{chow1957optimum,el2010foundations,geifman2017selective,geifman2019selectivenet}. However, such selectivity remains underexplored in KE. Most methods execute requested updates without checking semantic feasibility, even when edits are difficult or incompatible with preserved neighborhood knowledge. Since editing reliability depends on robust keys~\cite{yan2025keys}, facts differ in editability~\cite{bi-etal-2025-decoding,wang2025revealing}, and aggressive edits may degrade broader abilities~\cite{gu2024model}. JNO bridges this gap by rejecting high-risk parameter updates before execution.

\section{Quantitative Analysis of Ripple Effects}
\label{sec:revisiting}
Recent work links ripple effects to gradient similarity or representational entanglement~\cite{qin2024does,baser2026clare}, and a recent benchmark measures how model interventions affect performance across semantic-distance neighborhoods~\cite{rinberg2025ripplebench}. Yet they mainly characterize ripple responses, not why propagation and perturbation co-occur. We revisit ripple effects with a unified view: positive propagation and negative perturbation are two responses of the same neighborhood structure to a parameter update.

\begin{table}[t]
\centering
\vspace{-0.1in}
\scriptsize
\renewcommand{\arraystretch}{0.9}
\setlength{\tabcolsep}{2pt}
\resizebox{0.98\columnwidth}{!}{
\begin{tabular}{c|ccccc|c}
\toprule
\textbf{Metric} 
& \textbf{ROME} 
& \textbf{MEMIT} 
& \textbf{PMET} 
& \textbf{AlphaEdit} 
& \textbf{GLAME} 
& \textbf{Avg.} \\
\midrule
PRR & 60.47 & 58.17 & 58.93 & 66.83 & 68.23 & 62.53 \\
NRR & 65.92 & 65.85 & 66.28 & 71.19 & 54.46 & 64.74 \\
\bottomrule
\end{tabular}
}
\vspace{-0.1in}
\caption{PRR and NRR results~(\%) on GPT-J (6B).}
\label{tab:prr-nrr}
\vspace{-0.1in}
\end{table}
\subsection{Positive and Negative Ripples}
\label{subsec:ripple_magnitude}

For an edit request $\boldsymbol{e}=(s,r,o \rightarrow o^*)$ with subject $s$, relation $r$, original object $o$, and target object $o^*$, we construct a local neighborhood with two disjoint subsets: the variable set $\mathcal{C}_c$ contains facts $(s_i,r_i,o_i \rightarrow o_i^*)$ to be updated consistently with $\boldsymbol{e}$, while the preserved set $\mathcal{C}_s$ contains facts $(s_j,r_j,o_j)$ to remain unchanged. Given original and edited models $f$ and $f^{+}$, let $\chi_i^{+}=\mathbf{1}\{f^{+}(s_i,r_i)=o_i^{*}\}$ and $\chi_j^{-}=\mathbf{1}\{f(s_j,r_j)=o_j \wedge f^{+}(s_j,r_j)\neq o_j\}$. We quantify desirable propagation by the Positive Ripple Rate (PRR) and harmful perturbation by the Negative Ripple Rate (NRR):
\vspace{-0.35em}
\begin{equation}
\small
\mathrm{PRR}(\boldsymbol{e})
=
\frac{1}{|\mathcal{C}_c|}
\sum _{i\in\mathcal{C}_c}\chi_i^{+},
\quad\mathrm{NRR}(\boldsymbol{e})
=
\frac{1}{|\mathcal{C}_s|}
\sum_{j\in\mathcal{C}_s}\chi_j^{-}.
\label{eq:ripple_rate}
\vspace{-0.35em}
\end{equation}

{Table~\ref{tab:prr-nrr} shows that across typical baselines on RippleEdits, average PRR remains around $62.53\%$ while average NRR reaches $64.74\%$} (details in Appendix~\ref{subsec:app_datasets}). Neither propagation nor preservation is reliable, which motivates examining which local structures co-vary with such failures.

\subsection{Factors Behind Ripple Variation}
\label{subsec:ripple_factors}
To examine why ripple responses vary across edits, we group requests by three structural probes of the local neighborhood: \emph{key-space coupling} ($d_k$), measuring edited-layer proximity between the edit and its neighbors; \emph{pre-trained entanglement} ($\eta$), measuring parameter-space co-movement from pre-training; and \emph{local neighborhood sensitivity} ($\lambda$), measuring preserved-region density. Since probes are measured on naturally occurring neighborhoods without randomized interventions, we interpret patterns as correlational; details appear in Appendix~\ref{subsec:app_metrics} and~\ref{subsec:app_stratified_protocol}. 

Fig.~\ref{fig:ripple-factors} stratifies edits by probe and reports PRR and NRR for five editing methods. Three asymmetric patterns emerge. (i) PRR is non-monotonic with editable-side key proximity: moderate
proximity tracks higher propagation, whereas strong coupling can lower
PRR when neighboring targets differ. By contrast, NRR increases with
preserved-side proximity, indicating more frequent negative ripple
effects under stronger coupling. (ii) $\eta$ is positively associated with PRR for most methods, indicating that pre-trained co-movement can support propagation to related facts, while its association with NRR is weaker. (iii) $\lambda$ gives the clearest preservation signal: denser preserved regions track higher NRR without a comparable PRR gain.

These patterns reveal two separation modes: editable-side target incompatibility and preserved-side leakage. Notably, key proximity shows distinct associations on the editable
and preserved sides.

\subsection{Design Pressures from Ripple Diagnostics}
\label{sec:design-principles}
The above diagnostics motivate two pre-execution pressures on neighborhood target planning. They are not treated as causal laws, but as design pressures that a neighborhood-aware editor should account for before a parameter update.

\textbf{Editable-side coordination pressure.}
Variable facts should move consistently with the main edit, but their targets cannot be planned independently when keys are strongly coupled or pre-trained entanglement induces co-movement. The non-monotonic PRR over $d_k$ indicates that coupling helps only when targets are compatible; otherwise it creates coordination tension. The association between \(\eta\) and PRR further indicates that such co-movement is not merely geometric, but also reflected in the model's learned parameterization. This motivates the first design requirement:
\emph{coupled editable targets should avoid excessive separation while preserving distinct target semantics.}

\textbf{Preserved-side leakage pressure.}
Preserved facts should remain unchanged, yet denser or more aligned
preserved regions show more negative ripples. The monotone NRR trend
over \(\lambda\) indicates greater disruption in dense preserved
neighborhoods, while the residual association between \(d_k\) and NRR
suggests greater disruption under stronger directional alignment. This motivates the second design requirement:
\emph{editable residuals should be more tightly budgeted when their keys are more exposed to preserved keys.}

\begin{figure}[t]
    \centering
\includegraphics[width=\linewidth]{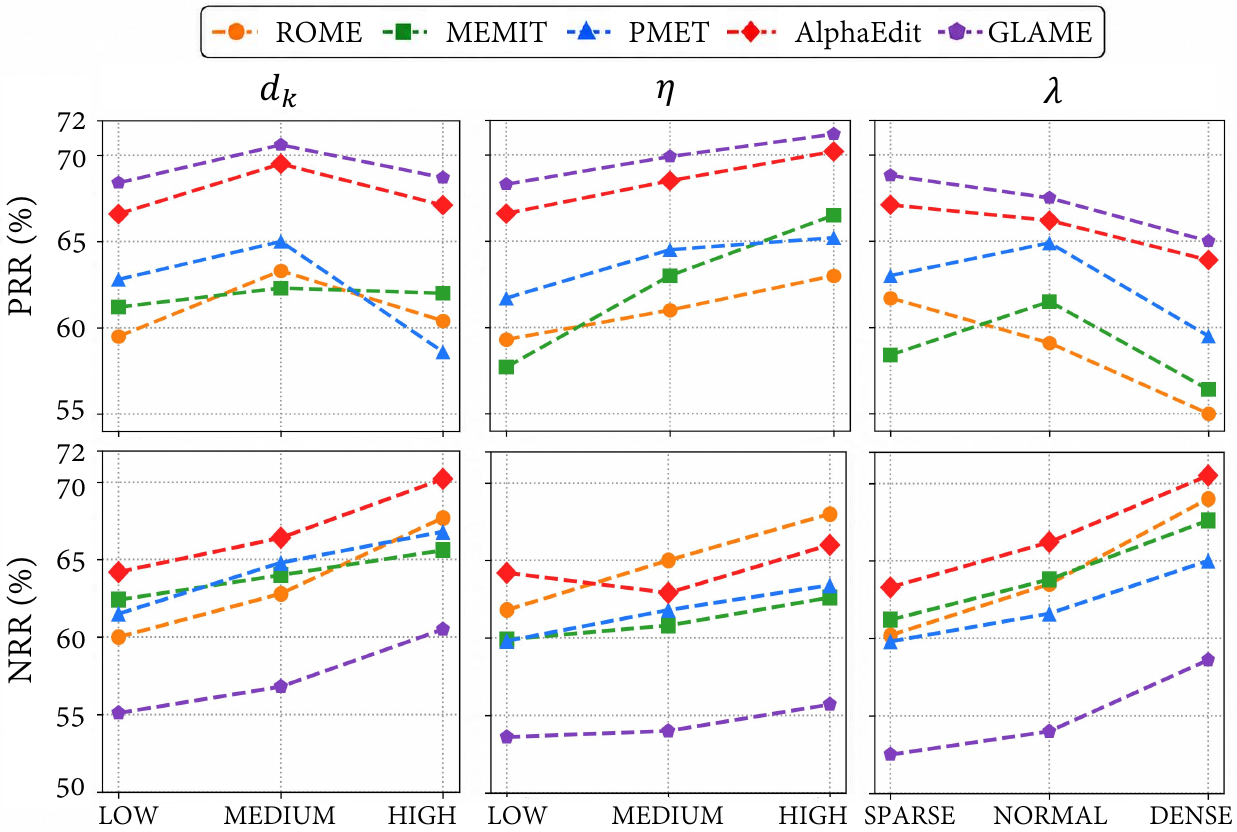}
    \vspace{-0.25in}
\caption{PRR and NRR stratified by structural probes ($d_k$, $\eta$, $\lambda$) across five editing methods on GPT-J (6B).}
    \label{fig:ripple-factors}
    \vspace{-0.3in}
\end{figure}

These requirements connect diagnostics to PAC. On the editable side, \(d_k\) diagnoses proximity over edited-layer keys \(\boldsymbol{k}_i\); JNO instantiates this coupling as weights \(s_{ij}\), yielding a target-separation penalty. The diagnostic probe uses cosine proximity for key-norm-invariant stratification, whereas PAC uses an inverse squared-distance kernel for weighting. For unit-normalized keys, \(\|\tilde{\boldsymbol{k}}_i-\tilde{\boldsymbol{k}}_j\|_2^2=2(1-\cos(\boldsymbol{k}_i,\boldsymbol{k}_j))\), where \(\tilde{\boldsymbol{k}}_i=\boldsymbol{k}_i/\|\boldsymbol{k}_i\|_2\). Thus, \(d_k\) and \(s_{ij}\) are diagnostic and optimization views of the same coupling pressure. The entanglement probe \(\eta\) validates co-movement but is not separately regularized because it depends on the layer, objective, and editor. On the preserved side, \(\lambda\) diagnoses neighborhood-level fragility, while residual budgeting requires per-sample exposure. JNO uses the upper-tail alignment score \(d_i\) to capture the preserved direction most exposed to spillover. Thus, \(s_{ij}\) and \(d_i\) are pressure-specific actionable surrogates derived from diagnostics, not replacements.

\begin{figure*}[t]
	\centering
\includegraphics[width=0.98\linewidth]{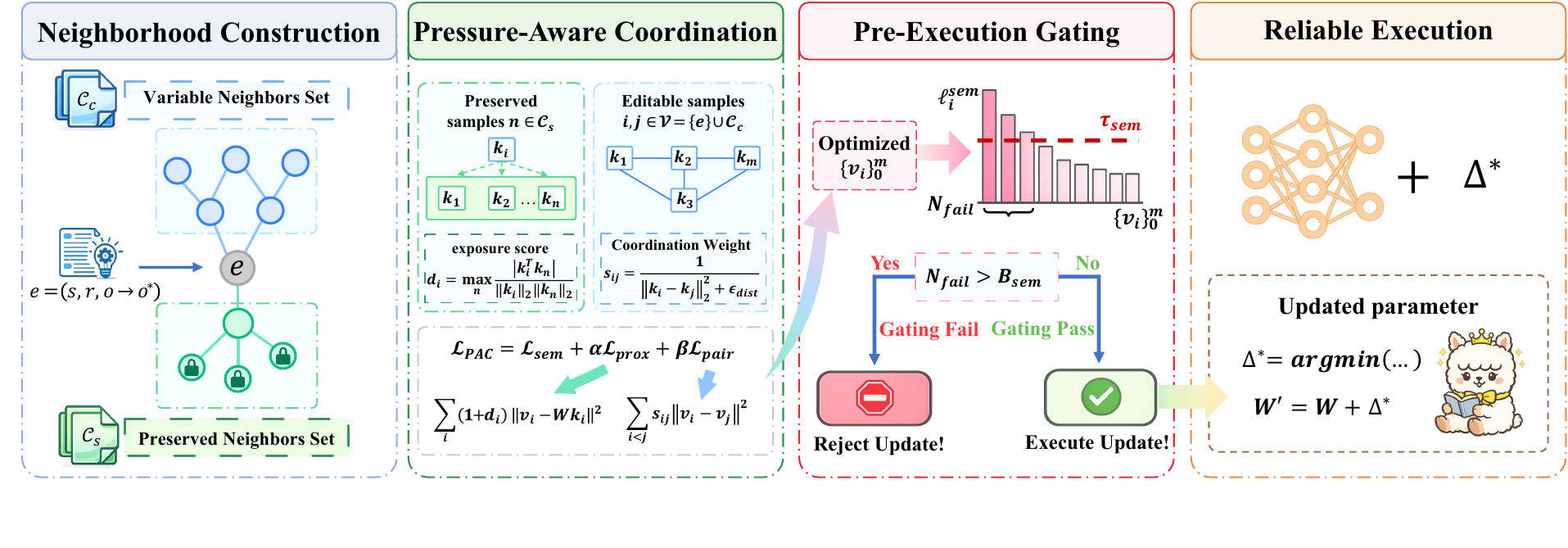}\vspace{-0.3in}
\caption{Overall framework of Joint Neighborhood Optimization (JNO).}\vspace{-0.2in}
 \label{fig:framework}
\end{figure*}

\section{Methodology}
\label{sec:method}
Guided by these two coupled pressures, JNO plans neighborhood targets by balancing editable-side compatibility and preserved-side exposure. As shown in Fig.~\ref{fig:framework}, JNO proceeds in four stages:
\begin{enumerate}[leftmargin=*,nosep]
    \item \emph{Structured Neighborhood Construction}: assemble the local neighborhood $\mathcal{C}=\{\boldsymbol{e}\}\cup\mathcal{C}_c\cup\mathcal{C}_s$.
    \item \emph{Pressure-Aware Coordination (PAC)}: jointly optimize target representations under a pressure-aware objective.
    \item \emph{Semantic Pre-execution Gating}: reject high-risk parametric updates before parameter execution.
    \item \emph{Reliable Execution}: instantiate gate-passed updates through a preservation-aware local update.
\end{enumerate}

This separates neighborhood-level target planning from parameter execution, in contrast to editing pipelines that commit a weight update directly after optimizing a single target fact.

\subsection{Structured Neighborhood Construction}
\label{sec:neighbor_construction}

For each edit request $\boldsymbol{e}$, JNO constructs $\mathcal{C}=\{\boldsymbol{e}\}\cup\mathcal{C}_c\cup\mathcal{C}_s$, where $\mathcal{C}_c$ contains editable consistency samples expected to co-update and $\mathcal{C}_s$ contains preserved specificity samples to remain unchanged. The editable subset is $\mathcal{V}=\{\boldsymbol{e}\}\cup\mathcal{C}_c$ with $|\mathcal{V}|=m+1$ and $\boldsymbol{e}$ indexed as $0$. Each neighborhood pairs the target edit $\boldsymbol{e}$ with variable neighbors generated through logical-rule-guided chained reasoning and preserved neighbors extracted from locality-related but non-target relations. Rule mining and prompting are detailed in Appendix~\ref{sec:neighborhood_construction}.

\subsection{Local Formalization of the Two Pressures}
\label{sec:motivation}
Under the local linear view, we derive two design regularizers for neighborhood target planning: editable-key compatibility and preserved-key residual budgeting.  As mapped in Sec.~\ref{sec:design-principles}, editable-side coordination is implemented through pairwise weights \(s_{ij}\), which instantiate key-space coupling via a distance-kernel form, while preserved-side leakage uses an exposure score \(d_i\) that captures the maximum alignment with preserved keys.

For each editable sample \(i\in\mathcal{V}\), let \(\boldsymbol{k}_i\), \(\boldsymbol{v}_i\), and \(\boldsymbol{r}_i=\boldsymbol{v}_i-\boldsymbol{W}\boldsymbol{k}_i\) denote its edited-layer key, planned target, and induced residual under the current weight \(\boldsymbol{W}\). Stacking them gives \(\boldsymbol{K}_{\mathcal{V}}=[\boldsymbol{k}_0,\dots,\boldsymbol{k}_m]\) and \(\boldsymbol{V}_{\mathcal{V}}=[\boldsymbol{v}_0,\dots,\boldsymbol{v}_m]\). For editable-side coupling, we define the coordination weight
\vspace{-0.7em}
\begin{equation}
s_{ij}=s_{ji}=\frac{1}{\|\boldsymbol{k}_i-\boldsymbol{k}_j\|_2^2+\epsilon_{\mathrm{dist}}},
\label{eq:sij}
\vspace{-0.3em}
\end{equation}
and let \(\boldsymbol{L}_s\) be the unnormalized graph Laplacian induced by \(s_{ij}\) for \(i\ne j\) with \(s_{ii}=0\), so \(\boldsymbol{L}_s\succeq 0\).

\begin{theorem}[Target-space compatibility under a shared local map]
\label{thm:editable_tension}
Suppose there exists a local update $\boldsymbol{\Delta}$ such that $(\boldsymbol{W}+\boldsymbol{\Delta})\boldsymbol{K}_{\mathcal{V}}=\boldsymbol{V}_{\mathcal{V}}$. Then
\vspace{-0.31em}
\begin{equation}
\small
\operatorname{tr}\!\big(\boldsymbol{V}_{\mathcal{V}}\boldsymbol{L}_s\boldsymbol{V}_{\mathcal{V}}^T\big)
\le
\|\boldsymbol{W}+\boldsymbol{\Delta}\|_{\mathrm{op}}^2
\operatorname{tr}\!\big(\boldsymbol{K}_{\mathcal{V}}\boldsymbol{L}_s\boldsymbol{K}_{\mathcal{V}}^T\big).
\label{eq:graph_executability_bound}
\vspace{-0.31em}
\end{equation}
\end{theorem}

{Theorem~\ref{thm:editable_tension} formalizes editable-side coordination
pressure: large target separation over tightly coupled editable keys
requires a larger shared local map. PAC relaxes this requirement through
the differentiable pairwise term
$\mathcal{L}_{\mathrm{pair}}
=\operatorname{tr}(\boldsymbol{V}_{\mathcal{V}}
\boldsymbol{L}_s\boldsymbol{V}_{\mathcal{V}}^T)$,
which acts as a soft executability regularizer rather than a hard
clustering constraint. Together, the pairwise and semantic terms balance
coordination against target-semantic satisfaction.} {
Appendix~\ref{subsec:proof_editable_tension} proves the theorem and
derives a residual-level compatibility bound, which is examined on
realized updates in Appendix~\ref{app:realized_update}.}

For preserved-side exposure, fix an editable $i\in\mathcal{V}$ and a preserved one $n\in\mathcal{C}_s$ with key $\boldsymbol{k}_n$. Split
$\boldsymbol{k}_n=a_{ni}\boldsymbol{k}_i+\boldsymbol{e}_{ni},
\boldsymbol{e}_{ni}\perp \boldsymbol{k}_i,
a_{ni}=\frac{\boldsymbol{k}_n^T\boldsymbol{k}_i}{\|\boldsymbol{k}_i\|_2^2}$.

\begin{proposition}[Preserved-side exposure and residual budget]
\label{prop:preserved_leakage}
Suppose a local update $\boldsymbol{\Delta}$ satisfies $\boldsymbol{\Delta}\boldsymbol{k}_i=\boldsymbol{r}_i$, $\|\boldsymbol{\Delta}\|_{\mathrm{op}}\le B$, and $\|\boldsymbol{\Delta}\boldsymbol{k}_n\|_2\le \tau_n$. For any $n$ with $a_{ni}\neq 0$,
$\|\boldsymbol{r}_i\|_2 \le \frac{\tau_n+B\|\boldsymbol{e}_{ni}\|_2}{|a_{ni}|}$.
Define the preserved-side exposure score
\vspace{-0.5em}
\begin{equation}
\small
d_i=\max_{n\in\mathcal{C}_s}\frac{|\boldsymbol{k}_i^T\boldsymbol{k}_n|}{\|\boldsymbol{k}_i\|_2\,\|\boldsymbol{k}_n\|_2}.
\label{eq:di_new}
\vspace{-0.5em}
\end{equation}
\end{proposition}

Proposition~\ref{prop:preserved_leakage} relates editable residual magnitude
to preserved-key exposure. Under exact preservation, Corollary~\ref{cor:normalized_exposure_budget}
gives
\(\|\boldsymbol{r}_i\|_2
\le B\|\boldsymbol{k}_i\|_2\sqrt{1-d_i^2}
\le B\|\boldsymbol{k}_i\|_2\),
showing that, at fixed key norm, the admissible residual budget
tightens as \(d_i\) increases. This
motivates stronger residual control for exposed samples to limit
preserved-side interference. PAC encodes this relation with the
bounded surrogate \((1+d_i)\|\boldsymbol{r}_i\|_2^2\), retaining proximal
control while increasing regularization. Proofs are
in Appendix~\ref{subsec:proof_preserved_exposure}.

\subsection{Pressure-Aware Coordination}
\label{sec:pac}
For each editable sample $i$, PAC optimizes a target representation
$\boldsymbol{v}_i$, where $-\log\mathbb{P}(o_i^*\mid \boldsymbol{v}_i)$
denotes the length-normalized target-token NLL.
Guided by Theorem~\ref{thm:editable_tension} and Proposition~\ref{prop:preserved_leakage}, PAC minimizes the pressure-aware objective
\vspace{-0.5em}
\begin{equation}
\label{eq:jno_pac}
\small
\begin{aligned}
\mathcal{L}_{\mathrm{PAC}}
&=
\mathcal{L}_{\mathrm{sem}} + \alpha \mathcal{L}_{\mathrm{prox}} + \beta \mathcal{L}_{\mathrm{pair}} \\
&= \sum\nolimits_{i=0}^{m}-\log \mathbb{P}(o_i^*\mid \boldsymbol{v}_i) \\
&\quad+ \alpha \sum\nolimits_{i=0}^{m}(1+d_i)\|\boldsymbol{v}_i-\boldsymbol{W}\boldsymbol{k}_i\|_2^2 \\
&\quad+ \beta \sum\nolimits_{0\le i<j\le m}s_{ij}\|\boldsymbol{v}_i-\boldsymbol{v}_j\|_2^2.
\end{aligned}
\vspace{-0.3em}
\end{equation}

The semantic term $\mathcal{L}_{\mathrm{sem}}$ encourages each editable sample to induce its target object. The exposure-weighted residual term $\mathcal{L}_{\mathrm{prox}}$ uses the preserved-side exposure score $d_i$ to control $\boldsymbol{r}_i=\boldsymbol{v}_i-\boldsymbol{W}\boldsymbol{k}_i$, assigning larger penalties to samples more exposed to preserved keys. The pairwise term $\mathcal{L}_{\mathrm{pair}}$ uses the editable-side coordination weights $s_{ij}$, equaling $\operatorname{tr}\!\big(\boldsymbol{V}_{\mathcal{V}}\boldsymbol{L}_s\boldsymbol{V}_{\mathcal{V}}^T\big)$ to penalize target separation over tightly coupled editable keys.

\subsection{Semantic Pre-execution Gating}
\label{sec:diagnostics}

PAC regularizes target geometry, but structural compatibility alone does not guarantee that the optimized representations still induce the intended target tokens. JNO therefore adds a semantic pre-execution gate as a conservative safety check:
\vspace{-0.3em}
\begin{equation}
\small
\ell_i^{\mathrm{sem}}=-\log \mathbb{P}(o_i^*\mid \hat{\boldsymbol{v}}_i),
\label{eq:semantic_loss_gate}
\vspace{-0.3em}
\end{equation}
where \(i=0\) indexes the edit request and \(i\in\{1,\ldots,m\}\) the variable neighbors. We count semantic failures as
\vspace{-0.3em}
\begin{equation}
\small
N_{\mathrm{fail}}=\sum\nolimits_{i=1}^{m}\mathbf{1}\!\left[\ell_i^{\mathrm{sem}}>\tau_{\mathrm{sem}}\right],
\label{eq:nfail}
\vspace{-0.3em}
\end{equation}
where $\tau_{\mathrm{sem}}$ is the semantic-loss threshold and $B_{\mathrm{sem}}\le m$ the failure budget. If \(\ell_0^{\mathrm{sem}}>\tau_{\mathrm{sem}}\) or \(N_{\mathrm{fail}} > B_{\mathrm{sem}}\), JNO deems the optimized target representations high-risk and abstains from parameter execution; otherwise it proceeds to execution. Gate-failed requests are discussed in Appendix~\ref{app:gate_failed_handling}.

\begin{algorithm}[t]
\centering
\small
\caption{JNO}
\label{alg:jno_execution}
\begin{algorithmic}[1]
\REQUIRE $\boldsymbol{e}$, layer weight $\boldsymbol{W}$, $\tau_{\mathrm{sem}}$, $B_{\mathrm{sem}}$, $\boldsymbol K_0$, $\boldsymbol K_p$
\ENSURE Updated $\boldsymbol{W}'$ or \texttt{HIGH\_RISK\_PARAMETRIC}
\STATE Construct neighborhood $\mathcal{C}=\{\boldsymbol{e}\}\cup\mathcal{C}_c\cup\mathcal{C}_s$
\STATE Extract editable and preserved keys $\{\boldsymbol{k}_i\}_{i=0}^{m}$, $\{\boldsymbol{k}_n\}_{n\in\mathcal{C}_s}$
\STATE Compute $\{d_i\}$ by Eq.~(\ref{eq:di_new}) and $\{s_{ij}\}$ by Eq.~(\ref{eq:sij})
\STATE Obtain optimized targets $\{\hat{\boldsymbol{v}}_i\}_{i=0}^{m}$ by minimizing Eq.~(\ref{eq:jno_pac})
\STATE Compute $\ell_0^{\mathrm{sem}}$ and semantic failures $N_{\mathrm{fail}}$ by Eqs.~(\ref{eq:semantic_loss_gate})--(\ref{eq:nfail})\IF{$\ell_0^{\mathrm{sem}}>\tau_{\mathrm{sem}}$ \textbf{or} $N_{\mathrm{fail}} > B_{\mathrm{sem}}$}
    \STATE \textbf{return} \texttt{HIGH\_RISK\_PARAMETRIC}
\ENDIF
\STATE Form residual matrix $\hat{\boldsymbol{R}}_{\mathcal{V}}$
\STATE Compute $\boldsymbol{\Delta}^*$ by Eq.~(\ref{eq:final_update})
\STATE \textbf{return} $\boldsymbol{W}'=\boldsymbol{W}+\boldsymbol{\Delta}^*$
\end{algorithmic}
\end{algorithm}
\setlength{\textfloatsep}{0.8em}

\begin{table*}[t]
\renewcommand{\arraystretch}{0.9}
\centering
\scriptsize
\setlength{\tabcolsep}{8.0pt}{
\begin{tabular}{cc|cccccc}
\toprule
\textbf{Model} & \textbf{Method} 
& \textbf{Rel.}$\uparrow$ 
& \textbf{LG}$\uparrow$ 
& \textbf{RE}$\uparrow$ 
& \textbf{SA}$\uparrow$ 
& \textbf{RS}$\uparrow$ 
& \textbf{FF}$\uparrow$ \\
\midrule

\multirow{11}{*}[-0.7ex]{\rotatebox[origin=c]{90}{Qwen2.5-1.5B-Instruct}}
& ROME & $97.4_{\pm 0.16}$ & $51.4_{\pm 0.24}$ & $28.4_{\pm 0.22}$ & $\underline{67.7}_{\pm 0.33}$ & $39.0_{\pm 0.31}$ & $32.2_{\pm 0.36}$ \\
& MEMIT & $\textcolor{red}{98.9}_{\pm 0.12}$ & $54.8_{\pm 0.21}$ & $32.6_{\pm 0.26}$ & $53.2_{\pm 0.29}$ & $39.7_{\pm 0.34}$ & $31.4_{\pm 0.30}$ \\
& PMET & $93.4_{\pm 0.20}$ & $53.6_{\pm 0.27}$ & $29.2_{\pm 0.24}$ & $60.1_{\pm 0.36}$ & $36.3_{\pm 0.37}$ & $34.3_{\pm 0.41}$ \\
& AlphaEdit & $\textcolor{red}{98.9}_{\pm 0.10}$ & $60.4_{\pm 0.26}$ & $43.8_{\pm 0.31}$ & $62.4_{\pm 0.34}$ & $31.8_{\pm 0.38}$ & $28.7_{\pm 0.35}$ \\
& $\text{AlphaEdit}_{\text{BLUE}}$ & $98.3_{\pm 0.15}$ & $59.7_{\pm 0.28}$ & $47.2_{\pm 0.33}$ & $65.8_{\pm 0.37}$ & $43.6_{\pm 0.41}$ & $43.1_{\pm 0.44}$ \\
& $\text{AlphaEdit}_{\text{Chain}}$ & $\underline{98.8}_{\pm 0.13}$ & $\underline{61.2}_{\pm 0.30}$ & $46.8_{\pm 0.29}$ & $63.3_{\pm 0.35}$ & $46.1_{\pm 0.43}$ & $45.8_{\pm 0.46}$ \\
& GLAME & $98.7_{\pm 0.17}$ & $60.1_{\pm 0.34}$ & $46.5_{\pm 0.36}$ & $64.2_{\pm 0.40}$ & $47.3_{\pm 0.47}$ & $46.9_{\pm 0.49}$ \\
& SAKE & $98.5_{\pm 0.14}$ & $59.3_{\pm 0.32}$ & $\underline{47.3}_{\pm 0.35}$ & $67.3_{\pm 0.39}$ & $49.2_{\pm 0.45}$ & $\underline{50.2}_{\pm 0.48}$ \\
& SUIT & $98.6_{\pm 0.18}$ & $60.2_{\pm 0.31}$ & $47.1_{\pm 0.37}$ & $63.7_{\pm 0.42}$ & $\underline{50.4}_{\pm 0.46}$ & $44.3_{\pm 0.50}$ \\
& \cellcolor{metablue}JNO 
& \cellcolor{metablue}$95.9_{\pm 0.24}$ 
& \cellcolor{metablue}$\textcolor{red}{65.6}_{\pm 0.36}$ 
& \cellcolor{metablue}$\textcolor{red}{56.4}_{\pm 0.39}$ 
& \cellcolor{metablue}$\textcolor{red}{74.1}_{\pm 0.43}$ 
& \cellcolor{metablue}$\textcolor{red}{58.8}_{\pm 0.50}$ 
& \cellcolor{metablue}$\textcolor{red}{61.8}_{\pm 0.52}$ \\
& \cellcolor{gray!15}\rule[-0.55ex]{0pt}{2.35ex}Improve 
& \cellcolor{gray!15}$-3.1\%$ 
& \cellcolor{gray!15}$7.2\%$ 
& \cellcolor{gray!15}$19.2\%$ 
& \cellcolor{gray!15}$9.5\%$ 
& \cellcolor{gray!15}$16.7\%$ 
& \cellcolor{gray!15}$23.1\%$ \\
\midrule \midrule

\multirow{11}{*}{\rotatebox{90}{GPT-J~(6B)}} 
& ROME & $96.8_{\pm 0.18}$ & $52.6_{\pm 0.25}$ & $32.8_{\pm 0.28}$ & $62.9_{\pm 0.34}$ & $37.5_{\pm 0.35}$ & $34.2_{\pm 0.39}$ \\
& MEMIT & $95.1_{\pm 0.21}$ & $56.5_{\pm 0.29}$ & $35.7_{\pm 0.30}$ & $57.1_{\pm 0.32}$ & $41.6_{\pm 0.38}$ & $35.6_{\pm 0.42}$ \\
& PMET & $95.8_{\pm 0.19}$ & $56.2_{\pm 0.27}$ & $35.4_{\pm 0.31}$ & $59.6_{\pm 0.35}$ & $41.3_{\pm 0.40}$ & $38.2_{\pm 0.44}$ \\
& AlphaEdit & $\underline{99.1}_{\pm 0.09}$ & $60.4_{\pm 0.26}$ & $43.5_{\pm 0.33}$ & $62.5_{\pm 0.38}$ & $43.6_{\pm 0.42}$ & $33.2_{\pm 0.36}$ \\
& $\text{AlphaEdit}_{\text{BLUE}}$ & $98.3_{\pm 0.16}$ & $58.3_{\pm 0.30}$ & $45.8_{\pm 0.35}$ & $65.1_{\pm 0.39}$ & $48.8_{\pm 0.46}$ & $38.9_{\pm 0.43}$ \\
& $\text{AlphaEdit}_{\text{Chain}}$ & $98.6_{\pm 0.14}$ & $60.9_{\pm 0.28}$ & $\underline{48.2}_{\pm 0.34}$ & $66.7_{\pm 0.41}$ & $52.4_{\pm 0.48}$ & $43.8_{\pm 0.45}$ \\
& GLAME & $98.4_{\pm 0.17}$ & $\underline{62.4}_{\pm 0.33}$ & $47.3_{\pm 0.38}$ & $67.8_{\pm 0.44}$ & $52.6_{\pm 0.49}$ & $43.5_{\pm 0.47}$ \\
& SAKE & $\textcolor{red}{99.2}_{\pm 0.08}$ & $60.7_{\pm 0.31}$ & $47.1_{\pm 0.36}$ & $\underline{68.4}_{\pm 0.40}$ & $\underline{56.2}_{\pm 0.50}$ & $\underline{45.9}_{\pm 0.46}$ \\
& SUIT & $98.7_{\pm 0.15}$ & $60.5_{\pm 0.34}$ & $47.8_{\pm 0.37}$ & $67.3_{\pm 0.42}$ & $54.1_{\pm 0.47}$ & $44.2_{\pm 0.49}$ \\
& \cellcolor{metablue}JNO 
& \cellcolor{metablue}$96.3_{\pm 0.23}$ 
& \cellcolor{metablue}$\textcolor{red}{67.3}_{\pm 0.38}$ 
& \cellcolor{metablue}$\textcolor{red}{55.3}_{\pm 0.41}$ 
& \cellcolor{metablue}$\textcolor{red}{76.5}_{\pm 0.45}$ 
& \cellcolor{metablue}$\textcolor{red}{65.6}_{\pm 0.53}$ 
& \cellcolor{metablue}$\textcolor{red}{60.1}_{\pm 0.51}$ \\
& \cellcolor{gray!15}\rule[-0.55ex]{0pt}{2.35ex}Improve
& \cellcolor{gray!15}$-3.0\%$
& \cellcolor{gray!15}$7.9\%$
& \cellcolor{gray!15}$14.7\%$
& \cellcolor{gray!15}$11.8\%$
& \cellcolor{gray!15}$16.7\%$
& \cellcolor{gray!15}$30.9\%$ \\
\midrule \midrule

\multirow{11}{*}{\rotatebox{90}{LLaMA3~(8B)}} 
& ROME & $96.3_{\pm 0.19}$ & $53.8_{\pm 0.28}$ & $37.3_{\pm 0.29}$ & $58.2_{\pm 0.36}$ & $36.1_{\pm 0.37}$ & $36.1_{\pm 0.40}$ \\
& MEMIT & $95.2_{\pm 0.22}$ & $58.3_{\pm 0.31}$ & $38.9_{\pm 0.32}$ & $61.1_{\pm 0.34}$ & $43.5_{\pm 0.41}$ & $37.6_{\pm 0.43}$ \\
& PMET & $94.5_{\pm 0.23}$ & $59.2_{\pm 0.30}$ & $40.9_{\pm 0.34}$ & $59.3_{\pm 0.37}$ & $47.1_{\pm 0.44}$ & $42.8_{\pm 0.47}$ \\
& AlphaEdit & $\textcolor{red}{99.2}_{\pm 0.09}$ & $60.5_{\pm 0.27}$ & $43.2_{\pm 0.32}$ & $62.6_{\pm 0.36}$ & $55.7_{\pm 0.45}$ & $41.4_{\pm 0.42}$ \\
& $\text{AlphaEdit}_{\text{BLUE}}$ & $98.4_{\pm 0.17}$ & $56.9_{\pm 0.29}$ & $44.6_{\pm 0.35}$ & $64.4_{\pm 0.40}$ & $54.3_{\pm 0.43}$ & $34.5_{\pm 0.41}$ \\
& $\text{AlphaEdit}_{\text{Chain}}$ & $98.9_{\pm 0.13}$ & $62.7_{\pm 0.33}$ & $47.5_{\pm 0.36}$ & $67.1_{\pm 0.42}$ & $55.6_{\pm 0.47}$ & $47.3_{\pm 0.50}$ \\
& GLAME & $98.7_{\pm 0.16}$ & $\underline{63.1}_{\pm 0.35}$ & $48.6_{\pm 0.39}$ & $65.3_{\pm 0.43}$ & $56.3_{\pm 0.48}$ & $48.4_{\pm 0.52}$ \\
& SAKE & $\underline{99.1}_{\pm 0.10}$ & $62.5_{\pm 0.32}$ & $\underline{49.3}_{\pm 0.37}$ & $66.4_{\pm 0.41}$ & $\underline{58.7}_{\pm 0.51}$ & $\underline{49.7}_{\pm 0.49}$ \\
& SUIT & $98.7_{\pm 0.15}$ & $61.6_{\pm 0.34}$ & $48.4_{\pm 0.40}$ & $\underline{68.5}_{\pm 0.44}$ & $58.2_{\pm 0.50}$ & $47.6_{\pm 0.48}$ \\
& \cellcolor{metablue}JNO 
& \cellcolor{metablue}$96.2_{\pm 0.25}$ 
& \cellcolor{metablue}$\textcolor{red}{72.1}_{\pm 0.40}$ 
& \cellcolor{metablue}$\textcolor{red}{58.2}_{\pm 0.43}$ 
& \cellcolor{metablue}$\textcolor{red}{73.3}_{\pm 0.46}$ 
& \cellcolor{metablue}$\textcolor{red}{72.4}_{\pm 0.54}$ 
& \cellcolor{metablue}$\textcolor{red}{63.8}_{\pm 0.53}$ \\
& \cellcolor{gray!15}\rule[-0.55ex]{0pt}{2.35ex}Improve
& \cellcolor{gray!15}$-3.1\%$
& \cellcolor{gray!15}$14.3\%$
& \cellcolor{gray!15}$18.1\%$
& \cellcolor{gray!15}$7.0\%$
& \cellcolor{gray!15}$23.3\%$
& \cellcolor{gray!15}$28.4\%$ \\
\bottomrule
\end{tabular}
}
\vspace{-0.1in}
\caption{\label{tab:ripple-results}
  Comparison of JNO with existing methods on RippleEdits. All metrics are computed over full requested-edit set, including gate-failed requests. The best results are highlighted in \textcolor{red}{red}, while the second-best results are \underline{underlined}. Improve reports the relative change of JNO compared with the strongest baseline in each metric. 
}
\vspace{-0.15in}
\end{table*}

\subsection{Reliable Execution}
\label{sec:execution}

{JNO is a neighborhood-aware target-planning framework
that operates before parameter execution.} PAC and semantic gating
yield accepted target residuals, which our main implementation realizes
using the AlphaEdit-style null-space-constrained backend below.
{For this instantiation,} let
\(\boldsymbol K_0\) denote the preserved-knowledge key matrix used by
AlphaEdit, \(\boldsymbol P\) be its orthogonal null-space projector, and
\(\boldsymbol K_p\) denote previous keys for stabilizing sequential editing. Any update
\(\boldsymbol\Delta=\widetilde{\boldsymbol\Delta}\boldsymbol P\)
satisfies \(\boldsymbol\Delta\boldsymbol K_0=\boldsymbol 0\).
Given the accepted PAC residual matrix
$\hat{\boldsymbol R}_{\mathcal V}
=
[\hat{\boldsymbol v}_i-\boldsymbol W\boldsymbol k_i]_{i=0}^{m}$,
JNO solves
\vspace{-0.5em}
\begin{equation}
\small
\mathop{\arg\min}\nolimits_{\widetilde{\boldsymbol\Delta}}
\|
\widetilde{\boldsymbol\Delta}\boldsymbol P\boldsymbol K_{\mathcal V}
-
\hat{\boldsymbol R}_{\mathcal V}
\|^2
+
\|
\widetilde{\boldsymbol\Delta}\boldsymbol P\boldsymbol K_p
\|^2
+
\|
\widetilde{\boldsymbol\Delta}\boldsymbol P
\|^2 .
\label{eq:final_update}
\vspace{-0.5em}
\end{equation}
Solving Eq.~\eqref{eq:final_update} yields
\(\boldsymbol\Delta^{*}
=\widetilde{\boldsymbol\Delta}^{*}\boldsymbol P\), and the final weight
is \(\boldsymbol W'=\boldsymbol W+\boldsymbol\Delta^{*}\).

Algorithm~\ref{alg:jno_execution} summarizes JNO, which constructs
editable--preserved neighborhoods, plans neighborhood targets with PAC,
filters high-risk plans through semantic gating, and executes accepted
updates.
{Alternative parameter-update backends are evaluated in Appendix~\ref{app:executor_controls}.}

\section{Experiments}
\label{sec:experiments}

We evaluate JNO aiming to answer the following questions.
Q1: Does JNO improve positive ripple propagation while suppressing negative ripple effects?
Q2: Does JNO remain effective on widely used KE benchmarks?
Q3: Which components of JNO contribute to the gains?
Q4: Does the pre-execution gate provide a controllable execution trade-off?
Q5: Are the constructed editable and preserved neighborhoods reliable?

\begin{table*}[t]
\renewcommand{\arraystretch}{0.95}
\centering
\small
\resizebox{\textwidth}{!}{
\begin{tabular}{cc|ccccc|ccc}
\toprule
\multirow{2}{*}{\textbf{Model}} & \multirow{2}{*}{\textbf{Method}} 
& \multicolumn{5}{c|}{\textbf{CounterFact}} 
& \multicolumn{3}{c}{\textbf{ZsRE}} \\
\cmidrule(lr){3-7} \cmidrule(lr){8-10}
& & \textbf{Eff.}$\uparrow$ & \textbf{Gen.}$\uparrow$ & \textbf{Spe.}$\uparrow$ & \textbf{Flu.}$\uparrow$ & \textbf{Consis.}$\uparrow$ 
& \textbf{Eff.}$\uparrow$ & \textbf{Gen.}$\uparrow$ & \textbf{Spe.}$\uparrow$ \\
\midrule

\multirow{11}{*}[-0.4ex]{\rotatebox[origin=c]{90}{Qwen2.5-1.5B-Instruct}}
& ROME & $52.86_{\pm 0.42}$ & $50.35_{\pm 0.23}$ & $51.21_{\pm 0.63}$ & $327.50_{\pm 0.21}$ & $3.05_{\pm 0.08}$ & $28.41_{\pm 0.12}$ & $27.33_{\pm 0.23}$ & $6.01_{\pm 0.11}$ \\
& MEMIT & $87.47_{\pm 0.12}$ & $83.11_{\pm 0.41}$ & $58.24_{\pm 0.31}$ & $505.01_{\pm 0.76}$ & $23.27_{\pm 0.16}$ & $65.44_{\pm 0.26}$ & $63.90_{\pm 0.13}$ & $24.77_{\pm 0.17}$ \\
& PMET & $88.36_{\pm 0.24}$ & $82.58_{\pm 0.32}$ & $59.12_{\pm 0.28}$ & $511.42_{\pm 0.63}$ & $24.06_{\pm 0.15}$ & $67.28_{\pm 0.25}$ & $65.12_{\pm 0.18}$ & $25.36_{\pm 0.22}$ \\
& AlphaEdit & $98.13_{\pm 0.21}$ & $96.29_{\pm 0.72}$ & $70.54_{\pm 0.34}$ & $613.35_{\pm 0.35}$ & $34.38_{\pm 0.15}$ & $97.78_{\pm 0.35}$ & $92.94_{\pm 0.13}$ & $28.36_{\pm 0.13}$ \\
& $\text{AlphaEdit}_{\text{BLUE}}$ & $98.35_{\pm 0.36}$ & $97.13_{\pm 0.51}$ & $73.24_{\pm 0.23}$ & $621.48_{\pm 0.32}$ & $36.64_{\pm 0.14}$ & $98.60_{\pm 0.21}$ & $93.75_{\pm 0.26}$ & $\underline{30.54}_{\pm 0.31}$ \\
& $\text{AlphaEdit}_{\text{Chain}}$ & $98.27_{\pm 0.18}$ & $96.75_{\pm 0.37}$ & $73.05_{\pm 0.29}$ & $617.62_{\pm 0.44}$ & $37.10_{\pm 0.19}$ & $98.78_{\pm 0.19}$ & $93.83_{\pm 0.22}$ & $30.06_{\pm 0.27}$ \\
& GLAME & $98.18_{\pm 0.11}$ & $97.16_{\pm 0.31}$ & $72.85_{\pm 0.27}$ & $615.23_{\pm 0.31}$ & $36.91_{\pm 0.13}$ & $\textcolor{red}{98.92}_{\pm 0.24}$ & $94.13_{\pm 0.28}$ & $29.72_{\pm 0.36}$ \\
& SAKE & $\textcolor{red}{98.47}_{\pm 0.26}$ & $\underline{97.25}_{\pm 0.18}$ & $\underline{73.60}_{\pm 0.24}$ & $\underline{622.09}_{\pm 0.38}$ & $37.05_{\pm 0.10}$ & $\underline{98.85}_{\pm 0.22}$ & $\underline{94.57}_{\pm 0.24}$ & $30.41_{\pm 0.31}$ \\
& SUIT & $\underline{98.38}_{\pm 0.15}$ & $97.19_{\pm 0.21}$ & $73.51_{\pm 0.31}$ & $620.42_{\pm 0.34}$ & $\underline{38.12}_{\pm 0.15}$ & $97.94_{\pm 0.15}$ & $94.26_{\pm 0.21}$ & $30.13_{\pm 0.23}$ \\
& \cellcolor{metablue}JNO 
& \cellcolor{metablue}$97.63_{\pm 0.18}$ 
& \cellcolor{metablue}$\textcolor{red}{97.68}_{\pm 0.34}$ 
& \cellcolor{metablue}$\textcolor{red}{75.52}_{\pm 0.35}$ 
& \cellcolor{metablue}$\textcolor{red}{624.31}_{\pm 0.41}$ 
& \cellcolor{metablue}$\textcolor{red}{41.26}_{\pm 0.12}$ 
& \cellcolor{metablue}$97.93_{\pm 0.24}$ 
& \cellcolor{metablue}$\textcolor{red}{96.19}_{\pm 0.31}$ 
& \cellcolor{metablue}$\textcolor{red}{32.46}_{\pm 0.48}$ \\
& \cellcolor{gray!15}\rule[-0.55ex]{0pt}{2.35ex}Improve
& \cellcolor{gray!15}$-0.9\%$
& \cellcolor{gray!15}$0.4\%$
& \cellcolor{gray!15}$2.6\%$
& \cellcolor{gray!15}$0.4\%$
& \cellcolor{gray!15}$8.2\%$
& \cellcolor{gray!15}$-1.0\%$
& \cellcolor{gray!15}$1.7\%$
& \cellcolor{gray!15}$6.3\%$ \\
\midrule \midrule

\multirow{11}{*}{\rotatebox{90}{GPT-J~(6B)}} 
& ROME & $56.68_{\pm 0.41}$ & $52.42_{\pm 0.27}$ & $51.77_{\pm 0.12}$ & $576.78_{\pm 0.12}$ & $3.06_{\pm 0.04}$ & $54.88_{\pm 0.32}$ & $51.82_{\pm 0.12}$ & $8.45_{\pm 0.11}$ \\
& MEMIT & $97.84_{\pm 0.13}$ & $94.86_{\pm 0.46}$ & $61.92_{\pm 0.41}$ & $538.75_{\pm 0.48}$ & $32.74_{\pm 0.12}$ & $93.73_{\pm 0.16}$ & $89.13_{\pm 0.23}$ & $28.59_{\pm 0.37}$ \\
& PMET & $98.62_{\pm 0.18}$ & $94.21_{\pm 0.38}$ & $63.85_{\pm 0.36}$ & $547.64_{\pm 0.44}$ & $34.36_{\pm 0.15}$ & $96.18_{\pm 0.18}$ & $91.36_{\pm 0.25}$ & $28.02_{\pm 0.34}$ \\
& AlphaEdit & $99.23_{\pm 0.28}$ & $95.94_{\pm 0.26}$ & $73.63_{\pm 0.23}$ & $611.85_{\pm 0.27}$ & $40.63_{\pm 0.24}$ & $99.17_{\pm 0.14}$ & $95.61_{\pm 0.12}$ & $27.52_{\pm 0.21}$ \\
& $\text{AlphaEdit}_{\text{BLUE}}$ & $\underline{99.62}_{\pm 0.12}$ & $96.42_{\pm 0.18}$ & $75.58_{\pm 0.55}$ & $618.29_{\pm 0.32}$ & $42.21_{\pm 0.32}$ & $\underline{99.73}_{\pm 0.13}$ & $96.28_{\pm 0.29}$ & $28.36_{\pm 0.34}$ \\
& $\text{AlphaEdit}_{\text{Chain}}$ & $99.53_{\pm 0.16}$ & $96.34_{\pm 0.22}$ & $74.45_{\pm 0.38}$ & $616.72_{\pm 0.35}$ & $42.09_{\pm 0.27}$ & $99.58_{\pm 0.15}$ & $96.05_{\pm 0.24}$ & $28.42_{\pm 0.29}$ \\
& GLAME & $99.15_{\pm 0.15}$ & $95.78_{\pm 0.26}$ & $74.92_{\pm 0.43}$ & $614.10_{\pm 0.38}$ & $41.45_{\pm 0.33}$ & $99.64_{\pm 0.13}$ & $95.85_{\pm 0.28}$ & $27.96_{\pm 0.40}$ \\
& SAKE & $\textcolor{red}{99.71}_{\pm 0.17}$ & $\underline{96.47}_{\pm 0.28}$ & $\underline{76.29}_{\pm 0.41}$ & $\underline{620.25}_{\pm 0.42}$ & $42.12_{\pm 0.38}$ & $\textcolor{red}{99.75}_{\pm 0.12}$ & $\underline{96.31}_{\pm 0.30}$ & $\underline{28.77}_{\pm 0.44}$ \\
& SUIT & $99.21_{\pm 0.13}$ & $95.25_{\pm 0.21}$ & $75.98_{\pm 0.32}$ & $617.32_{\pm 0.26}$ & $\underline{42.34}_{\pm 0.16}$ & $99.72_{\pm 0.15}$ & $95.13_{\pm 0.28}$ & $26.36_{\pm 0.35}$ \\
& \cellcolor{metablue}JNO 
& \cellcolor{metablue}$98.24_{\pm 0.21}$ 
& \cellcolor{metablue}$\textcolor{red}{97.72}_{\pm 0.34}$ 
& \cellcolor{metablue}$\textcolor{red}{78.03}_{\pm 0.37}$ 
& \cellcolor{metablue}$\textcolor{red}{623.14}_{\pm 0.51}$ 
& \cellcolor{metablue}$\textcolor{red}{43.85}_{\pm 0.62}$ 
& \cellcolor{metablue}$98.31_{\pm 0.27}$ 
& \cellcolor{metablue}$\textcolor{red}{97.62}_{\pm 0.31}$ 
& \cellcolor{metablue}$\textcolor{red}{31.24}_{\pm 0.53}$ \\
& \cellcolor{gray!15}\rule[-0.55ex]{0pt}{2.35ex}Improve
& \cellcolor{gray!15}$-1.5\%$
& \cellcolor{gray!15}$1.3\%$
& \cellcolor{gray!15}$2.3\%$
& \cellcolor{gray!15}$0.5\%$
& \cellcolor{gray!15}$3.6\%$
& \cellcolor{gray!15}$-1.5\%$
& \cellcolor{gray!15}$1.4\%$
& \cellcolor{gray!15}$8.6\%$ \\
\midrule \midrule

\multirow{11}{*}{\rotatebox{90}{LLaMA3~(8B)}} 
& ROME & $61.54_{\pm 0.25}$ & $57.62_{\pm 0.52}$ & $45.42_{\pm 0.38}$ & $425.86_{\pm 0.42}$ & $3.02_{\pm 0.04}$ & $1.86_{\pm 0.03}$ & $1.76_{\pm 0.02}$ & $0.64_{\pm 0.09}$ \\
& MEMIT & $62.60_{\pm 0.41}$ & $61.45_{\pm 0.12}$ & $47.41_{\pm 0.42}$ & $421.66_{\pm 0.33}$ & $6.16_{\pm 0.11}$ & $32.88_{\pm 0.16}$ & $28.85_{\pm 0.24}$ & $17.21_{\pm 0.18}$ \\
& PMET & $71.82_{\pm 0.38}$ & $66.26_{\pm 0.18}$ & $51.12_{\pm 0.39}$ & $458.74_{\pm 0.36}$ & $7.12_{\pm 0.12}$ & $36.46_{\pm 0.19}$ & $31.69_{\pm 0.22}$ & $20.16_{\pm 0.20}$ \\
& AlphaEdit & $97.78_{\pm 0.13}$ & $93.63_{\pm 0.14}$ & $65.57_{\pm 0.29}$ & $615.84_{\pm 0.14}$ & $31.67_{\pm 0.10}$ & $94.04_{\pm 0.13}$ & $90.24_{\pm 0.13}$ & $31.06_{\pm 0.23}$ \\
& $\text{AlphaEdit}_{\text{BLUE}}$ & $98.34_{\pm 0.02}$ & $95.31_{\pm 0.42}$ & $\underline{74.98}_{\pm 0.48}$ & $625.62_{\pm 0.29}$ & $33.19_{\pm 0.21}$ & $\textcolor{red}{98.11}_{\pm 0.31}$ & $92.91_{\pm 0.25}$ & $32.95_{\pm 0.41}$ \\
& $\text{AlphaEdit}_{\text{Chain}}$ & $98.27_{\pm 0.10}$ & $95.48_{\pm 0.36}$ & $73.15_{\pm 0.44}$ & $623.86_{\pm 0.31}$ & $32.36_{\pm 0.18}$ & $97.48_{\pm 0.26}$ & $92.18_{\pm 0.28}$ & $32.72_{\pm 0.39}$ \\
& GLAME & $97.81_{\pm 0.12}$ & $95.82_{\pm 0.49}$ & $74.30_{\pm 0.36}$ & $621.05_{\pm 0.27}$ & $33.48_{\pm 0.23}$ & $97.23_{\pm 0.32}$ & $91.72_{\pm 0.30}$ & $31.52_{\pm 0.50}$ \\
& SAKE & $\underline{98.65}_{\pm 0.15}$ & $\underline{96.54}_{\pm 0.52}$ & $74.68_{\pm 0.31}$ & $\underline{626.42}_{\pm 0.21}$ & $34.13_{\pm 0.25}$ & $\underline{97.89}_{\pm 0.34}$ & $\underline{93.12}_{\pm 0.32}$ & $33.15_{\pm 0.58}$ \\
& SUIT & $\textcolor{red}{98.79}_{\pm 0.21}$ & $94.14_{\pm 0.17}$ & $73.94_{\pm 0.25}$ & $622.81_{\pm 0.31}$ & $\underline{35.02}_{\pm 0.12}$ & $97.62_{\pm 0.25}$ & $91.93_{\pm 0.26}$ & $\underline{33.84}_{\pm 0.31}$ \\
& \cellcolor{metablue}JNO 
& \cellcolor{metablue}$97.93_{\pm 0.24}$ 
& \cellcolor{metablue}$\textcolor{red}{97.73}_{\pm 0.61}$ 
& \cellcolor{metablue}$\textcolor{red}{76.62}_{\pm 0.14}$ 
& \cellcolor{metablue}$\textcolor{red}{628.94}_{\pm 0.15}$ 
& \cellcolor{metablue}$\textcolor{red}{36.21}_{\pm 0.26}$ 
& \cellcolor{metablue}$97.12_{\pm 0.31}$ 
& \cellcolor{metablue}$\textcolor{red}{94.81}_{\pm 0.34}$ 
& \cellcolor{metablue}$\textcolor{red}{35.63}_{\pm 0.71}$ \\
& \cellcolor{gray!15}\rule[-0.55ex]{0pt}{2.35ex}Improve
& \cellcolor{gray!15}$-0.9\%$
& \cellcolor{gray!15}$1.2\%$
& \cellcolor{gray!15}$2.2\%$
& \cellcolor{gray!15}$0.4\%$
& \cellcolor{gray!15}$3.4\%$
& \cellcolor{gray!15}$-1.0\%$
& \cellcolor{gray!15}$1.8\%$
& \cellcolor{gray!15}$5.3\%$ \\
\bottomrule
\end{tabular}
}
\vspace{-0.1in}
\caption{\label{main-results}
  Comparison of JNO with methods on CounterFact and ZsRE. All metrics are computed over full requested-edit set, including gate-failed requests. The best results are highlighted in \textcolor{red}{red}, while the second-best results are \underline{underlined}. Improve reports the relative change of JNO compared with the strongest
baseline in each metric.
}
\vspace{-0.15in}
\end{table*}

\subsection{Experimental Setup}
\label{sec:exp_setup}

\paragraph{Datasets \& Evaluation Metrics.}
We utilize three benchmarks. For RippleEdits~\citep{cohen2024evaluating}, we follow ChainEdit~\citep{dong2025chainedit} and report Reliability (Rel.), LG, RE, SA, RS, and FF. Rel. assesses KE success~\citep{wang-etal-2024-easyedit}; LG, RE, and SA measure positive ripple propagation, while RS and FF measure negative-ripple prevention. On ZsRE~\citep{levy-etal-2017-zero} and CounterFact~\citep{meng2022locating}, we report Efficacy, Generalization, Specificity, Fluency, and Consistency. Metrics are computed over the full requested-edit set, including gate-failed requests.
For JNO, execution coverage denotes the semantic-gate pass rate over requested edits (reported in Appendix~\ref{app:coverage}).

\vspace{-0.05in}
\paragraph{Baselines \& Implementation Details.}
We compare JNO with KE methods: ROME \citep{meng2022locating}, MEMIT \citep{mengmass}, PMET \citep{li2024pmet}, AlphaEdit \citep{fang2025alphaedit}, BLUE \citep{li2025rethinking}, ChainEdit \citep{dong2025chainedit}, GLAME \citep{zhang2024knowledge}, SAKE \citep{scialanga2025sake}, and SUIT \citep{park2026suit}, and evaluate three architectures of varying scales, including Qwen2.5-1.5B-Instruct \citep{qwen2.5,yang2024qwen2technicalreport}, GPT-J (6B) \citep{gpt-j}, and LLaMA3 (8B) \citep{llama3modelcard}. For JNO, we set $\alpha=1.2$, $\beta=0.9$, $\tau_{\mathrm{sem}}=1.4$, and $B_{\mathrm{sem}}=1$ unless otherwise noted. All methods use the same constructed-neighborhood protocol: each edit request is augmented with variable and preserved neighborhood samples in Sec.~\ref{sec:neighbor_construction}, and baselines edit the resulting neighborhood facts as related batches following their original procedures. The results are averaged over three random seeds, and experiments use one NVIDIA L40 GPU (48 GB).

\subsection{Main Results on RippleEdits (Q1)}
\label{sec:exp_ripple}
To assess knowledge update and retention, {we evaluate all 839 requested
edits from RippleEdits-POPULAR under sequential editing with an edit-request batch size of 10.} As shown in Table~\ref{tab:ripple-results}, on LLaMA3, JNO improves RE
by 18.1\% over SAKE and SA by 7.0\% over SUIT; on GPT-J, it improves
RS and FF over SAKE by 16.7\% and 30.9\%, respectively. Compared with the strongest Rel. baseline, JNO incurs a 3.0\% relative Rel. drop across backbones while maintaining Rel. \(\ge 95.9\%\). Since all metrics include gate-failed requests, this reflects an
explicit reliability--ripple trade-off rather than selective
evaluation; {full-coverage options are detailed in
Appendix~\ref{app:gate_failed_handling}.} Appendix~\ref{app:efficiency} shows that JNO's target-planning overhead remains moderate.

\subsection{Results on CounterFact and ZsRE (Q2)}
\label{sec:exp_standard}
To examine whether JNO remains effective on two widely used benchmarks, {we evaluate all methods on 2,000 edit requests sampled from each dataset
under sequential editing protocol, augmenting each request with
variable and preserved neighborhood samples.}
Table~\ref{main-results} shows consistent gains across
datasets and models. On CounterFact with Qwen2.5-1.5B-Instruct, JNO improves Specificity by 2.6\% over SAKE and Consistency by 8.2\% over SUIT, suggesting better retention of unrelated neighborhood knowledge and more stable generations. On ZsRE with GPT-J, JNO improves Specificity by 8.6\% over SAKE, indicating stronger neighborhood preservation. Compared with the strongest Eff.\ baseline, JNO incurs Eff.\ drops of
at most 1.5\%. Results show that JNO preserves strong edit efficacy
while improving generalization, specificity, and generation stability. {Additional ground-truth-free locality results based on
$D_{\mathrm{KL}}$ and Top-$k$ overlap are reported in
Appendix~\ref{app:behavioral_locality}.}

\subsection{Ablation Studies (Q3)}
\label{sec:exp_ablation}
{To assess JNO components, we conduct ablations on LLaMA3 with RippleEdits: \textit{w/o $\mathcal{L}_\mathrm{prox}$} removes the exposure-weighted residual penalty; \textit{w/o $\mathcal{L}_\mathrm{pair}$} removes the pairwise coordination penalty; and \textit{w/o gate} removes semantic pre-execution filtering. Table~\ref{tab:ripple-results-ablation} reveals complementary,
pressure-specific roles: removing $\mathcal{L}_{\mathrm{prox}}$
reduces RS/FF by 16.3/19.2 points, while removing
$\mathcal{L}_{\mathrm{pair}}$ reduces LG/RE/SA by
10.9/10.8/10.6 points. These results indicate an asymmetric but
complementary division of labor: $\mathcal{L}_{\mathrm{prox}}$
primarily protects preserved knowledge, whereas
$\mathcal{L}_{\mathrm{pair}}$ supports coordination among editable
neighbors. Together, they balance propagation and preservation within
a unified target-planning objective. Removing the gate raises Rel.\
but lowers ripple-aware scores.}

\begin{table}[t]
\scriptsize
\renewcommand{\arraystretch}{0.8}
\setlength{\tabcolsep}{3pt}
\centering
\resizebox{\columnwidth}{!}{
\begin{tabular}{c|cccccc}
\toprule
\textbf{Method}
& \textbf{Rel.}$\uparrow$
& \textbf{LG}$\uparrow$
& \textbf{RE}$\uparrow$
& \textbf{SA}$\uparrow$
& \textbf{RS}$\uparrow$
& \textbf{FF}$\uparrow$ \\
\midrule

JNO
& $96.2$
& $72.1$
& $58.2$
& $73.3$
& $72.4$
& $63.8$ \\

JNO \textit{w/o $\mathcal{L}_\mathrm{prox}$}
& $94.8$
& $70.8$
& $56.3$
& $71.5$
& $56.1$
& $44.6$ \\

JNO \textit{w/o $\mathcal{L}_\mathrm{pair}$}
& $95.0$
& $61.2$
& $47.4$
& $62.7$
& $68.3$
& $59.1$ \\

JNO \textit{w/o $\mathrm{gate}$}
& $98.9$
& $66.4$
& $52.5$
& $67.7$
& $66.5$
& $57.2$ \\

\bottomrule
\end{tabular}
}
\vspace{-0.1in}
\caption{\label{tab:ripple-results-ablation}
  Ablation results on RippleEdits in JNO.
}
\vspace{-0.1in}
\end{table}

\subsection{Analysis of Semantic Pre-execution Gating (Q4)}
\label{sec:exp:gate_surface}
\begin{figure*}[t]
\centering
\includegraphics[width=0.95\linewidth]{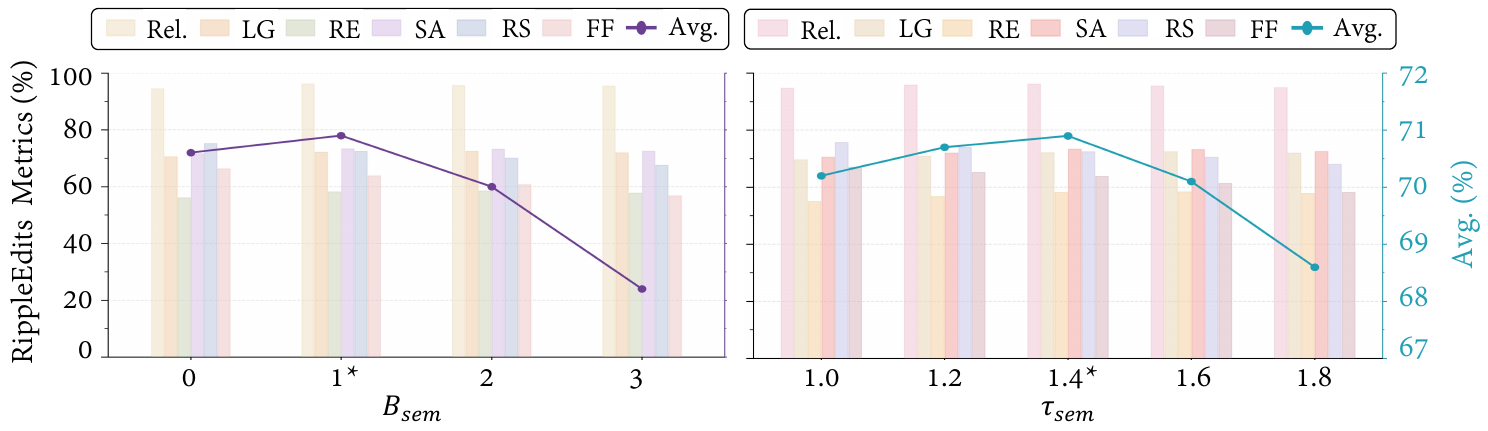}
\vspace{-0.1in}
\caption{
    Sensitivity analysis of semantic pre-execution gating. The Avg. is the harmonic mean of the metrics.
}
\vspace{-0.2in}
\label{fig:gate_surface}
\end{figure*}

\paragraph{Sensitivity Analysis of Semantic Pre-execution Gating.}
We examine whether the semantic gate provides a controllable execution
trade-off. Fig.~\ref{fig:gate_surface} evaluates
$B_{\mathrm{sem}}\in\{0,1,2,3\}$ and
$\tau_{\mathrm{sem}}\in\{1.0,1.2,1.4,1.6,1.8\}$ on
RippleEdits-POPULAR with LLaMA3, using the harmonic mean of the six
metrics as the aggregate score. Strict failure budgets over-reject
neighborhood plans, whereas overly loose ones admit unreliable plans,
with the aggregate score peaking at $B_{\mathrm{sem}}=1$.
Relaxing $\tau_{\mathrm{sem}}$ from $1.0$ to $1.4$ similarly improves
the aggregate score, while further relaxation weakens the
propagation--preservation balance. We therefore adopt
$B_{\mathrm{sem}}=1$ and $\tau_{\mathrm{sem}}=1.4$ as the best
empirical operating point among the evaluated combinations.
PAC sensitivity is reported in
Appendix~\ref{app:hyperparameter_sensitivity}.

\begin{table}[t]
\centering
\scriptsize
\renewcommand{\arraystretch}{0.7}
\setlength{\tabcolsep}{2pt}
\resizebox{\columnwidth}{!}{
\begin{tabular}{l|c|cccccc}
\toprule
\textbf{Method} & \textbf{Cov.} & \textbf{Rel.}$\uparrow$ & \textbf{LG}$\uparrow$ & \textbf{RE}$\uparrow$ & \textbf{SA}$\uparrow$ & \textbf{RS}$\uparrow$ & \textbf{FF}$\uparrow$ \\
\midrule
AlphaEdit & 100.0 & 99.2 & 60.5 & 43.2 & 62.6 & 55.7 & 41.4 \\
AlphaEdit * & 96.3  & 95.2 & 63.4 & 47.0 & 65.8 & 59.8 & 47.0 \\
SAKE & 100.0 & 99.1 & 62.5 & 49.3 & 66.4 & 58.7 & 49.7 \\
SAKE * & 96.3  & 95.7 & 65.5 & 53.6 & 69.8 & 63.0 & 53.5 \\
\midrule
JNO w/o gate & 100.0 & 98.9 & 66.4 & 52.5 & 67.7 & 66.5 & 57.2 \\
JNO & 96.3  & 96.2 & 72.1 & 58.2 & 73.3 & 72.4 & 63.8 \\
\bottomrule
\end{tabular}
}
\vspace{-0.1in}
\caption{
Coverage-matched selective execution on RippleEdits. * indicates +gate (semantic pre-execution).}
\label{tab:gated_baseline_rippleedits_main}
\vspace{-0.05in}
\end{table}

\paragraph{Coverage-matched selective execution.}
We examine whether JNO's gains depend on selective rejection. Table~\ref{tab:gated_baseline_rippleedits_main} compares
100\%-coverage methods and gated baselines matched to JNO's 96.3\%
coverage. When every request is
applied, JNO w/o gate retains 100\% coverage and outperforms 100\%-coverage
baselines across ripple-aware metrics. Compared with SAKE, it improves RE by
6.5\% and FF by 15.1\% while maintaining comparable Rel., indicating that the
improvement does not rely on selective rejection. At matched coverage, JNO
outperforms gated AlphaEdit and SAKE, so its advantage is not explained by
coverage reduction alone. Compared with its accept-all variant, JNO achieves a
stronger propagation--preservation balance, indicating that semantic gating
acts as a complementary execution-level safeguard rather than the primary
source of JNO's gains. {More coverage results and the structure of gate-rejected requests
are reported in Appendices~\ref{app:coverage_matched_gate}
and~\ref{app:rejection_structure}.}

\subsection{Auditing Neighborhood Reliability (Q5)}
\label{sec:exp:neighborhood_quality}
To audit the quality of constructed editable and preserved neighborhoods, we randomly sample 200 edit requests per dataset and inspect
up to three editable and three preserved neighbors per request. Table~\ref{tab:neighborhood_quality} reports editable-neighbor
precision ($\mathrm{Prec}_c$), preserved-neighbor precision
($\mathrm{Prec}_s$), conflict rate (Conf.), and valid-neighborhood rate
(Valid), where Valid denotes neighborhoods whose audited
candidates are valid and conflict-free. The constructed neighborhoods achieve
high average $\mathrm{Prec}_c$ and $\mathrm{Prec}_s$, with only a 2.9\% conflict rate and an 89.7\% Valid rate. The results indicate that most
constructed neighborhoods provide reliable editable--preserved constraints,
supporting JNO's neighborhood-guided optimization rather than arbitrary
neighborhood expansion. The detailed judging protocol and prompts are provided in Appendix~\ref{subsec:Neighborhood_Verification}, with neighborhood-noise sensitivity reported in Appendix~\ref{subsec:neighborhood_quality_sensitivity}.

\begin{table}[t]
\centering
\renewcommand{\arraystretch}{0.6}
\setlength{\tabcolsep}{6pt}
\footnotesize
\begin{tabular}{lcccc}
\toprule
\textbf{Dataset} &
\(\textbf{Prec}_{c}\uparrow\) &
\(\textbf{Prec}_{s}\uparrow\) &
\(\textbf{Conf.}\downarrow\) &
\(\textbf{Valid}\uparrow\) \\
\midrule
RippleEdits & 95.1 & 96.8 & 1.8 & 92.5 \\
CounterFact & 92.8 & 95.2 & 3.0 & 89.5 \\
ZsRE        & 90.6 & 94.1 & 3.8 & 87.0 \\
\midrule
Average     & 92.8 & 95.4 & 2.9 & 89.7 \\
\bottomrule
\end{tabular}
\vspace{-0.05in}
\caption{
Quality audit of constructed neighborhoods. All values are reported as percentages.
}
\label{tab:neighborhood_quality}
\vspace{-0.05in}
\end{table}

\section{Conclusion}
In this work, we revisited ripple effects in knowledge editing as coupled positive and negative responses of local neighborhood structure. Our analysis identified two design pressures: editable-side coordination and preserved-side leakage, motivating Joint Neighborhood Optimization (JNO). JNO constructs structured neighborhoods, applies Pressure-Aware Coordination (PAC) to jointly optimize target representations, and adopts a semantic pre-execution gate to avoid risky updates. Experiments show that JNO improves positive propagation, suppresses negative perturbation, and preserves editing stability. Future work will extend pressure-aware neighborhood planning to
multilingual, long-context, and broader continual-editing settings.

\section*{Limitations}
JNO mitigates ripple effects through pressure-aware target
planning, rather than exhaustively resolving all potentially affected
variable facts. In our experiments, editable and preserved
neighborhoods are constructed from rule-guided propagation and
locality-related preserved facts. Their coverage and quality may vary
across domains, relation types, and context-dependent edits, and finite
neighborhoods cannot guarantee correct propagation to arbitrary unseen
facts. Stronger mining, retrieval-based construction, multi-hop
discovery, or externally verified constraints may improve neighborhood
construction; holding out additional variable facts can strengthen
generalization assessment. The semantic pre-execution gate trades full
execution coverage for safer updates by rejecting target plans that
fail the semantic-loss check. For mandatory-execution settings, JNO can
use its accept-all variant or be paired with external fallback
strategies. PAC introduces activation-space optimization overhead,
which can be mitigated through caching, neighborhood pruning, or
approximate scoring.

\bibliography{custom}

\appendix
\section*{Appendix}

\section{Experimental Details for Quantitative Analysis of Ripple Effects in Sec.~\ref{sec:revisiting}}
\label{sec:appendix_empirical_setup}

\subsection{PRR and NRR Diagnostic Evaluation}
\label{subsec:app_datasets}

For each edit $e=(s,r,o\rightarrow o^*)$, we construct a local neighborhood $\mathcal{C}=\{e\}\cup \mathcal{C}_c\cup \mathcal{C}_s$, where $\mathcal{C}_c$ contains editable facts that should co-update with the target edit, and $\mathcal{C}_s$ contains preserved facts that should remain unchanged after editing (details are provided in Appendix~\ref{sec:neighborhood_construction}). In the diagnostic analysis, these neighborhoods are used only for post-edit evaluation of baseline editors. We evaluate GPT-J with five representative methods: ROME, MEMIT, PMET, AlphaEdit, and GLAME. For each edit request, each editor is applied only to the target edit $e$, yielding an edited model $f^+$. We then evaluate $f^+$ on the constructed neighborhood, compute PRR and NRR for each edit instance according to Eq.~\ref{eq:ripple_rate}, and average them over the edit pool. Table~\ref{tab:prr-nrr} reports the averaged diagnostic rates. Thus, the analysis uses RippleEdits edit requests but re-aggregates outcomes over our constructed editable and preserved neighborhoods.

\subsection{Structural Factors and Probe Definitions}
\label{subsec:app_metrics}

To analyze why ripple responses vary across edit instances, we define three structural probes of the local neighborhood: key-space coupling, pre-trained entanglement, and local neighborhood sensitivity, which are used for explanatory analysis.

\paragraph{Key-space coupling.}
Key-space coupling measures edited-layer proximity between an edit and its neighbor. For the edited fact $(s,r,o)$ and a neighboring fact $(s',r',o')$, we extract their key vectors $k_{(s,r)}$ and $k_{(s',r')}$ at the editing layer and define
\begin{equation}
\small
\mathrm{Score}_k\big((s,r),(s',r')\big)
=
\cos\big(k_{(s,r)},k_{(s',r')}\big).
\end{equation}

Higher values indicate stronger key-space proximity, suggesting that the edit-induced update is more likely to transfer to the neighbor. Such transfer can support positive propagation when the neighboring target is compatible with the edit, but can also create coordination tension when tightly coupled facts require different target movements.

To interpret the non-monotonic PRR pattern over key-space coupling, we also compute an auxiliary target-side discrepancy:
\begin{equation}
\small
\mathrm{Score}_{\tau}\big((s,r),(s',r')\big)
=
\|v_{(s,r)}^*-v_{(s',r')}^*\|_2.
\end{equation}
Here, $v_{(s,r)}^*$ and $v_{(s',r')}^*$ denote the desired post-edit target representations associated with the edited fact and the neighboring fact. This discrepancy is used only as an auxiliary diagnostic to explain when strong key coupling supports propagation or induces editable-side coordination tension. It is not used as an independent grouping axis in Fig.~\ref{fig:ripple-factors}.

\paragraph{Pre-trained entanglement.}
Pre-trained entanglement measures parameter-space co-movement inherited from pre-training. Following gradient-based diagnostics, we use gradient cosine similarity as an operational proxy. Given an edited fact $(s,r,o)$ and a related fact $(s',r',o')$, we compute
\begin{equation}
\begin{aligned}
\mathrm{Score}_{\eta}\big((s,r,o),(s',r',o')\big)
&\\=
\cos\big(
\nabla_\theta \log P_\theta(o \mid s,r),\,&\\
\nabla_\theta \log P_\theta(o' \mid s',r')
\big).
\end{aligned}
\end{equation}
The gradients are computed on the pre-edit model with respect to the edited-layer parameters, using the same parameter scope for all facts in the analysis. Higher values indicate stronger parameter-level coupling, meaning that the two facts tend to move together under local parameter updates.

\paragraph{Local neighborhood sensitivity.}
Local neighborhood sensitivity measures the internal compactness of
the preserved region associated with each edit. For an edit $e$, let
$\{k_j\}_{j=1}^{N_s}$ denote the key vectors of preserved samples in
$\mathcal{C}_s(e)$, with $N_s\geq 2$ in our diagnostic pool. We define
\begin{equation}
\small
\lambda_e
=
\left(
\frac{1}{N_s}\sum_{j=1}^{N_s}\|k_j-\bar{k}_s\|^2
\right)^{-1},
\end{equation}
where $\bar{k}_s$ is the mean key vector of preserved samples in $\mathcal{C}_s(e)$. Larger values indicate tighter clustering of preserved keys and higher local sensitivity. A dense preserved region is more likely to suffer collateral perturbation under local editing.

\subsection{Stratified Factor Analysis}
\label{subsec:app_stratified_protocol}
The stratified factor analysis examines how each structural probe co-varies with PRR and NRR. All probe values are computed on the pre-edit model and pre-edit neighborhoods before applying any editing method. For each probe, we compute the grouping thresholds once over the full edit pool under GPT-J and apply the same thresholds to all five editors. This yields a shared structural stratification, making the factor-wise response curves directly comparable across editors.

For pairwise probes, including key-space coupling and pre-trained entanglement, we first compute an edit--neighbor score for each neighbor. Let $z_{e,j}$ denote the probe score between edit $e$ and neighbor $j$, where $z_{e,j}$ is instantiated as $\mathrm{Score}_k(e,j)$ for key-space coupling and $\mathrm{Score}_{\eta}(e,j)$ for pre-trained entanglement. Since PRR and NRR are measured on different neighborhood subsets, we aggregate the pairwise scores separately for the editable and preserved sides:
\begin{equation}
\small
z_e^{c}
=
\frac{1}{|\mathcal{C}_c(e)|}
\sum_{j\in \mathcal{C}_c(e)} z_{e,j},
z_e^{s}
=
\frac{1}{|\mathcal{C}_s(e)|}
\sum_{j\in \mathcal{C}_s(e)} z_{e,j}.
\end{equation}
Here, $z_e^{c}$ is used to stratify PRR because PRR is measured on editable neighbors in $\mathcal{C}_c(e)$, while $z_e^{s}$ is used to stratify NRR because NRR is measured on preserved neighbors in $\mathcal{C}_s(e)$. For each pairwise probe, we rank all edit instances by the corresponding edit-level score and split them into three equally sized tercile groups. The bottom third is labeled \textsc{Low}, the middle third is labeled \textsc{Medium}, and the top third is labeled \textsc{High}. This procedure is applied separately to $z_e^{c}$ for PRR stratification and to $z_e^{s}$ for NRR stratification.
For local neighborhood sensitivity, each edit has a preserved-region density score $\lambda_e$. We rank all edit instances by $\lambda_e$ and split them into three equally sized tercile groups: the bottom third is labeled \textsc{Sparse}, the middle third is labeled \textsc{Normal}, and the top third is labeled \textsc{Dense}. We average PRR and NRR across edit instances. For an editor $m$ and a regime $\mathcal{R}$, the averaged diagnostic rates are
\begin{equation}
\small
\begin{aligned}
   \overline{\mathrm{PRR}}_{m,\mathcal{R}}
=
\frac{1}{|\mathcal{R}|}
\sum_{e\in\mathcal{R}}
\mathrm{PRR}_{m}(e),&\\
\overline{\mathrm{NRR}}_{m,\mathcal{R}}
=
\frac{1}{|\mathcal{R}|}
\sum_{e\in\mathcal{R}}
\mathrm{NRR}_{m}(e).
\end{aligned}
\end{equation}
The resulting factor-wise response curves are reported in Fig.~\ref{fig:ripple-factors}. This protocol fixes the edit pool, the pre-edit structural grouping, and the editor set, allowing us to inspect which local neighborhood properties co-vary with positive propagation and negative perturbation.

\section{Proofs for Theoretical Analysis}
\label{sec:appendix_proofs}

This appendix provides the proofs deferred from the main text and auxiliary results connecting PAC to residual variation and Reliable Execution.

\subsection{Proof of Theorem~\ref{thm:editable_tension}}
\label{subsec:proof_editable_tension}

\begin{proof}
Since
\begin{equation}
\boldsymbol{V}_{\mathcal{V}}
=
(\boldsymbol{W}+\boldsymbol{\Delta})\boldsymbol{K}_{\mathcal{V}},
\end{equation}
we have
\begin{equation}
\small
\boldsymbol{V}_{\mathcal{V}}\boldsymbol{L}_s\boldsymbol{V}_{\mathcal{V}}^T
=
(\boldsymbol{W}+\boldsymbol{\Delta})
\boldsymbol{K}_{\mathcal{V}}\boldsymbol{L}_s\boldsymbol{K}_{\mathcal{V}}^T
(\boldsymbol{W}+\boldsymbol{\Delta})^T.
\end{equation}
Because $\boldsymbol{L}_s\succeq 0$, the matrix
$\boldsymbol{M}=\boldsymbol{K}_{\mathcal{V}}\boldsymbol{L}_s\boldsymbol{K}_{\mathcal{V}}^T$
is positive semidefinite. For any positive semidefinite matrix $\boldsymbol{M}$ and any matrix $\boldsymbol{A}$,
\begin{equation}
\operatorname{tr}(\boldsymbol{A}\boldsymbol{M}\boldsymbol{A}^T)
\le
\|\boldsymbol{A}\|_{\mathrm{op}}^2 \operatorname{tr}(\boldsymbol{M}).
\end{equation}
Applying this inequality with
$\boldsymbol{A}=\boldsymbol{W}+\boldsymbol{\Delta}$ gives
\begin{equation}
\small
\operatorname{tr}\!\big(\boldsymbol{V}_{\mathcal{V}}\boldsymbol{L}_s\boldsymbol{V}_{\mathcal{V}}^T\big)
\le
\|\boldsymbol{W}+\boldsymbol{\Delta}\|_{\mathrm{op}}^2
\operatorname{tr}\!\big(\boldsymbol{K}_{\mathcal{V}}\boldsymbol{L}_s\boldsymbol{K}_{\mathcal{V}}^T\big),
\end{equation}
which proves Eq.~(\ref{eq:graph_executability_bound}).

Let
\begin{equation}
t=\operatorname{tr}\!\big(\boldsymbol{K}_{\mathcal{V}}\boldsymbol{L}_s\boldsymbol{K}_{\mathcal{V}}^T\big).
\end{equation}
If $t>0$, Eq.~(\ref{eq:graph_executability_bound}) implies
\begin{equation}
\small
\|\boldsymbol{W}+\boldsymbol{\Delta}\|_{\mathrm{op}}^2
\ge
\frac{\operatorname{tr}\!\big(\boldsymbol{V}_{\mathcal{V}}\boldsymbol{L}_s\boldsymbol{V}_{\mathcal{V}}^T\big)}{t}
\ge
\frac{\operatorname{tr}\!\big(\boldsymbol{V}_{\mathcal{V}}\boldsymbol{L}_s\boldsymbol{V}_{\mathcal{V}}^T\big)}{t+\epsilon_{\mathrm{num}}}.
\end{equation}
If $t=0$, then Eq.~(\ref{eq:graph_executability_bound}) gives
$\operatorname{tr}(\boldsymbol{V}_{\mathcal{V}}\boldsymbol{L}_s\boldsymbol{V}_{\mathcal{V}}^T)=0$, so the stabilized ratio is zero. Therefore a numerically stabilized ratio form
$\|\boldsymbol{W}+\boldsymbol{\Delta}\|_{\mathrm{op}}^2
\ge
\frac{\operatorname{tr}\!\big(\boldsymbol{V}_{\mathcal{V}}\boldsymbol{L}_s\boldsymbol{V}_{\mathcal{V}}^T\big)}{\operatorname{tr}\!\big(\boldsymbol{K}_{\mathcal{V}}\boldsymbol{L}_s\boldsymbol{K}_{\mathcal{V}}^T\big)+\epsilon_{\mathrm{num}}}, \epsilon_{\mathrm{num}}>0$
holds in all cases.
\end{proof}

\begin{corollary}[Residual compatibility under a shared update]
\label{cor:residual_compatibility}
Let
$\boldsymbol{R}_{\mathcal{V}}
=
\boldsymbol{V}_{\mathcal{V}}-\boldsymbol{W}\boldsymbol{K}_{\mathcal{V}}.$
Suppose there exists a local update $\boldsymbol{\Delta}$ such that
$\boldsymbol{\Delta}\boldsymbol{K}_{\mathcal{V}}=\boldsymbol{R}_{\mathcal{V}}$. Then
\begin{equation}
\small
\operatorname{tr}\!\big(\boldsymbol{R}_{\mathcal{V}}\boldsymbol{L}_s\boldsymbol{R}_{\mathcal{V}}^T\big)
\le
\|\boldsymbol{\Delta}\|_{\mathrm{op}}^2
\operatorname{tr}\!\big(\boldsymbol{K}_{\mathcal{V}}\boldsymbol{L}_s\boldsymbol{K}_{\mathcal{V}}^T\big).
\label{eq:residual_executability_bound}
\end{equation}
\end{corollary}

\begin{proof}
Since $\boldsymbol{R}_{\mathcal{V}}=\boldsymbol{\Delta}\boldsymbol{K}_{\mathcal{V}}$,
\begin{equation}
\boldsymbol{R}_{\mathcal{V}}\boldsymbol{L}_s\boldsymbol{R}_{\mathcal{V}}^T
=
\boldsymbol{\Delta}
\boldsymbol{K}_{\mathcal{V}}\boldsymbol{L}_s\boldsymbol{K}_{\mathcal{V}}^T
\boldsymbol{\Delta}^T.
\end{equation}
As above, $\boldsymbol{K}_{\mathcal{V}}\boldsymbol{L}_s\boldsymbol{K}_{\mathcal{V}}^T\succeq 0$. Applying
$\operatorname{tr}(\boldsymbol{A}\boldsymbol{M}\boldsymbol{A}^T)
\le
\|\boldsymbol{A}\|_{\mathrm{op}}^2\operatorname{tr}(\boldsymbol{M})$
with $\boldsymbol{A}=\boldsymbol{\Delta}$ proves Eq.~(\ref{eq:residual_executability_bound}).
\end{proof}

\begin{lemma}[Pairwise target penalty controls residual variation up to fixed geometry]
\label{lem:pairwise_residual_variation}
Let
\begin{equation}
\boldsymbol{r}_i=\boldsymbol{v}_i-\boldsymbol{W}\boldsymbol{k}_i.
\end{equation}
If
\begin{equation}
\small
\mathcal{L}_{\mathrm{pair}}
=
\sum_{0\le i<j\le m}
s_{ij}\|\boldsymbol{v}_i-\boldsymbol{v}_j\|_2^2
=
\operatorname{tr}\!\big(\boldsymbol{V}_{\mathcal{V}}\boldsymbol{L}_s\boldsymbol{V}_{\mathcal{V}}^T\big),
\end{equation}
then
\begin{equation}
\sum_{0\le i<j\le m}
s_{ij}\|\boldsymbol{r}_i-\boldsymbol{r}_j\|_2^2
\le
2\mathcal{L}_{\mathrm{pair}}
+
2C_{W,K},
\label{eq:pairwise_to_residual_bound}
\end{equation}
where
\begin{equation}
C_{W,K}
:=
\sum_{0\le i<j\le m}
s_{ij}\|\boldsymbol{W}(\boldsymbol{k}_i-\boldsymbol{k}_j)\|_2^2.
\end{equation}
\end{lemma}

\begin{proof}
For every pair $(i,j)$,
\begin{equation}
\boldsymbol{r}_i-\boldsymbol{r}_j
=
(\boldsymbol{v}_i-\boldsymbol{v}_j)
-
\boldsymbol{W}(\boldsymbol{k}_i-\boldsymbol{k}_j).
\end{equation}
Using $\|\boldsymbol{a}-\boldsymbol{b}\|_2^2\le 2\|\boldsymbol{a}\|_2^2+2\|\boldsymbol{b}\|_2^2$, we obtain
\begin{equation}
\small
\|\boldsymbol{r}_i-\boldsymbol{r}_j\|_2^2
\le
2\|\boldsymbol{v}_i-\boldsymbol{v}_j\|_2^2
+
2\|\boldsymbol{W}(\boldsymbol{k}_i-\boldsymbol{k}_j)\|_2^2.
\end{equation}
Multiplying by $s_{ij}$ and summing over $0\le i<j\le m$ proves Eq.~(\ref{eq:pairwise_to_residual_bound}).
\end{proof}

The term $C_{W,K}$ depends only on the current edited-layer weight and editable keys, not on the optimized targets. Hence the pairwise target penalty also controls the spatial variation of the induced residual plan up to a fixed geometry term.

\subsection{Proof of Proposition~\ref{prop:preserved_leakage}}
\label{subsec:proof_preserved_exposure}

\begin{proof}
From
\begin{equation}
\boldsymbol{k}_n
=
a_{ni}\boldsymbol{k}_i+\boldsymbol{e}_{ni},
\end{equation}
we obtain
\begin{equation}
\boldsymbol{\Delta}\boldsymbol{k}_n
=
a_{ni}\boldsymbol{\Delta}\boldsymbol{k}_i+\boldsymbol{\Delta}\boldsymbol{e}_{ni}
=
a_{ni}\boldsymbol{r}_i+\boldsymbol{\Delta}\boldsymbol{e}_{ni}.
\end{equation}
Thus
\begin{equation}
a_{ni}\boldsymbol{r}_i
=
\boldsymbol{\Delta}\boldsymbol{k}_n-\boldsymbol{\Delta}\boldsymbol{e}_{ni}.
\end{equation}
Taking norms and using the triangle inequality gives
\begin{equation}
|a_{ni}|\,\|\boldsymbol{r}_i\|_2
\le
\|\boldsymbol{\Delta}\boldsymbol{k}_n\|_2+\|\boldsymbol{\Delta}\boldsymbol{e}_{ni}\|_2.
\end{equation}
By assumption,
\begin{equation}
\begin{aligned}
\|\boldsymbol{\Delta}\boldsymbol{k}_n\|_2
&\le \tau_n, \\
\|\boldsymbol{\Delta}\boldsymbol{e}_{ni}\|_2
&\le \|\boldsymbol{\Delta}\|_{\mathrm{op}}\|\boldsymbol{e}_{ni}\|_2
\le B\|\boldsymbol{e}_{ni}\|_2.
\end{aligned}
\end{equation}
Therefore,
\begin{equation}
|a_{ni}|\,\|\boldsymbol{r}_i\|_2
\le
\tau_n+B\|\boldsymbol{e}_{ni}\|_2,
\end{equation}
which yields
\begin{equation}
\|\boldsymbol{r}_i\|_2
\le
\frac{\tau_n+B\|\boldsymbol{e}_{ni}\|_2}{|a_{ni}|}.
\end{equation}
\end{proof}

\paragraph{Relation between $d_i$ and the largest residual-amplifying coefficient.}

Define
\begin{equation}
A_i:=\max\nolimits_{n\in\mathcal{C}_s}|a_{ni}|.
\end{equation}
For every preserved sample $n$,
\begin{equation}
|a_{ni}|
=
\frac{|\boldsymbol{k}_n^T\boldsymbol{k}_i|}{\|\boldsymbol{k}_i\|_2^2}
=
\frac{\|\boldsymbol{k}_n\|_2}{\|\boldsymbol{k}_i\|_2}
\cdot
\frac{|\boldsymbol{k}_n^T\boldsymbol{k}_i|}
{\|\boldsymbol{k}_n\|_2\|\boldsymbol{k}_i\|_2}.
\end{equation}
If the key-norm ratios are bounded as
\begin{equation}
c_{\min}\|\boldsymbol{k}_i\|_2
\le
\|\boldsymbol{k}_n\|_2
\le
c_{\max}\|\boldsymbol{k}_i\|_2,
\end{equation}
then
\begin{equation}
c_{\min}d_i\le A_i\le c_{\max}d_i.
\end{equation}
Thus $d_i$ is equivalent, up to bounded norm-ratio constants, to the
largest alignment coefficient with preserved directions, supporting its
use as an exposure surrogate. The corollary gives an exposure-dependent bound under exact
preservation.

\begin{corollary}[Exposure bound under exact preservation]
\label{cor:normalized_exposure_budget}
Let
\[
n^\star\in\arg\max_{n\in\mathcal{C}_s}
\frac{|\boldsymbol{k}_i^T\boldsymbol{k}_n|}
{\|\boldsymbol{k}_i\|_2\|\boldsymbol{k}_n\|_2}.
\]
Under the assumptions of Proposition~\ref{prop:preserved_leakage}, if
$\boldsymbol{\Delta}\boldsymbol{k}_{n^\star}=0$, then
\begin{equation}
\small
\|\boldsymbol{r}_i\|_2
\le
B\|\boldsymbol{k}_i\|_2\sqrt{1-d_i^2}
\le
B\|\boldsymbol{k}_i\|_2.
\label{eq:normalized_exposure_budget}
\end{equation}
Thus, for fixed $\|\boldsymbol{k}_i\|_2$, the admissible residual budget
tightens monotonically with $d_i\in[0,1]$.
\end{corollary}

\begin{proof}
Let
\[
\boldsymbol{u}_{n^\star}
=
\frac{\boldsymbol{k}_{n^\star}}
{\|\boldsymbol{k}_{n^\star}\|_2},
\qquad
b_i=\boldsymbol{u}_{n^\star}^T\boldsymbol{k}_i .
\]
Then
$|b_i|=\|\boldsymbol{k}_i\|_2 d_i.$

Decompose
\[
\boldsymbol{k}_i
=
b_i\boldsymbol{u}_{n^\star}
+
\boldsymbol{e}_i^\perp,
\qquad
\boldsymbol{e}_i^\perp\perp\boldsymbol{u}_{n^\star},
\]
so
$\|\boldsymbol{e}_i^\perp\|_2
=
\|\boldsymbol{k}_i\|_2\sqrt{1-d_i^2}.$
Since
$\boldsymbol{\Delta}\boldsymbol{k}_{n^\star}=0$,
we also have
$\boldsymbol{\Delta}\boldsymbol{u}_{n^\star}=0$.
Hence
\[
\boldsymbol{r}_i
=
\boldsymbol{\Delta}\boldsymbol{e}_i^\perp,
\]
and therefore
\[
\|\boldsymbol{r}_i\|_2
\le
B\|\boldsymbol{k}_i\|_2\sqrt{1-d_i^2}.
\]
\end{proof}

\subsection{Realized-Update Diagnostics}
\label{app:realized_update}
{
We examine whether $\mathcal{L}_{\mathrm{pair}}$ improves residual
compatibility and reduces realized update magnitude. JNO and JNO w/o
$\mathcal{L}_{\mathrm{pair}}$ are compared on all 839
RippleEdits-POPULAR requests using LLaMA3-8B. For request-matched
comparison, semantic gating is disabled for both variants.
}
For each editable sample $i$, let
$\hat{\boldsymbol{r}}_i=\hat{\boldsymbol{v}}_i-\boldsymbol{W}\boldsymbol{k}_i$
denote the residual induced by its optimized target representation.
We measure the coupled residual variation as
\begin{equation}
    \mathrm{CRV}
    =
    \sum\nolimits_{0 \leq i < j \leq m}
    s_{ij}
    \left\|
        \hat{\boldsymbol{r}}_i - \hat{\boldsymbol{r}}_j
    \right\|_2^2,
    \label{eq:crv}
\end{equation}
which measures residual inconsistency among tightly coupled editable
samples. We additionally measure the operator norm of the realized update,
\begin{equation}
    \left\|\boldsymbol{\Delta}^{\ast}\right\|_{\mathrm{op}}
    =
    \sigma_{\max}\left({\boldsymbol{\Delta}}^{\ast}\right),
    \label{eq:realized_update_norm}
\end{equation}
where $\boldsymbol{\Delta}^{\ast}=\widetilde{\boldsymbol\Delta}^{*}\boldsymbol P$ denotes the realized parameter
update produced by the Reliable Execution stage, and
$\sigma_{\max}(\cdot)$ is the largest singular value.

\begin{table}[t]
    \centering
    \small
    \setlength{\tabcolsep}{7pt}
    \begin{tabular}{lcc}
        \toprule
        Method
        & $\mathrm{CRV}\downarrow$
        & $\|\boldsymbol{\Delta}^{\ast}\|_{\mathrm{op}}\downarrow$ \\
        \midrule
        JNO w/o $\mathcal{L}_{\mathrm{pair}}$
        & 0.85
        & 2.43 \\
        JNO
        & \textbf{0.68}
        & \textbf{1.96} \\
        \bottomrule
    \end{tabular}
    \vspace{-0.1in}
    \caption{
        Realized-update diagnostics of
        $\mathcal{L}_{\mathrm{pair}}$.
    }\label{tab:realized_update}
\end{table}

Table~\ref{tab:realized_update} shows that
$\mathcal{L}_{\mathrm{pair}}$ reduces CRV and the realized update norm
by 20.0\% and 19.3\%, respectively. These results
indicate that $\mathcal{L}_{\mathrm{pair}}$ produces a more compatible
residual plan that can be realized through a smaller shared parameter
update, providing execution-level evidence for the target-space
compatibility analysis.

\subsection{Boundary of Finite-Neighborhood Generalization}
\label{app:finite_neighborhood}

{JNO aims to mitigate ripple effects under available neighborhood
information rather than discover or resolve all potentially affected
facts in an open-world setting. Moreover, a finite neighborhood cannot
guarantee correct propagation to arbitrary unseen variable facts
without an additional coverage assumption. Specifically, let
$\mathcal{C}=\{\boldsymbol{e}\}\cup\mathcal{C}_c\cup\mathcal{C}_s$
be the observed neighborhood and let $u$ be an unobserved variable
fact. Without a constraint relating $u$ to $\mathcal{C}$, two local
editing maps may satisfy
\[
f_1(\boldsymbol{k}(c))=f_2(\boldsymbol{k}(c)),
\qquad \forall c\in\mathcal{C},
\]
while
\[
f_1(\boldsymbol{k}_u)\neq f_2(\boldsymbol{k}_u).
\]
Thus, the two maps are indistinguishable from the observed neighborhood,
and correct propagation to $u$ cannot be guaranteed from
$\mathcal{C}$ alone. JNO therefore targets improved propagation and
preservation under reasonably available neighborhood information,
rather than claiming complete coverage of all variable facts.}

\section{Neighborhood Construction Details}
\label{sec:neighborhood_construction}

As illustrated in Fig.~\ref{fig:5}, the construction pipeline
transforms each edit request into a structured neighborhood for target
planning. The first part (left, blue panel) generates
\textbf{variable neighbors} (variable set $\mathcal{C}_c$) as
coordinated editable targets, while the second part (right, green
panel) constructs \textbf{preserved neighbors} (preserved set
$\mathcal{C}_s$) as preservation constraints under local parameter
coupling. Taking the edit \textit{``USA's President changes from Biden to Trump''} as our running example, we demonstrate how a structured neighborhood $\mathcal{C} = \{e\} \cup \mathcal{C}_c \cup \mathcal{C}_s$ is systematically assembled around the original edit request $e = (s, r, o \rightarrow o^*)$.

\begin{figure*}[t]
	\centering
	\includegraphics[width=0.83\linewidth]{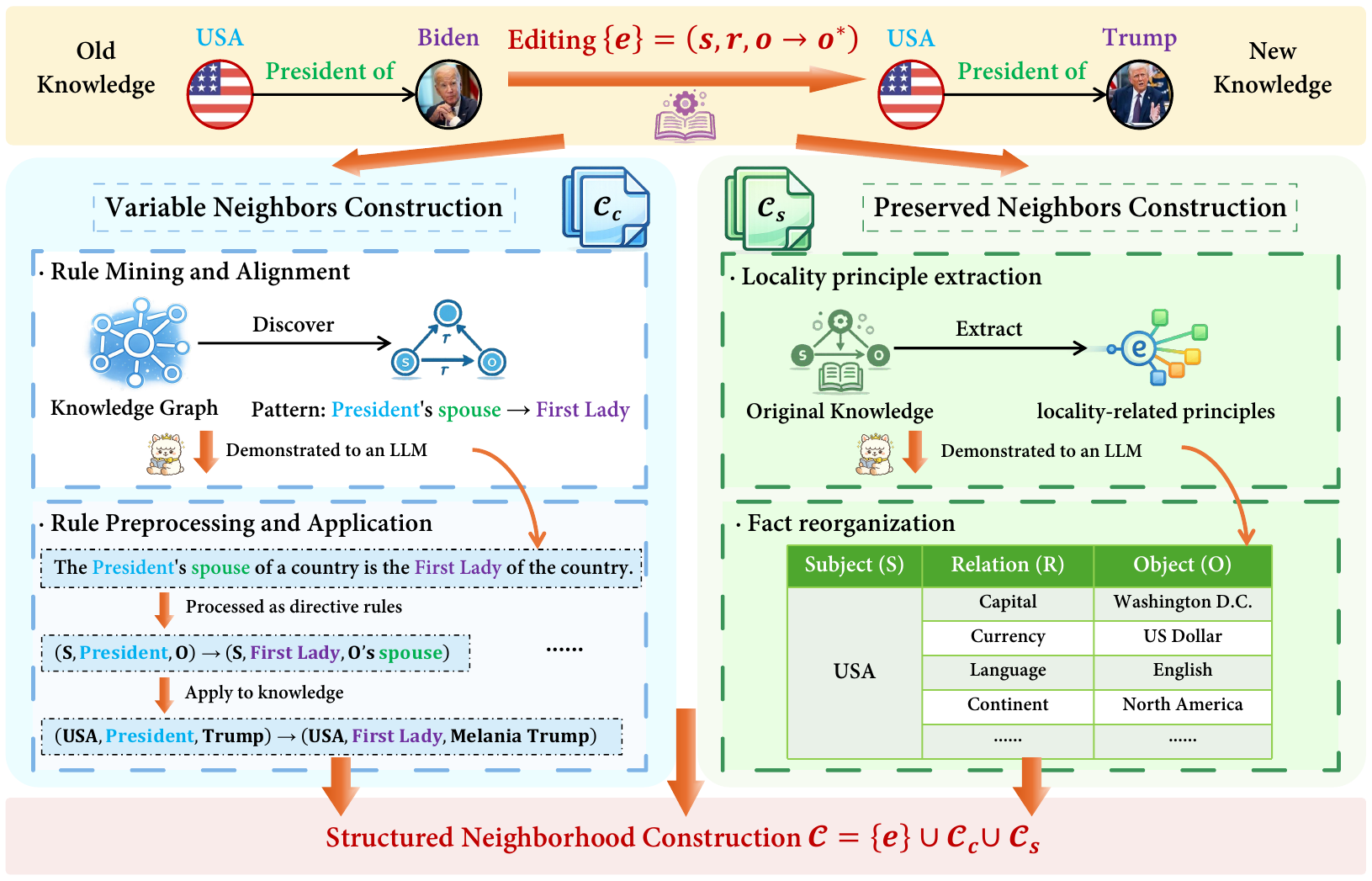}\vspace{-0.1in}
	\caption{Structured neighborhood construction for the edit \textit{``USA's President changes from Biden to Trump''}. Left: variable neighbors $\mathcal{C}_c$ are generated via logical rules (e.g., \textsc{President}'s \textsc{spouse} $\rightarrow$ \textsc{First Lady}). Right: preserved neighbors $\mathcal{C}_s$ are extracted under the locality principle (e.g., \textsc{Capital}, \textsc{Currency}).}
	\vspace{-0.2in}
	\label{fig:5}
\end{figure*}

\subsection{Variable Neighborhood Construction}
\label{subsec:variable_construction}

The variable neighbors construction panel (left side of Fig.~\ref {fig:5}) follows a logical-rule-guided generation paradigm to populate the consistency set $\mathcal{C}_c$. This process chains logically associated knowledge that should co-update with the target edit.

\paragraph{Variable-neighbor generation.}
We follow the logical rule-guided chain construction of \citealp{dong2025chainedit}. ChainEdit mines inverse and multi-hop relational rules from Wikidata~\citep{vrandevcic2014wikidata}, aligns candidate rules with LLM-recognized logical patterns, converts retained rules into directive templates, and applies matched templates to derive facts to be co-updated. We adopt this procedure to generate the initial variable candidate set $\mathcal{C}_c$ for each edit request.

\paragraph{Rule Preprocessing and Application.}
Each retained rule is normalized into an edit template. For the President relation, the logical pattern is converted into the directive $(S, \text{President}, O) \rightarrow (S, \text{First Lady}, O\text{'s spouse})$. Applying this template to the example edit $(\text{USA}, \text{President}, \text{Trump})$ yields $(\text{USA}, \text{First Lady}, \text{Melania Trump})$. If a template requires an intermediate entity, we follow ChainEdit by querying the LLM to resolve the instantiated relation. This rule-guided propagation ensures $\mathcal{C}_c$ captures not superficial semantic similarity, but logically induced knowledge that must consistently evolve with the original edit.

\subsection{Preserved Neighborhood Construction}
\label{subsec:preserved_construction}

The preserved-neighbor construction panel (right side of Fig.~\ref{fig:5}) populates preserved specificity set $\mathcal{C}_s$ under local consistency. Unlike variable neighbors, these facts remain semantically independent of the intended edit path while vulnerable to collateral interference due to parameter coupling.

\paragraph{Locality Principle Extraction.} From the original locality knowledge of the edited entity (USA), we design prompts to extract locality-related relations logically independent of the edited relation. These relations share the subject and local context with the target fact, including \textsc{Capital}, \textsc{Currency}, \textsc{Language}, and \textsc{Continent}, while remaining outside the desired propagation chain.

\paragraph{Fact Reorganization.} The extracted locality-related relations are reorganized into preserved factual triples. For the USA example, the preserved neighborhood includes: (USA, Capital, Washington D.C.), (USA, Currency, US Dollar), (USA, Language, English), and (USA, Continent, North America). These facts are locally close enough to expose leakage risk, yet semantically distant enough to require preservation rather than co-update. Thus, $\mathcal{C}_s$ provides preservation constraints that anchor the edit's specificity boundary.

\subsection{Resulting Neighborhood Group}
\label{subsec:resulting_group}

Assembling both parts yields a \textbf{Structured Neighborhood} 
$\mathcal{C}=\{\boldsymbol{e}\}\cup\mathcal{C}_c\cup\mathcal{C}_s$  that transforms 
isolated editing into neighborhood-aware optimization. 
For our presidential edit example:

\begin{itemize}[leftmargin=*,nosep]
    \item \textbf{Original edit} $e$:\\ 
    (USA, President, Biden $\to$ Trump)
    \item \textbf{Variable neighbors} $\mathcal{C}_c$:\\ 
    (USA, First Lady, Jill Biden $\to$ Melania Trump) and other logically derived facts
    \item \textbf{Preserved neighbors} $\mathcal{C}_s$:\\ 
    (USA, Capital, Washington D.C.), (USA, Currency, US Dollar), etc.
\end{itemize}

The variable subset provides coordinated propagation targets, while
the preserved subset defines the preservation boundary during target
planning. For each benchmark request, the resulting neighborhood
transforms an isolated edit into a local editing problem with explicit
propagation targets and preservation constraints. The edited model is
then evaluated using the benchmark's original queries and metrics.

\paragraph{Data separation.}
{For each edit request, JNO constructs variable and preserved neighbors
using predefined logical rules, external knowledge, and local context.
Apart from the current edit request, no
benchmark evaluation query or label is used for neighborhood
construction or optimization. This follows a common task-defined
construction principle: ChainEdit mines logical rules from external knowledge
graphs to construct related edit sets \citep{dong2025chainedit}.
Existing methods such as MEND and KnowledgeEditor likewise construct
training supervision according to predefined editing objectives and
evaluate on held-out test instances
\citep{mitchellfast,de2021editing}. Thus, sharing
benchmark-defined relation or ripple types reflects a common public
task specification rather than overlap in concrete test instances or
labels.}

\subsection{Neighborhood Quality Audit}
\label{subsec:Neighborhood_Verification}

Structured neighborhood construction automatically generates variable and preserved neighbors, but the resulting neighborhoods may contain noise. A variable neighbor may be semantically associated with the edit rather than logically entailed by it, while a preserved neighbor may in fact be affected by the edit. We audit whether the constructed candidates encode reliable post-edit requirements.

\begin{figure*}[p]
\centering
\begin{tcolorbox}[
    width=0.85\textwidth,
    colback=blue!3,
    colframe=blue!70!black,
    coltitle=white,
    colbacktitle=blue!70!black,
    title=\textbf{Prompt for Neighborhood Quality Audit},
    fonttitle=\bfseries\large,
    boxrule=0.8pt,
    arc=2pt,
    left=7pt,
    right=7pt,
    top=7pt,
    bottom=7pt
]
\small
\linespread{1.08}\selectfont
\setlength{\parskip}{2.5pt}

\textbf{\large Task Description:}

Given an edit fact \(e=(s,r,o\rightarrow o^{*})\) and a candidate neighbor \(c\), verify whether the candidate is correctly assigned as a variable or preserved neighbor. The candidate should be labeled as a valid variable neighbor, a valid preserved neighbor, an ambiguous neighbor, or an invalid neighbor.

\vspace{0.35em}
\textbf{\large Evaluation Criteria:}

\textbf{1. Variable Neighbor.}
If \(c=(s_i,r_i,o_i\rightarrow o_i^{*})\), verify whether:

\begin{itemize}[leftmargin=1.3em, itemsep=2pt, topsep=2pt]
    \item the edit \(e\) implies that \(c\) should also be updated;
    \item the target object \(o_i^{*}\) is factually correct;
    \item the relation \(r_i\) lies on the logical propagation chain of \(r\), rather than being merely associated.
\end{itemize}

Judge \(c\) as a valid variable neighbor only if all three conditions hold.

\vspace{0.35em}
\textbf{2. Preserved Neighbor.}
If \(c=(s_j,r_j,o_j)\), verify whether:

\begin{itemize}[leftmargin=1.3em, itemsep=2pt, topsep=2pt]
    \item the fact should remain unchanged after applying \(e\);
    \item the relation \(r_j\) is independent of the logical consequences of \(r\);
    \item the fact \((s_j,r_j,o_j)\) is factually correct.
\end{itemize}

Judge \(c\) as a valid preserved neighbor only if all three conditions hold.

\vspace{0.35em}
\textbf{\large Decision Labels:}

\textbf{Valid Variable Neighbor}, \textbf{Valid Preserved Neighbor}, \textbf{Ambiguous Neighbor}, or \textbf{Invalid Neighbor}.

\vspace{0.35em}
\textbf{\large Example Inputs and Outputs:}

\textbf{Example 1:}
Edit fact:
\(e=(\text{USA},\text{President},\text{Biden}\rightarrow\text{Trump})\).
Candidate:
\(c=(\text{USA},\text{First Lady},\text{Jill Biden}\rightarrow\text{Melania Trump})\).

\textbf{Output:}
Valid Variable Neighbor.
The candidate is entailed by the edit, the target object is correct, and the relation lies on the propagation chain.

\vspace{0.35em}
\textbf{Example 2:}
Edit fact:
\(e=(\text{USA},\text{President},\text{Biden}\rightarrow\text{Trump})\).
Candidate:
\(c=(\text{USA},\text{Capital},\text{Washington D.C.})\).

\textbf{Output:}
Valid Preserved Neighbor.
The capital should remain unchanged, the relation is independent of the edit, and the fact is correct.

\vspace{0.35em}
\textbf{\large Your Task:}

You will receive one candidate neighbor produced by the neighborhood construction process. The candidate may have been initially assigned as either variable or preserved. Verify this assignment using the criteria above.

\begin{itemize}[leftmargin=1.3em, itemsep=2pt, topsep=2pt]
    \item Edit fact: \(e=(s,r,o\rightarrow o^{*})\)
    \item Candidate neighbor: \(c=\{\text{candidate\_neighbor}\}\)
    \item Initial assignment: \(\text{type}\in\{\text{variable},\text{preserved}\}\)
\end{itemize}

Please output:
the final label, the criterion-level judgments, and a brief explanation. A candidate should be marked valid only when it satisfies the criteria for its initial assignment; otherwise, mark it as invalid or ambiguous and briefly explain the mismatch.

\end{tcolorbox}
\caption{
Prompt used for auditing the quality of sampled variable and preserved neighbors.
}
\label{fig:neighborhood-verification-prompt}
\end{figure*}

For the audit in Sec.~\ref{sec:exp:neighborhood_quality}, we randomly sample 200 constructed neighborhoods per dataset. For each neighborhood, we inspect up to three candidates from $\mathcal{C}_c$ and up to three candidates from $\mathcal{C}_s$, sampling without replacement when more candidates are available. A fixed LLM judge assesses each candidate conditioned on the edit request using the prompt in Fig.~\ref{fig:neighborhood-verification-prompt}. Variable neighbors are evaluated by edit entailment, target-object correctness, and propagation dependency; preserved neighbors are evaluated by preservation validity, propagation independence, and factual correctness. The audit is used only for quality assessment and does not filter or modify the neighborhoods used in the main experiments.

Variable-neighbor precision $\mathrm{Prec}_c$ is the fraction of audited candidates from $\mathcal{C}_c$ judged as valid variable neighbors, and preserved-neighbor precision $\mathrm{Prec}_s$ is defined analogously for $\mathcal{C}_s$. A sampled neighborhood is marked as conflicting if any audited variable candidate and preserved candidate impose contradictory post-edit requirements. The audited valid-neighborhood rate is a neighborhood-level metric: a neighborhood is counted as valid when its audited variable candidates satisfy the co-update requirement, its audited preserved candidates satisfy the preservation requirement, and no conflict is detected between the two subsets.

Since JNO treats $\mathcal{C}_c$ and $\mathcal{C}_s$ as input constraints for pressure-aware target planning, stronger retrieval, human curation, or domain-specific mining can improve neighborhood quality without changing the PAC objective.

\begin{table}[t]
\centering
\footnotesize
\renewcommand{\arraystretch}{0.8}
\setlength{\tabcolsep}{6.5pt}
\begin{tabular}{lcccccc}
\toprule
\textbf{Noise} & \textbf{Rel.}$\uparrow$ & \textbf{LG}$\uparrow$ & \textbf{RE}$\uparrow$ & \textbf{SA}$\uparrow$ & \textbf{RS}$\uparrow$ & \textbf{FF}$\uparrow$ \\
\midrule
0\%  & 96.2 & 72.1 & 58.2 & 73.3 & 72.4 & 63.8 \\
10\% & 95.7 & 70.4 & 56.7 & 72.1 & 70.5 & 61.6 \\
20\% & 94.9 & 67.6 & 54.1 & 69.8 & 67.3 & 57.8 \\
30\% & 93.8 & 64.8 & 51.6 & 67.2 & 63.9 & 53.5 \\
\bottomrule
\end{tabular}
\vspace{-0.1in}
\caption{Sensitivity of JNO to neighborhood noise.}
\label{tab:neighborhood_noise_sensitivity}
\end{table}

\subsection{Sensitivity to Neighborhood Quality}
\label{subsec:neighborhood_quality_sensitivity}

To test JNO's robustness to imperfect neighborhoods, we conduct a
noise-injection study on RippleEdits with LLaMA3. Keeping edit requests,
labels, and evaluation data fixed, we independently replace each
candidate in $\mathcal{C}_c$ or $\mathcal{C}_s$ with probability
$\rho\in\{0\%,10\%,20\%,30\%\}$ by a same-type candidate from another
request in the same dataset, preserving neighborhood size and type
distribution. As shown in
Table~\ref{tab:neighborhood_noise_sensitivity}, JNO degrades gradually
as neighborhood noise increases. The drop at 10\% is small and close to
the audited invalid-neighborhood rate; ripple-aware performance remains
strong at 20\%, while degradation at 30\% is clear but non-catastrophic.
These results suggest that JNO does not require oracle-clean
neighborhoods, although higher-quality construction could further
improve pressure-aware target planning.

\section{Experimental Setup}
\label{sec:appendix_exp}

\subsection{Datasets}
\label{sec:appendix_datasets}
\paragraph{RippleEdits.}
RippleEdits~\citep{cohen2024evaluating} evaluates ripple effects beyond the directly edited fact. Each instance pairs a factual edit request with related queries that test propagation to logically affected knowledge and preservation of unrelated facts. We use it as the primary benchmark for positive ripple propagation and negative ripple suppression under structured neighborhood editing.

\paragraph{ZsRE.}
ZsRE~\citep{levy-etal-2017-zero} is a question-answering benchmark for factual knowledge editing. Each example contains a natural-language query for a subject--relation pair \((s_i,r_i)\) and a target answer \(o_i\). Following standard protocols~\citep{meng2022locating}, we evaluate efficacy on the original question, generalization on paraphrases, and specificity on out-of-scope locality questions.

\paragraph{CounterFact.}
CounterFact~\citep{meng2022locating} evaluates knowledge editing by replacing the object of a subject--relation association. Each instance provides the object and a counterfactual target for the subject--relation pair. We report efficacy, generalization, and specificity on original, paraphrased, and neighborhood prompts, together with quality measured by fluency and consistency.

\subsection{Evaluation Metrics}
\label{sec:appendix_metrics}

\subsubsection{Metrics on RippleEdits}
\label{sec:appendix_metrics_rippleedits}

For RippleEdits~\citep{cohen2024evaluating}, we follow ChainEdit~\citep{dong2025chainedit} and report Reliability (Rel.), Logical Generalization (LG), Reasoning (RE), Subject Aliasing (SA), Relation Specificity (RS), and Forgetfulness (FF). Rel. evaluates whether the edited model succeeds on the original edit request~\citep{wang-etal-2024-easyedit}.

LG measures whether facts implied by logical constraints are updated consistently with the edit. RE aggregates Compositionality I (CI) and Compositionality II (CII), evaluating two-hop composition over the edited fact. SA tests whether the update transfers to aliases of the edited subject. RS evaluates whether facts involving independent relations remain unchanged. FF follows the released RippleEdits naming of the Preservation criterion and measures whether valid objects under the edited subject--relation pair are retained.

\subsubsection{Metrics on ZsRE}
\label{sec:appendix_metrics_zsre}

For ZsRE, we report the standard editing metrics: \emph{efficacy}, \emph{generalization}, and \emph{specificity}~\citep{mengmass,mitchellfast,meng2022locating}.

\paragraph{Efficacy.}
Efficacy measures whether the edited model predicts the target answer for the original edit query. Let \(f_\theta\) be the edited model, \(x_i\) the original query of the \(i\)-th edit, and \(o_i\) the target object. We compute:
\begin{equation}
\mathbb{E}_i
\left[
\arg\max_o P_{f_\theta}(o \mid x_i) = o_i
\right].
\end{equation}

\paragraph{Generalization.}
Generalization measures whether the edited fact transfers to paraphrased queries expressing the same subject--relation semantics. Let \(N(x_i)\) denote the paraphrase set of \(x_i\). The metric is:
\begin{equation}
\mathbb{E}_i \ \mathbb{E}_{\tilde{x}_i \in N(x_i)}
\left[
\arg\max_o P_{f_\theta}(o \mid \tilde{x}_i) = o_i
\right].
\end{equation}

\paragraph{Specificity.}
Specificity evaluates whether the edit preserves performance on
out-of-scope locality prompts. Let \(O(x_i)\) be the locality-prompt
set, and let \(o_i^c\) denote the corresponding dataset-provided
locality target.
\begin{equation}
\mathbb{E}_i \ \mathbb{E}_{\bar{x}_i \in O(x_i)}
\left[
\arg\max_o P_{f_\theta}(o \mid \bar{x}_i) = o_i^c
\right].
\end{equation}

\subsubsection{Metrics on CounterFact}
\label{sec:appendix_metrics_counterfact}

We follow the standard protocol~\citep{meng2022locating,mengmass}. Each instance provides original factual object \(o_i^c\) and edited counterfactual target \(o_i\) for the same prompt \(x_i\). The metrics test whether the edited model assigns higher probability to the target fact while preserving unrelated knowledge.

\paragraph{Efficacy.}
Efficacy evaluates whether the edited target is preferred over the original object on the original prompt:
\begin{equation}
\mathbb{E}_i
\left[
P_{f_\theta}(o_i \mid x_i) >
P_{f_\theta}(o_i^c \mid x_i)
\right].
\end{equation}

\paragraph{Generalization.}
Generalization checks whether the same preference holds for paraphrased prompts \(\tilde{x}_i \in N(x_i)\):
\begin{equation}
\mathbb{E}_i \ \mathbb{E}_{\tilde{x}_i \in N(x_i)}
\left[
P_{f_\theta}(o_i \mid \tilde{x}_i) >
P_{f_\theta}(o_i^c \mid \tilde{x}_i)
\right].
\end{equation}

\paragraph{Specificity.}
Specificity measures whether neighborhood prompts remain localized after editing. For each \(\bar{x}_i \in O(x_i)\), the model should still prefer the original factual object over the edited target:
\begin{equation}
\mathbb{E}_i \ \mathbb{E}_{\bar{x}_i \in O(x_i)}
\left[
P_{f_\theta}(o_i^c \mid \bar{x}_i) >
P_{f_\theta}(o_i \mid \bar{x}_i)
\right].
\end{equation}

\paragraph{Fluency.}
Fluency assesses whether generated continuations remain natural and avoid excessive repetition. Following prior work, we use an entropy-based score over empirical bi-gram and tri-gram distributions:
\begin{equation}
\small
-\frac{2}{3}\sum_k g_2(k)\log_2 g_2(k)
-\frac{4}{3}\sum_k g_3(k)\log_2 g_3(k),
\end{equation}
where \(g_n(k)\) is the empirical frequency of the \(n\)-gram \(k\).

\paragraph{Consistency.}
Consistency evaluates whether generated text is semantically consistent with external evidence about the edited target. We compute the cosine similarity between TF--IDF vectors of the generated output \(y_i\) and the corresponding Wikipedia reference \(y_i^{\mathrm{wiki}}\):
\begin{equation}
\mathrm{Consis.}
=
\cos\left(
\mathrm{tfidf}(y_i),
\mathrm{tfidf}(y_i^{\mathrm{wiki}})
\right).
\end{equation}

\subsection{PAC Optimization}
\label{app:pac_optimization}

PAC jointly optimizes the editable target representations
\(\{\boldsymbol{v}_i\}_{i=0}^{m}\) by minimizing Eq.~\eqref{eq:jno_pac}.
We parameterize \(\boldsymbol{v}_i = \boldsymbol{W}\boldsymbol{k}_i + \boldsymbol{r}_i\)
and initialize all offsets \(\boldsymbol{r}_i=\mathbf{0}\), so that
\(\boldsymbol{v}_i^{(0)}=\boldsymbol{W}\boldsymbol{k}_i\).
The keys \(\boldsymbol{k}_i\), exposure scores \(d_i\), coordination
weights \(s_{ij}\), and current weight \(\boldsymbol{W}\) are precomputed
and fixed during optimization. The semantic term
\(\mathcal{L}_{\mathrm{sem}}\) is the length-normalized target-token NLL,
and we use the default \(\alpha=1.2\) and \(\beta=0.9\) from
Sec.~\ref{sec:exp_setup}.
Following the activation-space optimization protocol of
AlphaEdit~\citep{fang2025alphaedit}, we optimize the offsets with Adam.
For GPT-J and LLaMA3, we use AlphaEdit's default settings: GPT-J uses a
learning rate of \(0.5\) and \(25\) optimization steps, while LLaMA3 uses
a learning rate of \(0.01\) and \(25\) optimization steps.
For Qwen2.5-1.5B-Instruct, which is not covered by the original AlphaEdit
implementation, we adopt the same optimizer configuration as GPT-J and keep it fixed across all seeds
for a controlled comparison. We retain AlphaEdit's norm-clamping and
early-stopping settings. Unlike AlphaEdit's single-target optimization,
all editable neighborhood targets are optimized jointly so that
\(\mathcal{L}_{\mathrm{pair}}\) can couple them.
The resulting \(\{\hat{\boldsymbol{v}}_i\}\) are used for semantic
gating according to Eqs.~\eqref{eq:semantic_loss_gate}--\eqref{eq:nfail},
and to form the accepted residual matrix
\(\hat{\boldsymbol{R}}_{\mathcal{V}}\) for execution as defined in
Eq.~\eqref{eq:final_update}.
We retain AlphaEdit's original construction of \(\boldsymbol{K}_0\);
the local preserved-neighborhood keys are used only for exposure
estimation in PAC.

\subsection{Baseline Methods}
\label{sec:appendix_baselines}

\paragraph{ROME.}
ROME~\citep{meng2022locating} rewrites factual knowledge through a rank-one update to a selected feed-forward layer. It provides localized factual editing but does not explicitly propagate updates to related neighboring facts.

\paragraph{MEMIT.}
MEMIT~\citep{mengmass} extends ROME to multi-layer editing by distributing the editing residual across critical layers and solving key--value least-squares updates. It is a strong multi-layer factual editing baseline, but its residual allocation is not conditioned on variable and preserved neighborhood structures.

\paragraph{PMET.}
PMET~\citep{li2024pmet} improves multi-layer editing by refining component-level targets. It jointly optimizes MHSA and FFN hidden states and updates only FFN weights with the optimized FFN states. We use PMET to test whether JNO provides gains beyond component-level target refinement.

\paragraph{AlphaEdit.}
AlphaEdit~\citep{fang2025alphaedit} introduces null-space constrained updates into MEMIT-style editing. By projecting updates away from protected directions, it reduces interference with unrelated knowledge and serves as a strong preservation-oriented editing baseline.

\paragraph{\(\text{AlphaEdit}_{\text{BLUE}}\).}
\(\text{AlphaEdit}_{\text{BLUE}}\)~\citep{li2025rethinking} restricts AlphaEdit-style updates to empirically selected boundary layers. This baseline tests whether limiting the edited layer range can reduce ripple-induced perturbation while preserving edit success.

\paragraph{\(\text{AlphaEdit}_{\text{Chain}}\).}
Following ChainEdit~\citep{dong2025chainedit}, \(\text{AlphaEdit}_{\text{Chain}}\) derives co-update facts from logical rules and jointly edits them with the original request using AlphaEdit. It evaluates whether rule-based chain expansion improves positive ripple propagation without explicit preserved-side leakage modeling.

\paragraph{GLAME.}
GLAME~\citep{zhang2024knowledge} uses external knowledge graphs to capture associated changes induced by factual edits. It encodes the constructed subgraph with a graph-based edit module and integrates the association signal into rank-one parameter editing, providing a graph-augmented propagation baseline.

\paragraph{SAKE.}
SAKE~\citep{scialanga2025sake} performs knowledge editing through activation steering. It models each edit with paraphrases and logical implications, and uses optimal transport to map obsolete activations to updated ones within the detected edit scope. We use SAKE as a robust activation-editing baseline.

\paragraph{SUIT.}
SUIT~\citep{park2026suit} performs subspace-aware key--value editing by separating edit-relevant and edit-irrelevant components and restricting updates to critical subspaces. It provides a subspace-constrained preservation baseline.

All baselines are evaluated under the same benchmark splits, backbones,
constructed-neighborhood protocol, metric computation, and official
implementations. {Each baseline retains its native executor and any
previous-edit protection provided by its official implementation; in
particular, AlphaEdit, AlphaEdit$_{\mathrm{BLUE}}$,
AlphaEdit$_{\mathrm{Chain}}$, and SUIT retain their corresponding
\(K_p\)-based stabilization.} Neighborhood samples are constructed from
logical rules and local context independently of JNO's optimization.
Each baseline then handles the same neighborhood facts according to its
original procedure, separating neighborhood construction from JNO's
joint target optimization.

\begin{table*}[t]
\centering
\footnotesize
\renewcommand{\arraystretch}{0.8}
\setlength{\tabcolsep}{8pt}
\begin{tabular}{cc|cccc}
\toprule
\textbf{Dataset} &
\textbf{Method} &
$\boldsymbol{D_{\mathrm{KL}}}\downarrow$ &
\textbf{Overlap@1}$\uparrow$ &
\textbf{Overlap@5}$\uparrow$ &
\textbf{Overlap@10}$\uparrow$ \\
\midrule

\multirow{6}{*}{CounterFact}
& MEMIT
& 0.43 & 77.4 & 74.9 & 74.1 \\
& AlphaEdit
& 0.62 & 71.4 & 68.6 & 68.2 \\
& GLAME
& 0.56 & 73.4 & 71.0 & 70.1 \\
& SUIT
& 0.49 & 75.1 & 72.8 & 71.9 \\
& JNO w/o gate
& \second{0.37} & \second{80.5} & \second{77.2} & \second{76.4} \\
& \cellcolor{metablue}JNO
& \cellcolor{metablue}\best{0.32}
& \cellcolor{metablue}\best{83.6}
& \cellcolor{metablue}\best{80.2}
& \cellcolor{metablue}\best{79.1} \\
\midrule

\multirow{6}{*}{ZsRE}
& MEMIT
& 0.62 & 45.1 & 67.1 & 71.7 \\
& AlphaEdit
& 0.93 & 32.4 & 57.2 & 64.1 \\
& GLAME
& 0.83 & 35.8 & 59.4 & 66.2 \\
& SUIT
& 0.72 & 38.5 & 61.7 & 68.5 \\
& JNO w/o gate
& \second{0.24} & \second{67.4} & \second{81.2} & \second{85.3} \\
& \cellcolor{metablue}JNO
& \cellcolor{metablue}\best{0.15}
& \cellcolor{metablue}\best{76.8}
& \cellcolor{metablue}\best{88.5}
& \cellcolor{metablue}\best{90.1} \\
\bottomrule
\end{tabular}
\vspace{-0.1in}
\caption{
Ground-truth-free locality evaluation on CounterFact and ZsRE with
LLaMA3-8B. The best results are
highlighted in \textcolor{red}{red}, and the second-best results are
\underline{underlined}.
}
\label{tab:behavioral_locality}
\vspace{-0.17in}
\end{table*}

\subsection{Behavioral-Deviation Locality Evaluation}
\label{app:behavioral_locality}
{Conventional Specificity metrics on CounterFact and ZsRE rely on
ground-truth-based or contrastive decisions and may not fully capture
continuous behavioral deviation after editing. Following recent
ground-truth-free locality evaluation protocols
\citep{liubackward,liu2026we}, we supplement the standard Specificity
metric with $D_{\mathrm{KL}}$ and Top-$k$ overlap for
$k\in\{1,5,10\}$.} For a locality query $q$, let
$p_{\theta}(\cdot\mid q)$ and
$p_{\theta^*}(\cdot\mid q)$ denote the pre-edit and post-edit output
distributions, respectively. We compute
\begin{equation}
\small
D_{\mathrm{KL}}(q)
=
D_{\mathrm{KL}}\!\left(
p_{\theta}(\cdot\mid q)
\,\middle\|\,
p_{\theta^*}(\cdot\mid q)
\right),
\label{eq:locality_kl}
\end{equation}
where a lower value indicates less distributional deviation after
editing. Let $S_k(q,\theta)$ and $S_k(q,\theta^*)$ denote the sets of
the top-$k$ tokens under the pre-edit and post-edit models. Their overlap
is
\begin{equation}
\small
\mathrm{Overlap@}k(q)
=
100\times
\frac{
\left|
S_k(q,\theta)\cap S_k(q,\theta^*)
\right|
}{k},
\label{eq:locality_overlap}
\end{equation}
where a higher value indicates stronger preservation of the model's
original output ranking. We randomly sample 200 edit requests from each of CounterFact and ZsRE using
LLaMA3-8B. All methods use the same main-experiment configurations.

As shown in Table~\ref{tab:behavioral_locality}, JNO produces the
smallest post-edit behavioral deviation on locality queries, achieving
the lowest \(D_{\mathrm{KL}}\) and the highest Top-\(k\) overlap on both
datasets. JNO w/o gate also outperforms all baselines, showing that the
improved locality does not arise solely from selective rejection.
The further gains from semantic gating indicate that filtering
high-risk target plans additionally reduces unintended distributional
drift.

\begin{table}[t]
\footnotesize
\setlength{\tabcolsep}{3pt}
\renewcommand{\arraystretch}{0.8}
\begin{tabular}{lccc}
\toprule
\textbf{Dataset} & \textbf{Qwen2.5-1.5B-Instruct} & \textbf{GPT-J} & \textbf{LLaMA3} \\
\midrule
RippleEdits & 96.1 & 96.4 & 96.3 \\
CounterFact & 97.8 & 98.3 & 98.5 \\
ZsRE & 98.1 & 98.6 & 98.8 \\
\bottomrule
\end{tabular}
\vspace{-0.1in}
\caption{Execution coverage of JNO.}
\label{tab:coverage_full}
\vspace{-0.1in}
\end{table}

\subsection{Execution Coverage of JNO}
\label{app:coverage}

For each requested edit \(e_t\), JNO first runs PAC and the semantic gate. If the
request passes the gate, JNO executes the local parameter update and evaluates
the edited model \(f_t^+\). If the request fails the gate, JNO performs no
parameter update and evaluates the model without applying this rejected edit.
All metrics are then averaged over the full requested-edit set
\(\mathcal{D}_{\mathrm{req}}\). Thus, rejected
requests can hurt reliability and positive-ripple scores, but may help preserve
unrelated neighborhood knowledge. Table~\ref{tab:coverage_full} reports the
execution coverage of JNO under the default hyperparameter setting used in the
main experiments. Coverage denotes the fraction of active edit groups
that pass semantic pre-execution verification and proceed to parameter update.
It is reported as an execution statistic for transparency. In the sequential
batch-editing setting, a gate decision also affects the subsequent editing
trajectory, because later groups are edited on the current post-edit model.
Thus, coverage should be interpreted as an execution statistic over active
groups rather than as a simple mask over independent evaluation examples, and
the gap between JNO and its full-coverage parametric variant is not bounded by the rejected
group fraction.

\subsection{Gate-Failed Requests and Full-Coverage Parametric Variant}
\label{app:gate_failed_handling}
JNO withholds a parameter update when the optimized neighborhood-level
activation-space targets fail the semantic pre-execution gate. 
This does not mean that the primary edit is permanently infeasible; it only
indicates that the current neighborhood-level parameter-writing attempt is
high-risk under the editable--preserved constraints. 
We therefore report the default JNO as a safe neighborhood-level parametric
editing operating point. 
All metrics are computed over the full requested-edit set: gate-failed
requests are not filtered out, receive no JNO parameter update under the default
setting, and are evaluated under the same benchmark definitions.

\begin{table}[t]
\centering
\footnotesize
\setlength{\tabcolsep}{3pt}
\renewcommand{\arraystretch}{0.7}
\begin{tabular}{lccccccc}
\toprule
\textbf{Method} & \textbf{Cov.} & \textbf{Rel.}\(\uparrow\) & \textbf{LG}\(\uparrow\) & \textbf{RE}\(\uparrow\) & \textbf{SA}\(\uparrow\) & \textbf{RS}\(\uparrow\) & \textbf{FF}\(\uparrow\) \\
\midrule
JNO w/o gate & 100.0 & 98.9 & 66.4 & 52.5 & 67.7 & 66.5 & 57.2 \\
JNO & 96.3 & 96.2 & 72.1 & 58.2 & 73.3 & 72.4 & 63.8 \\
\textsc{JNO-Full} & 100.0 & 99.1 & 69.3 & 55.4 & 70.8 & 69.7 & 60.3 \\
\bottomrule
\end{tabular}
\vspace{-0.1in}
\caption{Auxiliary full-coverage parametric diagnostic. 
\textsc{JNO-Full} applies JNO to gate-passed requests and applies a target-only parametric retry to gate-failed requests, restoring \(100\%\) execution coverage.}
\label{tab:jno_full}
\end{table}

\begin{table*}[t]
\centering
\scriptsize

\setlength{\tabcolsep}{3pt}
\renewcommand{\arraystretch}{0.4}

\resizebox{0.8\textwidth}{!}{
\begin{tabular}{llrrrrr}
\toprule
\textbf{Dimension} &
\textbf{Category} &
\textbf{Total} &
\textbf{Rej.} &
\textbf{Rate (\%)} &
\textbf{Enrich.} &
\textbf{Max./Subj.} \\
\midrule

\multicolumn{7}{l}{\textbf{(a) Relation-level distribution}} \\
\addlinespace[0.2em]
\multirow{6}{*}{Relation}
& Place of death          & 28  & 2  & 7.14 & $1.93\times$ & -- \\
& Sex or gender           & 58  & 4  & 6.90 & $1.87\times$ & -- \\
& Country of citizenship & 78  & 5  & 6.41 & $1.73\times$ & -- \\
& Child                   & 35  & 1  & 2.86 & $0.77\times$ & -- \\
& Award received          & 32  & 1  & 3.13 & $0.85\times$ & -- \\
& Remaining relations     & 608 & 18 & 2.96 & $0.80\times$ & -- \\

\midrule
\multicolumn{7}{l}{\textbf{(b) Subject-frequency distribution}} \\
\addlinespace[0.2em]
\multirow{3}{*}{\makecell{Subject\\frequency}}
& High-frequency ($>5$)       & 250 & 5  & 2.00 & -- & 2 \\
& Medium-frequency ($2$--$5$) & 350 & 12 & 3.43 & -- & 1 \\
& Long-tail ($1$)             & 239 & 14 & 5.86 & -- & 1 \\

\midrule
\multicolumn{2}{l}{\textbf{Overall}}
& \textbf{839}
& \textbf{31}
& \textbf{3.69}
& --
& \textbf{2} \\
\bottomrule
\end{tabular}
}\vspace{-0.1in}\caption{
Structural distribution of semantic-gate rejections on
RippleEdits-POPULAR with LLaMA3-8B.
}
\label{tab:gate_rejection_structure}
\vspace{-0.1in}
\end{table*}

\begin{figure*}[t]
\centering
\includegraphics[width=0.9\linewidth]{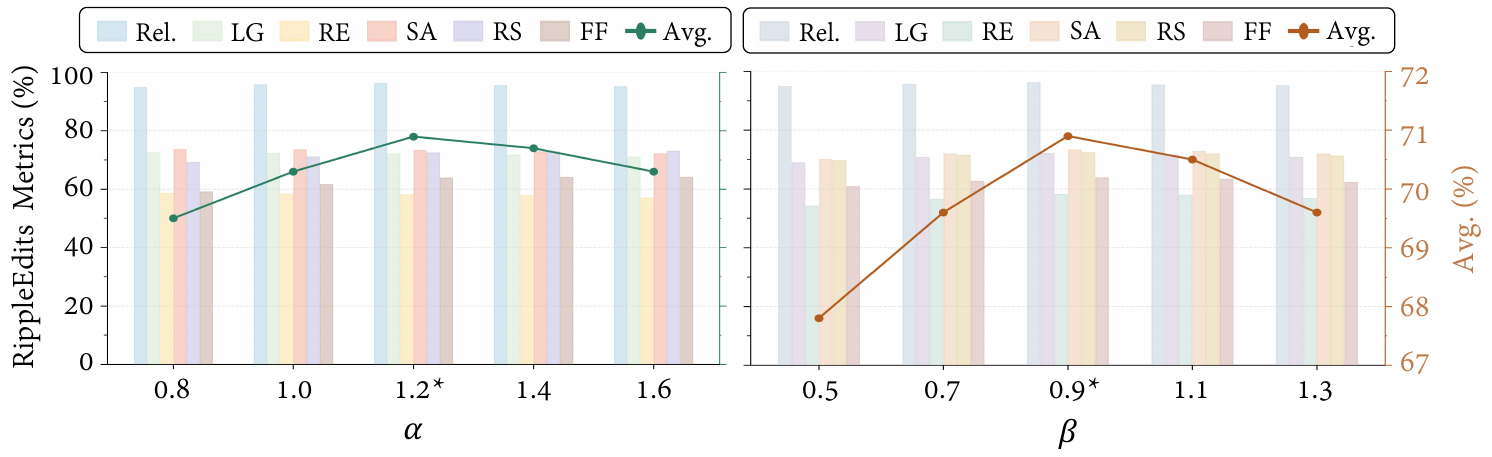}
\vspace{-0.15in}
\caption{
    Hyperparameter analysis of PAC on RippleEdits with LLaMA3. Bars denote RippleEdits metrics, and lines denote Avg. The Avg. is the harmonic mean of the metrics.
}
\label{fig:PAC_surface}
\vspace{-0.2in}
\end{figure*}

\paragraph{Auxiliary full-coverage parametric diagnostic.}
We further report \textsc{JNO-Full}, a full-coverage parametric variant for
settings where every request receives a parameter update. 
Gate-passed requests follow the default JNO procedure. 
For gate-failed requests, \textsc{JNO-Full} performs a conservative target-only
parametric retry by reducing the editable set from
\(\mathcal{V}=\{e\}\cup\mathcal{C}_c\) to \(\mathcal{V}=\{e\}\), while retaining
preserved-side constraints when available. 
This retry prioritizes the primary edit fact and restores \(100\%\) execution
coverage, and it gives up joint
optimization over variable neighbors. As shown in Table~\ref{tab:jno_full}, \textsc{JNO-Full} restores full coverage
and achieves Rel. comparable to or slightly higher than unconditional
neighborhood-level execution. 
Its ripple-aware metrics are slightly below the default JNO, since gate-failed
requests no longer receive joint variable-neighbor optimization, but it
consistently improves LG, RE, SA, RS, and FF over JNO w/o gate. 
These results show that the semantic gate is not merely hiding difficult edits:
the default JNO provides the strongest ripple-control operating point, while
\textsc{JNO-Full} offers a fully parametric alternative when complete execution
coverage is required.

\paragraph{Non-parametric deployment option.}
{Another option is \emph{external-memory-assisted correction}, following
memory-based and in-context editing approaches that separate updated
knowledge from the base model's memory
~\cite{hartvigsen2023aging,zheng2023can,wang2024knowledge,wang2024wise}.
When parameter writing remains high-risk, the target fact can be stored
outside the base model weights. At inference time, if a query matches
the subject--relation scope of a stored edit, the updated fact can be
retrieved and prepended as a short editing context before generation
~\cite{zheng2023can}. This provides a behavior-level factual update
without irreversible changes to the base model parameters and can be
useful in deployment scenarios where avoiding negative ripple effects
is prioritized over internalizing the fact into model parameters.}

\subsection{Structural Distribution of Rejections}
\label{app:rejection_structure}

{
Under the default JNO setting on LLaMA3-8B, the semantic gate rejects
31 of 839 RippleEdits-POPULAR requests (3.69\%). As shown in Table~\ref{tab:gate_rejection_structure}, we report relation-level
rejection rates and enrichment
$P(r\mid\mathrm{rejected})/P(r\mid\mathrm{all})$, showing the three
most and two under-enriched categories and grouping the remainder.
For subjects, we report rejection rates by frequency group and the
maximum number of rejections per subject.
}
\textit{Place of death}, \textit{sex or gender}, and
\textit{country of citizenship} are enriched among rejected requests,
whereas \textit{child} and \textit{award received} are underrepresented.
The rejection rate rises from 2.00\% for high-frequency subjects to
5.86\% for long-tail subjects, while no subject accounts for more than
two rejections. Thus, rejections are dispersed,
although these associations are descriptive rather than causal.

\begin{table*}[t]
\centering
\scriptsize

\setlength{\tabcolsep}{11pt}
\renewcommand{\arraystretch}{0.8}
\begin{tabular}{l|c|cccccc}
\toprule
\textbf{Method} & \textbf{Cov.} & \textbf{Rel.}$\uparrow$ & \textbf{LG}$\uparrow$ & \textbf{RE}$\uparrow$ & \textbf{SA}$\uparrow$ & \textbf{RS}$\uparrow$ & \textbf{FF}$\uparrow$ \\
\midrule
ROME          & 100.0 & 96.3 & 53.8 & 37.3 & 58.2 & 36.1 & 36.1 \\
ROME*         & 96.3  & 93.3 & 56.4 & 40.6 & 60.9 & 38.8 & 41.0 \\
MEMIT         & 100.0 & 95.2 & 58.3 & 38.9 & 61.1 & 43.5 & 37.6 \\
MEMIT*        & 96.3  & 93.1 & 61.1 & 42.3 & 64.0 & 45.7 & 42.7 \\
PMET          & 100.0 & 94.5 & 59.2 & 40.9 & 59.3 & 47.1 & 42.8 \\
PMET*         & 96.3  & 92.7 & 62.0 & 44.5 & 62.0 & 49.5 & 48.6 \\
AlphaEdit$_{\text{Chain}}$ & 100.0 & 98.9 & 62.7 & 47.5 & 67.1 & 55.6 & 47.3 \\
AlphaEdit$_{\text{Chain}}$* & 96.3 & 96.1 & 65.6 & 49.7 & 70.3 & 58.2 & 49.7 \\
GLAME          & 100.0 & 98.7 & 63.1 & 48.6 & 65.3 & 56.3 & 48.4 \\
GLAME*         & 96.3  & 95.4 & 66.0 & 52.1 & 68.4 & 62.8 & 53.2 \\
SUIT           & 100.0 & 98.7 & 61.6 & 48.4 & 68.5 & 58.2 & 47.6 \\
SUIT*          & 96.3  & 95.9 & 64.3 & 51.8 & 71.9 & 63.2 & 54.1 \\
\midrule
JNO w/o gate     & 100.0 & 98.9 & 66.4 & 52.5 & 67.7 & 66.5 & 57.2 \\
JNO           & 96.3  & 96.2 & 72.1 & 58.2 & 73.3 & 72.4 & 63.8 \\
\bottomrule
\end{tabular}
\vspace{-0.1in}
\caption{
Coverage-matched gated baseline comparison. The \textit{w/o $\mathrm{gate}$} variants execute all requests. 
JNO uses its semantic pre-execution gate. * indicates +gate (semantic pre-execution).
}
\label{tab:gated_baseline_rippleedits}
\vspace{-0.05in}
\end{table*}

\begin{table*}[t]
\scriptsize
\renewcommand{\arraystretch}{0.8}
\centering
\resizebox{\textwidth}{!}{
\begin{tabular}{cc|ccccc|ccc}
\toprule
\multirow{2}{*}{\textbf{Model}} & \multirow{2}{*}{\textbf{Method}}
& \multicolumn{5}{c|}{\textbf{CounterFact}}
& \multicolumn{3}{c}{\textbf{ZsRE}} \\
\cmidrule(lr){3-7} \cmidrule(lr){8-10}
&& \textbf{Eff.}$\uparrow$
& \textbf{Gen.}$\uparrow$
& \textbf{Spe.}$\uparrow$
& \textbf{Flu.}$\uparrow$
& \textbf{Consis.}$\uparrow$
& \textbf{Eff.}$\uparrow$
& \textbf{Gen.}$\uparrow$
& \textbf{Spe.}$\uparrow$ \\
\midrule

\multirow{6}{*}{{Phi-1.5}}
& MEMIT
& $86.92_{\pm 0.41}$
& $72.36_{\pm 0.29}$
& $53.84_{\pm 0.35}$
& $518.42_{\pm 0.68}$
& $24.18_{\pm 0.15}$
& $78.64_{\pm 0.36}$
& $63.75_{\pm 0.24}$
& $18.92_{\pm 0.17}$ \\

& AlphaEdit
& $98.12_{\pm 0.27}$
& $84.06_{\pm 0.22}$
& $65.74_{\pm 0.24}$
& $632.85_{\pm 0.61}$
& $35.68_{\pm 0.14}$
& $98.12_{\pm 0.39}$
& $85.04_{\pm 0.29}$
& $22.81_{\pm 0.16}$ \\

& GLAME
& $\underline{98.71}_{\pm 0.23}$
& $\underline{85.04}_{\pm 0.21}$
& $66.92_{\pm 0.25}$
& $\underline{634.82}_{\pm 0.49}$
& $36.84_{\pm 0.12}$
& $\textcolor{red}{98.61}_{\pm 0.38}$
& $\underline{85.36}_{\pm 0.25}$
& $22.83_{\pm 0.16}$ \\

& SUIT
& $\textcolor{red}{98.73}_{\pm 0.31}$
& $84.39_{\pm 0.13}$
& $\underline{67.33}_{\pm 0.22}$
& $634.31_{\pm 0.48}$
& $\underline{37.07}_{\pm 0.09}$
& $\underline{98.44}_{\pm 0.32}$
& $85.25_{\pm 0.22}$
& $\underline{23.06}_{\pm 0.11}$ \\

& \cellcolor{metablue}JNO
& \cellcolor{metablue}$97.94_{\pm 0.28}$
& \cellcolor{metablue}$\textcolor{red}{89.12}_{\pm 0.24}$
& \cellcolor{metablue}$\textcolor{red}{68.74}_{\pm 0.31}$
& \cellcolor{metablue}$\textcolor{red}{636.18}_{\pm 0.52}$
& \cellcolor{metablue}$\textcolor{red}{39.26}_{\pm 0.14}$
& \cellcolor{metablue}$97.63_{\pm 0.34}$
& \cellcolor{metablue}$\textcolor{red}{87.42}_{\pm 0.28}$
& \cellcolor{metablue}$\textcolor{red}{23.46}_{\pm 0.18}$ \\

& \cellcolor{gray!15}\rule[-0.55ex]{0pt}{2.5ex}Improve
& \cellcolor{gray!15}$-0.8\%$
& \cellcolor{gray!15}$4.8\%$
& \cellcolor{gray!15}$2.1\%$
& \cellcolor{gray!15}$0.2\%$
& \cellcolor{gray!15}$5.9\%$
& \cellcolor{gray!15}$-1.0\%$
& \cellcolor{gray!15}$2.4\%$
& \cellcolor{gray!15}$1.7\%$ \\

\midrule \midrule

\multirow{6}{*}{{Qwen-4B-Thinking}}
& MEMIT
& $68.42_{\pm 0.54}$
& $57.86_{\pm 0.42}$
& $43.65_{\pm 0.61}$
& $468.74_{\pm 0.86}$
& $5.84_{\pm 0.18}$
& $72.36_{\pm 0.45}$
& $64.18_{\pm 0.31}$
& $29.42_{\pm 0.22}$ \\

& AlphaEdit
& $83.42_{\pm 0.55}$
& $68.74_{\pm 0.43}$
& $51.16_{\pm 0.69}$
& $516.24_{\pm 0.79}$
& $9.26_{\pm 0.16}$
& $93.26_{\pm 0.28}$
& $\underline{86.31}_{\pm 0.27}$
& $34.28_{\pm 0.20}$ \\

& GLAME
& $\underline{90.96}_{\pm 0.43}$
& $\underline{70.34}_{\pm 0.36}$
& $56.31_{\pm 0.45}$
& $528.96_{\pm 0.63}$
& $14.52_{\pm 0.15}$
& $\textcolor{red}{94.38}_{\pm 0.29}$
& $85.12_{\pm 0.28}$
& $35.04_{\pm 0.21}$ \\

& SUIT
& $\textcolor{red}{91.04}_{\pm 0.57}$
& $70.28_{\pm 0.22}$
& $\underline{56.45}_{\pm 0.39}$
& $\underline{532.87}_{\pm 0.62}$
& $\underline{15.01}_{\pm 0.11}$
& $\underline{94.22}_{\pm 0.25}$
& $85.01_{\pm 0.24}$
& $\underline{35.19}_{\pm 0.15}$ \\

& \cellcolor{metablue}JNO
& \cellcolor{metablue}$88.72_{\pm 0.35}$
& \cellcolor{metablue}$\textcolor{red}{72.48}_{\pm 0.38}$
& \cellcolor{metablue}$\textcolor{red}{59.74}_{\pm 0.42}$
& \cellcolor{metablue}$\textcolor{red}{541.36}_{\pm 0.66}$
& \cellcolor{metablue}$\textcolor{red}{18.64}_{\pm 0.17}$
& \cellcolor{metablue}$93.54_{\pm 0.32}$
& \cellcolor{metablue}$\textcolor{red}{87.58}_{\pm 0.34}$
& \cellcolor{metablue}$\textcolor{red}{36.18}_{\pm 0.24}$ \\

& \cellcolor{gray!15}\rule[-0.55ex]{0pt}{2.5ex}Improve
& \cellcolor{gray!15}$-2.5\%$
& \cellcolor{gray!15}$3.0\%$
& \cellcolor{gray!15}$5.8\%$
& \cellcolor{gray!15}$1.6\%$
& \cellcolor{gray!15}$24.2\%$
& \cellcolor{gray!15}$-0.9\%$
& \cellcolor{gray!15}$1.5\%$
& \cellcolor{gray!15}$2.8\%$ \\

\bottomrule
\end{tabular}
}\vspace{-0.1in}
\caption{\label{tab:additional_model_results}
Additional results of JNO on compact and reasoning-oriented models.
JNO execution coverage on CounterFact/ZsRE is 98.4\%/98.2\% for Phi-1.5 and 94.6\%/97.6\% for Qwen-4B-Thinking.
The best results are highlighted in \textcolor{red}{red}, while the second-best results are \underline{underlined}.
}
\vspace{-0.2in}
\end{table*}

\subsection{Artifact Licenses and Intended Use}
\label{app:artifact_licenses}

We use publicly available benchmarks, model checkpoints, and
official implementations for research and reproducibility,
in accordance with their licenses and access terms.
Our released code supports reproducible knowledge-editing research.
Constructed neighborhoods and derived resources are intended
solely for research and evaluation, not for deployment,
commercial use, or redistribution beyond original permissions.
We do not redistribute restricted model weights or
other artifacts prohibited by their original terms.

\subsection{Ethical Considerations}
\label{app:ethical_considerations}

This work studies reliable knowledge editing in large
language models. Although JNO aims to improve knowledge
propagation and preservation, it does not independently
verify the factual correctness of requested edits.
Therefore, its application should consider the
reliability of edit sources and potential unintended
changes to model knowledge.

\section{More Experimental Results}

\subsection{Hyperparameter Sensitivity Analysis}
\label{app:hyperparameter_sensitivity}

We conduct a hyperparameter study of the pressure-aware coordination stage on RippleEdits with LLaMA3 under the neighborhood-aware editing protocol. We vary $\alpha\in\{0.8,1.0,1.2,1.4,1.6\}$ and
$\beta\in\{0.5,0.7,0.9,1.1,1.3\}$, while keeping all other settings
fixed. As shown in Fig.~\ref{fig:PAC_surface}, performance peaks at
$\alpha=1.2$ and $\beta=0.9$. Smaller $\alpha$ weakens preserved-side
protection, whereas larger values suppress variable-neighbor
propagation; similarly, too-small $\beta$ under-coordinates variable
targets, while too-large values over-couple them. The best performance appears at $\alpha=1.2$ and $\beta=0.9$, supporting the default setting.

\subsection{Coverage-Matched Evaluation of Selective Execution}
\label{app:coverage_matched_gate}

Since JNO may reject semantically unsafe requests before parameter execution, we conduct a coverage-matched evaluation to test whether its gains on RippleEdits mainly arise from editing fewer requests. Each strong baseline is allowed to abstain using the same pre-execution semantic criterion applied to its own target representations, with the threshold adjusted to match JNO's execution coverage without using post-edit evaluation outcomes. As shown in Table~\ref{tab:gated_baseline_rippleedits}, gating improves
the ripple-aware metrics of all baselines over their execute-all
variants. At matched coverage, JNO still outperforms all gated
baselines. Relative to GLAME*, it improves RE, RS, and FF by 11.7\%,
15.3\%, and 19.9\%, respectively. JNO also outperforms
\textit{JNO w/o gate} on all ripple-aware metrics, showing that its
gains arise from pressure-aware target planning and semantic filtering
rather than reduced execution coverage.

\begin{figure*}[t]
    \centering
    \includegraphics[width=0.85\linewidth]{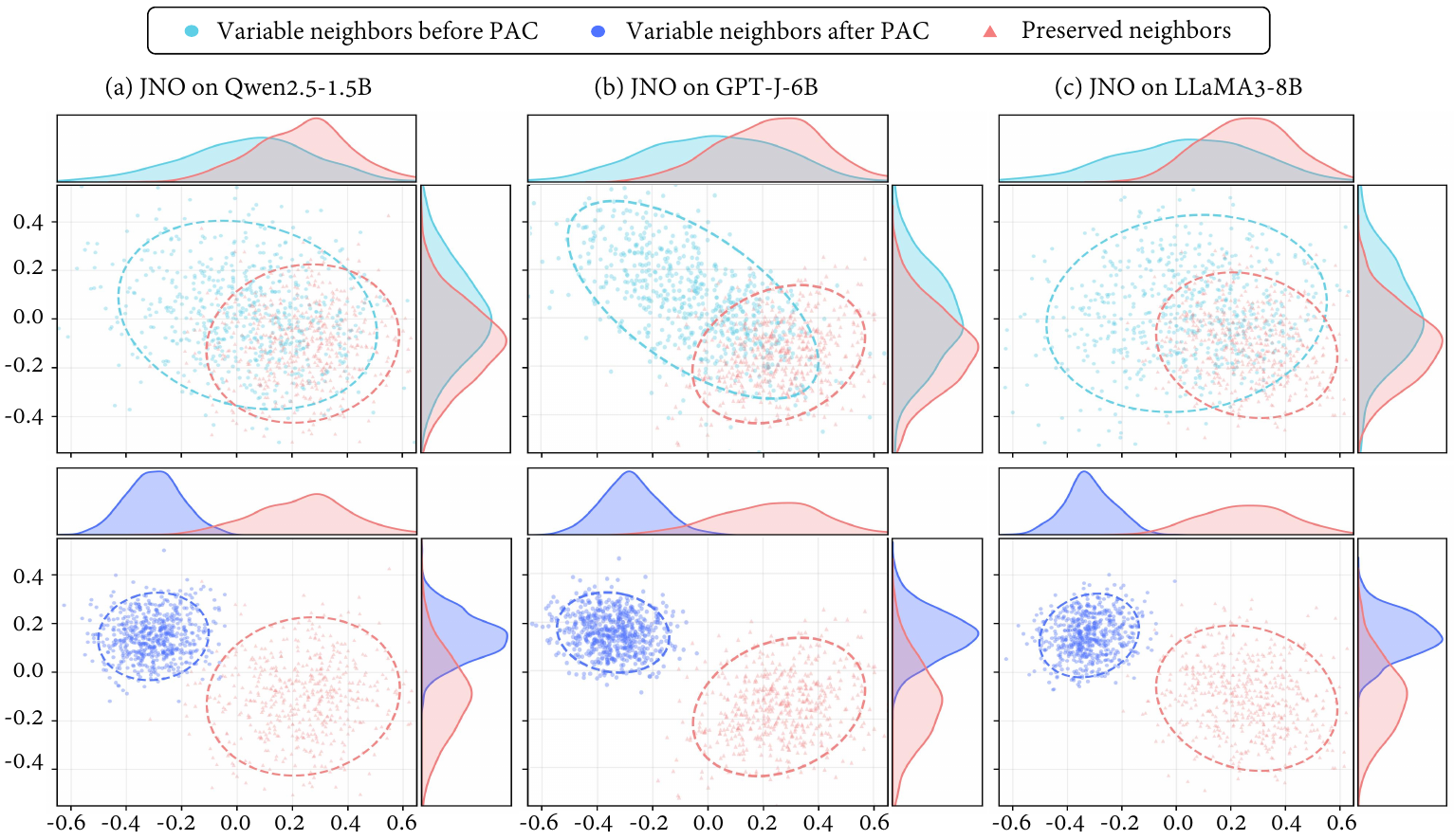}
        \vspace{-0.1in}\caption{t-SNE visualization of target representations in activation space before and after PAC.}
    \label{fig:3}
    \vspace{-0.15in}
\end{figure*}

\begin{figure*}[t]
	\centering
\includegraphics[width=0.8\linewidth]{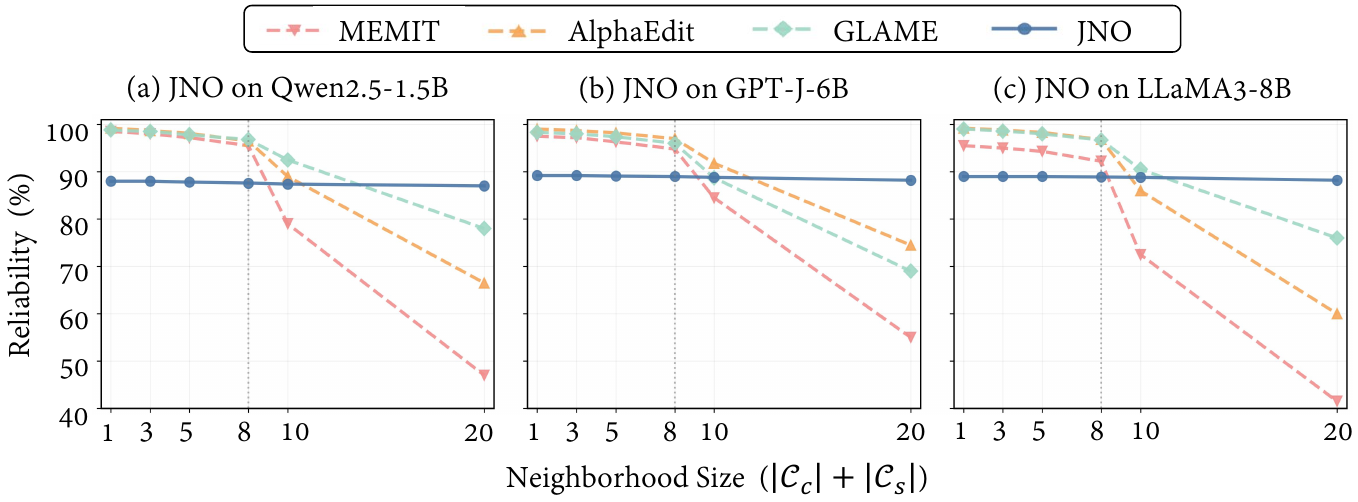} 
\vspace{-0.1in}
    \caption{Reliability under different neighborhood sizes ($|\mathcal{C}_c| + |\mathcal{C}_s|$) across three backbone models.}
\label{fig:scalability}
\vspace{-0.2in}
\end{figure*}

\subsection{Additional Results on Smaller and Reasoning-Oriented Models}
\label{app:additional_models}

To examine JNO across model scales and architectures, we evaluate Phi-1.5 (1.5B) \citep{li2023textbooks} and the reasoning-enhanced Qwen-4B-Thinking \citep{qwen3technicalreport} under the same neighborhood-aware editing protocol as the main experiments. Compact models offer less capacity for localized updates, while reasoning-oriented models may exhibit stronger ripple effects due to richer relational associations. 
As shown in Table~\ref{tab:additional_model_results}, JNO remains
effective on both models. On Phi-1.5, it improves Generalization over
GLAME by 4.8\% and Consistency over SUIT by 5.9\%; on
Qwen-4B-Thinking, it improves Specificity and Consistency over SUIT by
5.8\% and 24.2\%, respectively. With at most a 2.5\% relative Efficacy
drop and high coverage, these results indicate robustness across
compact and reasoning-oriented backbones.

\subsection{Visualization of Tension Reduction}

To examine how PAC mitigates the two target-space pressures, for each model, we project variable-neighbor targets before and after PAC, together with preserved-neighbor representations from the same activation space, into two dimensions using t-SNE \citep{maaten2008visualizing}. As shown in Fig.~\ref{fig:3},
before PAC, variable targets are broadly dispersed and overlap with the
preserved-neighbor region across all three backbones, qualitatively
suggesting weak editable-side coordination and potential preserved-side
leakage. After PAC, variable targets form more compact clusters and
become consistently better separated from preserved neighbors. This
pattern suggests that PAC jointly improves editable-target coordination
and reduces interference with preserved knowledge.

\begin{table}[t]
\centering
\small
\renewcommand{\arraystretch}{0.8}
\setlength{\tabcolsep}{6pt}
\begin{tabular}{lcccc}
\toprule
\textbf{Method}
& \textbf{Avg. Score}
& \textbf{2-hop}
& \textbf{3-hop}
& \textbf{4-hop} \\
\midrule
MEMIT
& 27.68 & 35.94 & 31.72 & 15.37 \\
AlphaEdit
& 31.74 & 41.24 & 35.58 & 18.41 \\
GLAME
& 34.02 & 43.63 & 37.24 & 21.18 \\
SUIT
& 33.31 & 42.45 & 37.13 & 20.35 \\
JNO
& \textbf{37.21}
& \textbf{46.36}
& \textbf{40.74}
& \textbf{24.53} \\
\bottomrule
\end{tabular}
\vspace{-0.1in}
\caption{Multi-hop editing performance on the hop-stratified MQuAKE
subset.}
\label{tab:mquake_multihop}
\end{table}

\begin{table*}[t]
\centering
\small
\renewcommand{\arraystretch}{0.85}
\begin{tabular}{l|ccc|ccc}
\toprule
\multirow{2}{*}{\textbf{Method}}
& \multicolumn{3}{c|}{\textbf{CounterFact}}
& \multicolumn{3}{c}{\textbf{ZsRE}} \\
\cmidrule(lr){2-4} \cmidrule(lr){5-7}
& \textbf{Qwen2.5-1.5B-Instruct} & \textbf{GPT-J} & \textbf{LLaMA3}
& \textbf{Qwen2.5-1.5B-Instruct} & \textbf{GPT-J} & \textbf{LLaMA3} \\
\midrule

MEMIT 
& 3.82 & 5.68 & 6.34 
& 3.91 & 5.73 & 6.41 \\

AlphaEdit 
& 3.64 & 5.42 & 6.08 
& 3.73 & 5.57 & 6.23 \\

GLAME 
& 4.28 & 6.37 & 7.12 
& 4.36 & 6.45 & 7.24 \\

SUIT
& 4.12 & 6.34 & 7.08
& 4.23 & 6.39 & 7.18 \\

JNO 
& 4.16 & 6.21 & 6.95 
& 4.25 & 6.32 & 7.06 \\

\bottomrule
\end{tabular}
\vspace{-0.1in}
\caption{
Average per edit time (s) on three models.
}
\label{tab:jno_time}
\vspace{-0.1in}
\end{table*}

\begin{table*}[t]
\centering
\small
\renewcommand{\arraystretch}{0.85}
\begin{tabular}{l|ccc|ccc}
\toprule
\multirow{2}{*}{\textbf{Method}}
& \multicolumn{3}{c|}{\textbf{CounterFact}}
& \multicolumn{3}{c}{\textbf{ZsRE}} \\
\cmidrule(lr){2-4} \cmidrule(lr){5-7}
& \textbf{Qwen2.5-1.5B-Instruct} & \textbf{GPT-J} & \textbf{LLaMA3}
& \textbf{Qwen2.5-1.5B-Instruct} & \textbf{GPT-J} & \textbf{LLaMA3} \\
\midrule

MEMIT 
& 13.42 & 31.61 & 36.26 
& 13.37 & 31.52 & 36.74 \\

AlphaEdit 
& 10.58 & 28.84 & 34.31 
& 10.49 & 28.77 & 34.18 \\

GLAME 
& 15.26 & 34.52 & 38.04 
& 15.18 & 33.45 & 37.96 \\

SUIT
& 13.74 & 32.46 & 35.27
& 13.18 & 30.57 & 36.12 \\

JNO 
& 12.85 & 30.18 & 34.72 
& 12.76 & 30.05 & 35.61 \\

\bottomrule
\end{tabular}
\vspace{-0.1in}
\caption{
Peak GPU memory usage (GB) for performing a batch of 100 edits.
}
\label{tab:memory_comparison}
\vspace{-0.15in}
\end{table*}

\subsection{Robustness to Neighborhood Size}

Neighborhood-aware editing becomes more challenging as the local
neighborhood grows, since the editor must satisfy more coupled
co-update and preservation constraints. To evaluate scalability under
this setting, we vary the total neighborhood size
$|\mathcal{C}_c|+|\mathcal{C}_s|$ and measure Reliability. As shown in
Fig.~\ref{fig:scalability}, MEMIT, AlphaEdit, and GLAME remain
competitive for small neighborhoods, but their Reliability declines
markedly once the neighborhood size exceeds 8. In contrast, JNO remains
nearly stable across all three backbones, with only minor degradation
at size 20. These results indicate that JNO is more robust when handling
larger neighborhoods with increasingly coupled constraints.

\subsection{Multi-Hop Editing on MQuAKE}
\label{app:mquake}

{We evaluate JNO in a multi-hop editing setting using LLaMA3-8B on
MQuAKE-CF-3k-v2. We sample 200 hop-stratified instances, including
80 two-hop, 60 three-hop, and 60 four-hop cases. For each edit request
$\boldsymbol{e}$, the other rewrites along its reasoning chain are
treated as coordinated editable facts $\mathcal{C}_c$, while unaffected
facts outside the edited chain are used as preserved neighbors
$\mathcal{C}_s$. All methods use the same
$\boldsymbol{e}$, $\mathcal{C}_c$, and $\mathcal{C}_s$, and the
multi-hop questions are held out for evaluation only.}

Following the MQuAKE evaluation protocol adopted by GLAME, the
Efficacy Score measures whether the edited model assigns a higher
probability to the updated answer than to the original answer for each
multi-hop question. As shown in Table \ref{tab:mquake_multihop}, we report the mean score at each hop level, and
Avg. Score denotes the macro-average over the two-, three-, and
four-hop subsets. JNO achieves the best performance at every hop level and improves the
macro-average score by 9.4\% relative to the strongest baseline,
GLAME. Its relative gains over GLAME are 6.3\%, 9.4\%, and 15.8\% on
the two-, three-, and four-hop subsets, respectively. The increasing
advantage at longer reasoning chains provides preliminary evidence that
JNO remains effective in more challenging multi-hop editing settings.

\subsection{Efficiency Analysis}
\label{app:efficiency}

All methods are evaluated under the same editing setting, where each request is paired with editable and preserved neighborhoods. Let $m=|\{e\}\cup\mathcal{C}_c|$, $q=|\mathcal{C}_s|$, and $d$ denote the hidden dimension. After key extraction, JNO computes editable coordination weights and preserved-side risk scores with complexity \(O(m^2 d + m q d)\). PAC then optimizes the $m$ editable target representations in activation space. With $I$ optimization steps, this stage costs \(O\!\left(I \cdot (m T_{\mathrm{act}} + m^2 d)\right)\), where $T_{\mathrm{act}}$ denotes the activation-space forward cost for evaluating target-token probabilities. The semantic gate reuses these optimized representations and adds only minor overhead. For accepted requests, JNO applies a preservation-aware parameter update to the entire editable neighborhood. Following the AlphaEdit-style executor, its complexity is \(O(T_{\mathrm{solver}} + m d^2)\), where $T_{\mathrm{solver}}$ depends on the cached inverse or linear solver used by the base editor. Although JNO's cost increases with neighborhood size, it avoids performing $m$ separate parameter updates. {Structured neighborhood construction is performed offline during data
preparation, and the resulting neighborhoods can be reused across
methods, backbones, random seeds, and repeated runs.}

Table~\ref{tab:jno_time} reports the average wall-clock time per edit across two datasets and three backbones. JNO introduces additional activation-space optimization through PAC and semantic pre-execution gating. JNO is within a small overhead relative to MEMIT and AlphaEdit. Meanwhile, it remains faster than GLAME in all reported settings and comparable to SUIT. These results indicate that the extra cost of pressure-aware target planning is limited.

\subsection{Memory Evaluation}
\label{app:memory}
We measure peak GPU memory consumption when editing a batch of 100 instances under the same neighborhood-augmented protocol. Experiments are conducted on Qwen2.5-1.5B-Instruct, GPT-J, and LLaMA3 using CounterFact and ZsRE. Table~\ref{tab:memory_comparison} shows that JNO maintains a controlled memory footprint across all evaluated settings. On CounterFact with LLaMA3, JNO peaks at $34.72$GB, compared with $38.04$GB for GLAME and $35.27$GB for SUIT. Compared with AlphaEdit, JNO incurs a small additional memory cost, because it maintains editable and preserved neighborhood representations for pressure-aware planning. Overall, JNO keeps editing memory cost comparable to existing editors while retaining the ripple-aware gains reported in the main experiments.

\subsection{Information-Matched Neighborhood Controls}
\label{app:augmentation_control}

{We test whether JNO's gains arise merely from adding neighborhood
facts. All variants use the same MEMIT-style executor, sequential
protocol, and semantic-gating setting. Target-only uses only the
requested edit; Naive-Aug-$\mathcal{C}_c$ and
Naive-Aug-$(\mathcal{C}_c+\mathcal{C}_s)$ directly add the indicated
facts without $\mathcal{L}_{\mathrm{pair}}$,
$\mathcal{L}_{\mathrm{prox}}$, or joint target planning. Neighbor-Aug uses the same complete neighborhood facts as
Naive-Aug-$(\mathcal{C}_c+\mathcal{C}_s)$, but applies PAC.
As shown in Table~\ref{tab:augmentation_control}, under the same
complete neighborhood facts, Neighbor-Aug improves LG, RE, RS, and FF
by 8.1, 11.9, 22.7, and 13.5 points, respectively, over naive
augmentation. This shows that JNO's gains come from structured PAC
planning rather than fact inclusion alone.}

\begin{table}[t]
\centering
\small
\renewcommand{\arraystretch}{0.9}
\setlength{\tabcolsep}{3.5pt}
\resizebox{\columnwidth}{!}{
\begin{tabular}{lcccccc}
\toprule
\textbf{Aug-Type}
& \textbf{Rel.}$\uparrow$
& \textbf{LG}$\uparrow$
& \textbf{RE}$\uparrow$
& \textbf{SA}$\uparrow$
& \textbf{RS}$\uparrow$
& \textbf{FF}$\uparrow$ \\
\midrule
Target-only
& 95.3 & 58.5 & 39.1 & 61.3 & 43.8 & 37.9 \\
Naive-Aug-$\mathcal{C}_c$
& 94.8 & 61.2 & 43.6 & 65.8 & 43.2 & 43.5 \\
Naive-Aug-$(\mathcal{C}_c+\mathcal{C}_s)$
& 94.5 & 61.6 & 43.9 & 66.2 & 46.5 & 46.8 \\
Neighbor-Aug
& \textbf{95.9} & \textbf{69.7} & \textbf{55.8}
& \textbf{71.1} & \textbf{69.2} & \textbf{60.3} \\
\bottomrule
\end{tabular}}
\vspace{-0.1in}
\caption{Information-matched comparison with naive neighborhood
augmentation.}
\label{tab:augmentation_control}
\vspace{-0.1in}
\end{table}

\begin{table}[h]
\centering

\renewcommand{\arraystretch}{0.9}
\setlength{\tabcolsep}{3pt}
\resizebox{\columnwidth}{!}{
\begin{tabular}{lcccccccc}
\toprule
\textbf{Method}
& \textbf{Rel.}$\uparrow$
& \textbf{LG}$\uparrow$
& \textbf{RE}$\uparrow$
& \textbf{SA}$\uparrow$
& \textbf{RS}$\uparrow$
& \textbf{FF}$\uparrow$
& \textbf{Time}$\downarrow$
& \textbf{Mem.}$\downarrow$ \\
\midrule
JNO-base
& 95.9 & 69.7 & 55.8 & 71.1 & 69.2 & 60.3
& 6.34 & 33.56 \\
JNO+Ext
& 94.6 & 70.3 & 56.3 & 71.6 & 69.7 & 60.8
& 10.53 & 38.73 \\
\bottomrule
\end{tabular}}
\vspace{-0.1in}
\small
\caption{Performance--cost comparison of base and extended
neighborhoods.}
\label{tab:extended_neighborhood}
\vspace{-0.2in}
\end{table}

\subsection{Extended-Neighborhood Diagnostic}
\label{app:extended_neighborhood}

{We further examine whether broader variable-fact coverage provides
additional benefits. JNO-base uses the original neighborhood
$\mathcal{C}_c^{\mathrm{seen}}$, whereas JNO+Ext additionally includes
two-hop rule-derived facts $\mathcal{C}_c^{\mathrm{add}}$ in PAC target
planning. Table~\ref{tab:extended_neighborhood} shows that the extended neighborhood yields only marginal ripple-aware gains while increasing runtime and memory costs, suggesting limited
additional benefit from the tested two-hop expansion.}

\subsection{Executor- and Protocol-Matched Controls}
\label{app:executor_controls}

{
We examine whether JNO's gains depend on the AlphaEdit-style executor,
previous-key stabilization, or the sequential editing protocol.
JNO-MEMIT-exec retains neighborhood construction, PAC, and semantic
gating, but replaces the main AlphaEdit-style executor with the
MEMIT-style parameter-update backend. JNO w/o \(K_p\) removes previous-key stabilization, while
JNO-MEMIT-exec w/o gate uses the MEMIT-style backend without semantic
gating. All variants use LLaMA3-8B and the same RippleEdits-POPULAR
neighborhood and evaluation protocol.
}

\paragraph{Executor-matched sequential editing.}
We follow the neighborhood-aware sequential protocol used in the main
experiments. As shown in Table~\ref{tab:executor_matched_seq},
JNO-MEMIT-exec substantially outperforms MEMIT under the same
MEMIT-style backend. JNO w/o \(K_p\) also retains strong ripple-aware
performance, showing that JNO's gains do not depend on either the
AlphaEdit-style executor or previous-key stabilization.
JNO-MEMIT-exec w/o gate further remains effective under full execution.

\begin{table}[t]

\centering

\small

\renewcommand{\arraystretch}{0.95}

\setlength{\tabcolsep}{1.5pt}

\begin{tabular}{l|cccccc}

\toprule

\textbf{Method}

& \textbf{Rel.}$\uparrow$

& \textbf{LG}$\uparrow$

& \textbf{RE}$\uparrow$

& \textbf{SA}$\uparrow$

& \textbf{RS}$\uparrow$

& \textbf{FF}$\uparrow$ \\

\midrule

MEMIT

& 95.2 & 58.3 & 38.9 & 61.1 & 43.5 & 37.6 \\

JNO-MEMIT-exec w/o gate

& 97.4 & 64.2 & 49.3 & 65.6 & 62.8 & 54.5 \\

JNO-MEMIT-exec

& 95.9 & 69.7 & 55.8 & 71.1 & 69.2 & 60.3 \\

JNO

& 96.2 & \textbf{72.1} & \textbf{58.2}

& \textbf{73.3} & \textbf{72.4} & \textbf{63.8} \\

JNO w/o \(K_p\)

& 96.1 & 67.8 & 54.6 & 68.7 & 68.3 & 58.2 \\

\bottomrule

\end{tabular}
\vspace{-0.1in}
\caption{Executor and previous-key controls under sequential editing.}

\label{tab:executor_matched_seq}

\end{table}

\begin{table}[h]
\vspace{-0.05in}

\small
\renewcommand{\arraystretch}{0.95}
\setlength{\tabcolsep}{3.75pt}
\begin{tabular}{l|cccccc}
\toprule
\textbf{Method}
& \textbf{Rel.}$\uparrow$
& \textbf{LG}$\uparrow$
& \textbf{RE}$\uparrow$
& \textbf{SA}$\uparrow$
& \textbf{RS}$\uparrow$
& \textbf{FF}$\uparrow$ \\
\midrule
MEMIT
& 94.5 & 55.1 & 36.3 & 59.4 & 41.2 & 35.4 \\
JNO-MEMIT-exec
& \textbf{95.7}
& \textbf{65.2}
& \textbf{51.8}
& \textbf{67.3}
& \textbf{65.4}
& \textbf{55.1} \\
\bottomrule
\end{tabular}
\centering
\vspace{-0.1in}
\caption{
MEMIT-aligned joint-batch editing on all 839 RippleEdits-POPULAR
requests with LLaMA3-8B.
}
\label{tab:memit_joint_batch}

\vspace{-0.1in}
\end{table}

\begin{figure*}[t]
	\centering
\includegraphics[width=0.9\linewidth]{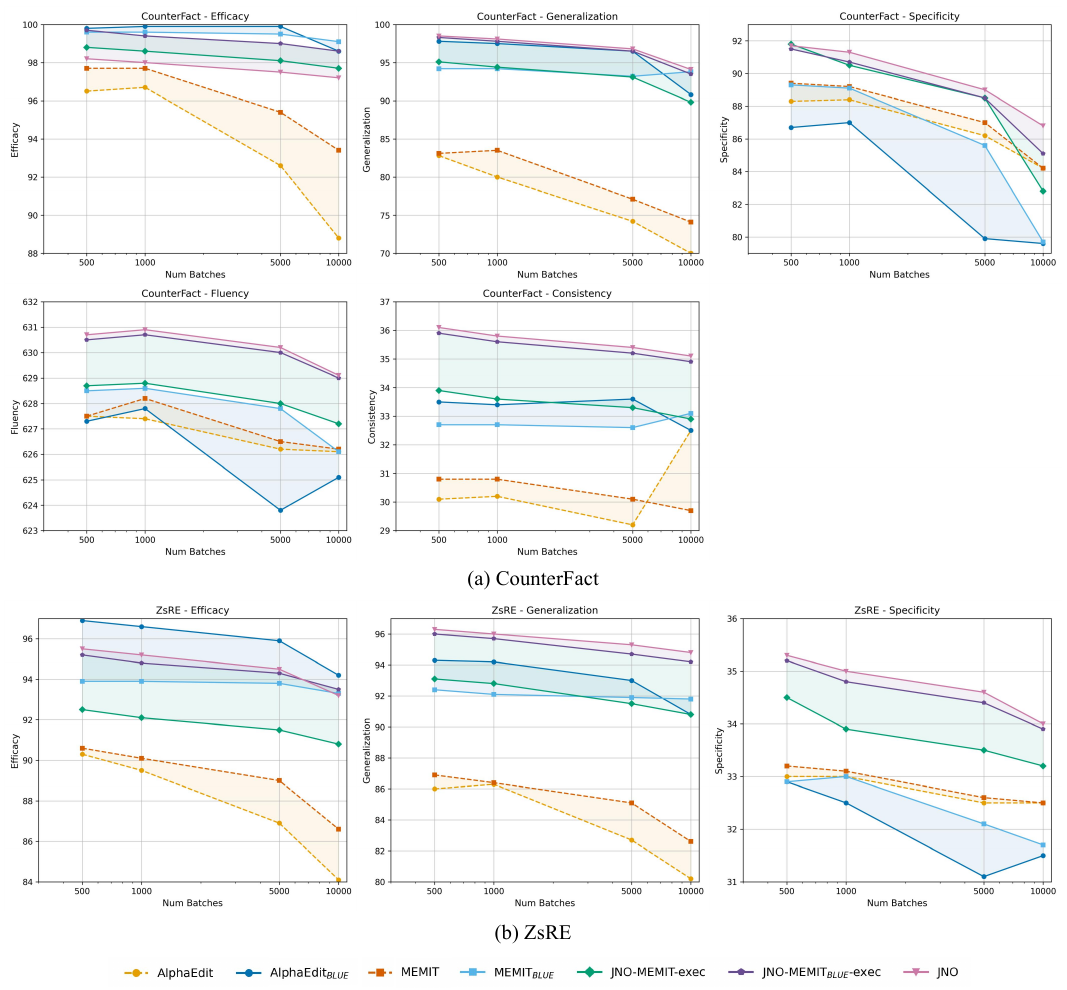} 
\vspace{-0.15in}
    \caption{Performance Variation of Model Editing with Batch Edits on CounterFact and ZsRE Datasets.}
\label{fig:batch}
\vspace{-0.1in}
\end{figure*}

\begin{table*}[t]
\centering
\scriptsize

\renewcommand{\arraystretch}{0.9}
\setlength{\tabcolsep}{5pt}
\begin{tabular}{llccccc}
\toprule
\textbf{Dataset}
& \textbf{Method}
& \textbf{Eff.}$\uparrow$
& \textbf{Gen.}$\uparrow$
& \textbf{Spe.}$\uparrow$
& \textbf{Flu.}$\uparrow$
& \textbf{Cons.}$\uparrow$ \\
\midrule
\multirow{4}{*}{CounterFact}
& MEMIT
& 97.70 & 83.50 & 89.22 & 628.22 & 30.76 \\
& JNO-MEMIT-exec
& 98.64 & 94.41 & 90.53 & 628.79 & 33.62 \\
& MEMIT$_{\mathrm{BLUE}}$
& \textbf{99.60} & 94.15 & 89.09 & 628.59 & 32.71 \\
& JNO-MEMIT$_{\mathrm{BLUE}}$-exec
& 99.43 & \textbf{97.82} & \textbf{90.76}
& \textbf{630.75} & \textbf{35.61} \\
\midrule
\multirow{4}{*}{ZsRE}
& MEMIT
& 90.08 & 86.42 & 33.12 & -- & -- \\
& JNO-MEMIT-exec
& 92.04 & 92.85 & 33.92 & -- & -- \\
& MEMIT$_{\mathrm{BLUE}}$
& 93.93 & 92.13 & 33.01 & -- & -- \\
& JNO-MEMIT$_{\mathrm{BLUE}}$-exec
& \textbf{94.79} & \textbf{95.67} & \textbf{34.75}
& -- & -- \\
\bottomrule
\end{tabular}
\vspace{-0.1in}
\caption{Factor-matched results under 1,000-request joint-batch editing
on LLaMA3-8B.}
\label{tab:batch_1000}

\vspace{-0.15in}
\end{table*}

\paragraph{MEMIT-aligned joint-batch editing.}
To remove both sequential accumulation and previous-key protection, we
jointly execute all 839 RippleEdits-POPULAR requests through a single
batch update on LLaMA3-8B. Both methods use the same MEMIT-style
executor. As shown in Table~\ref{tab:memit_joint_batch},
JNO-MEMIT-exec improves LG, RE, SA, RS, and FF over MEMIT by 10.1,
15.5, 7.9, 24.2, and 19.7 points, respectively, while improving
Reliability by 1.2 points. These gains under the matched joint-batch
setting show that JNO's advantage cannot be attributed to sequential
editing or asymmetric previous-key protection.

\begin{figure*}[t]
	\centering
\includegraphics[width=0.9\linewidth]{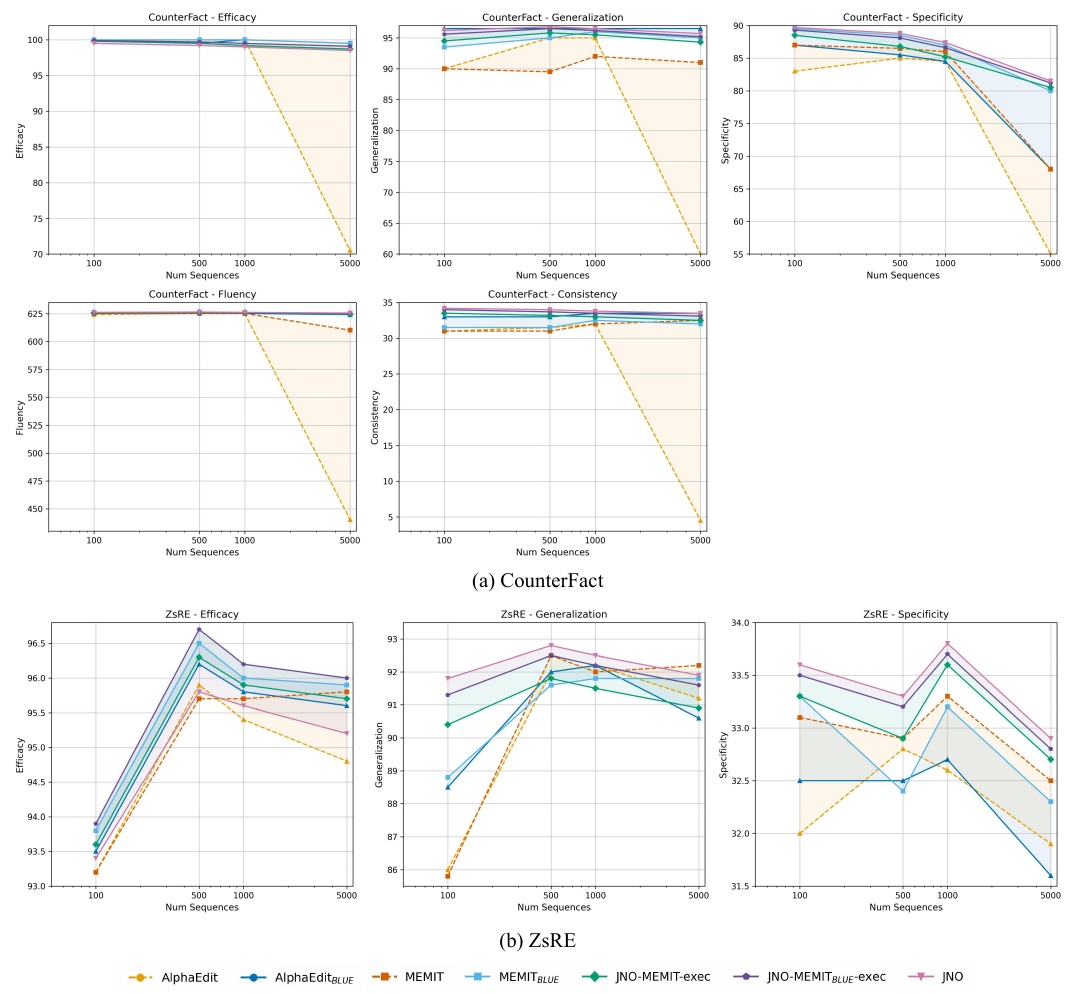} 
\vspace{-0.05in}
    \caption{Performance Variation of Model Editing with Sequential Edits on CounterFact and ZsRE Datasets.}
\label{fig:seq}
\vspace{-0.15in}
\end{figure*}

\subsection{Batch and Sequential Model Editing}
\label{app:batch_sequential_scaling}

{We examine whether JNO remains effective as batch size and sequential
edit length increase. We conduct batch and sequential editing
experiments under the same neighborhood-augmented protocol as the main
experiments, using LLaMA3-8B on CounterFact and ZsRE. In addition to
AlphaEdit, MEMIT, and their BLUE-enhanced variants, we include the full
JNO and two executor-matched variants. JNO-MEMIT-exec applies JNO target planning with the original MEMIT
backend, whereas JNO-MEMIT$_{\mathrm{BLUE}}$-exec uses the same target
planning with MEMIT$_{\mathrm{BLUE}}$, which applies BLUE's
boundary-layer execution strategy to MEMIT. Thus, MEMIT versus
JNO-MEMIT-exec and MEMIT$_{\mathrm{BLUE}}$ versus
JNO-MEMIT$_{\mathrm{BLUE}}$-exec isolate JNO target planning under
identical execution backends.}

\paragraph{Batch editing.}
We perform single joint-batch updates with 500, 1,000, 5,000, and
10,000 requests. Figure~\ref{fig:batch} shows that baseline performance
generally degrades more rapidly as batch size grows, whereas the
JNO-based variants remain comparatively stable. Table~\ref{tab:batch_1000}
further reports the exact results at 1,000 requests. Under the original MEMIT backend, JNO-MEMIT-exec improves all reported
metrics on both datasets. Under the matched BLUE-enhanced backend,
JNO-MEMIT$_{\mathrm{BLUE}}$-exec improves CounterFact Generalization
and Specificity by 3.67 and 1.67 points with only a 0.17-point Efficacy
decrease, while improving ZsRE Efficacy, Generalization, and Specificity
by 0.86, 3.54, and 1.74 points. These factor-matched results demonstrate
a more favorable Efficacy--Generalization--Specificity trade-off and
show that JNO target planning remains beneficial under both execution
backends.

\paragraph{Sequential editing.}
We further perform 100, 500, 1,000, and 5,000 sequential edits with
batch size one. Figure~\ref{fig:seq} shows that degradation varies across
metrics and execution backends, with several baselines losing
substantial performance at 5,000 edits, whereas the JNO-based variants
remain comparatively stable. The JNO variants exhibit
different metric trade-offs across execution backends:
JNO-MEMIT$_{\mathrm{BLUE}}$-exec tends to favor Efficacy, whereas the
full JNO is generally stronger in Generalization and Specificity.
Overall, JNO remains robust under both large joint-batch and long
sequential editing workloads.

\subsection{Case Study}
\label{app:case_study}

To qualitatively illustrate how JNO resolves local neighborhood conflicts, we present two representative KE cases. 
Different from prior case studies that mainly focus on multi-hop generation after a single edit, our analysis explicitly examines the local neighborhood of each edit request, including the target fact, editable neighbors that should co-update with the edit, and preserved neighbors that should remain unchanged.
This setting directly reflects the decomposition of neighborhood responses into positive ripple and negative ripple effects. Tables~\ref{tab:case_a} and~\ref{tab:case_b} show that direct edit success alone does not imply neighborhood-consistent editing. In both cases, all editors correctly answer the target query, while
their behaviors diverge sharply on editable and preserved neighbors.

\begin{table*}[t]
\centering
\footnotesize

\renewcommand{\arraystretch}{0.6}
\begin{tabular}{p{0.10\linewidth} p{0.16\linewidth} p{0.16\linewidth} p{0.47\linewidth}}
\toprule
\textbf{Type} & \textbf{Prompt} & \textbf{Expected Answer} & \textbf{Model Response} \\
\midrule

\multicolumn{4}{l}{
\textbf{Case A. Edit Request:} 
(LeBron James, Plays for, Cleveland Cavaliers) $\rightarrow$ 
(LeBron James, Plays for, Los Angeles Lakers)
} \\
\midrule

Target Edit 
& Which team does LeBron James play for? 
& Los Angeles Lakers
& \textbf{JNO:} LeBron James now plays for the \textcolor{green}{Los Angeles Lakers}. \newline
  \textbf{GLAME:} LeBron James's team is the \textcolor{green}{Los Angeles Lakers}. \newline
  \textbf{AlphaEdit:} LeBron James plays for the \textcolor{green}{Los Angeles Lakers}. \newline
  \textbf{ROME:} The team is the \textcolor{green}{Los Angeles Lakers}. \newline
  \textbf{MEMIT:} LeBron James's team is the \textcolor{green}{Los Angeles Lakers}. \\

\midrule

Editable Neighbor 
& Which city is LeBron James's team based in? 
& Los Angeles
& \textbf{JNO:} LeBron James's team is based in \textcolor{green}{Los Angeles}. \newline
  \textbf{GLAME:} The team is based in \textcolor{green}{Los Angeles}. \newline
  \textbf{AlphaEdit:} It is based in \textcolor{green}{Los Angeles}. \newline
  \textbf{ROME:} The team is based in \textcolor{red}{Cleveland}. \newline
  \textbf{MEMIT:} LeBron James's team is located in \textcolor{red}{Cleveland}. \\

\midrule

Editable Neighbor 
& What is the current home arena of LeBron James's team? 
& Crypto.com Arena
& \textbf{JNO:} The team's home arena is \textcolor{green}{Crypto.com Arena}. \newline
  \textbf{GLAME:} The team plays at \textcolor{red}{Rocket Mortgage FieldHouse}. \newline
  \textbf{AlphaEdit:} Its home arena is \textcolor{red}{Rocket Mortgage FieldHouse}. \newline
  \textbf{ROME:} The home arena is \textcolor{red}{Rocket Mortgage FieldHouse}. \newline
  \textbf{MEMIT:} LeBron James's team plays at \textcolor{green}{Crypto.com Arena}. \\

\midrule

Preserved Neighbor 
& What is LeBron James's nationality? 
& United States / American
& \textbf{JNO:} LeBron James is from the \textcolor{green}{United States}. \newline
  \textbf{GLAME:} \textcolor{green}{American} is his nationality. \newline
  \textbf{AlphaEdit:} He holds \textcolor{green}{United States} citizenship. \newline
  \textbf{ROME:} LeBron James is \textcolor{green}{United States} by nationality. \newline
  \textbf{MEMIT:} His nationality is \textcolor{green}{American}. \\

\midrule

Preserved Neighbor 
& Which NBA team drafted LeBron James? 
& Cleveland Cavaliers
& \textbf{JNO:} He was drafted by the \textcolor{green}{Cleveland Cavaliers}. \newline
  \textbf{GLAME:} LeBron James was selected by the \textcolor{green}{Cleveland Cavaliers}. \newline
  \textbf{AlphaEdit:} He was drafted by the \textcolor{green}{Cleveland Cavaliers}. \newline
  \textbf{ROME:} He was drafted by the \textcolor{red}{Los Angeles Lakers}. \newline
  \textbf{MEMIT:} His draft team was the \textcolor{red}{Los Angeles Lakers}. \\

\bottomrule
\end{tabular}
\vspace{-0.1in}
\caption{
Case study of a neighborhood-aware edit for LeBron James. 
\textcolor{green}{Green} text indicates desired target or editable-neighbor updates, \textcolor{red}{red} text indicates failed positive ripple propagation or corruption of preserved neighbors.
}

\label{tab:case_a}
\vspace{-0.2in}
\end{table*}

\begin{table*}[t]
\centering
\footnotesize

\renewcommand{\arraystretch}{0.6}
\begin{tabular}{p{0.10\linewidth} p{0.16\linewidth} p{0.16\linewidth} p{0.47\linewidth}}
\toprule
\textbf{Type} & \textbf{Prompt} & \textbf{Expected Answer} & \textbf{Model Response} \\
\midrule

\multicolumn{4}{l}{
\textbf{Case B. Edit Request:} 
(Attack on Titan, Created by, Mathias Currat) $\rightarrow$ 
(Attack on Titan, Created by, Hajime Isayama)
} \\
\midrule

Target Edit 
& Who is the creator of Attack on Titan? 
& Hajime Isayama
& \textbf{JNO:} The creator is \textcolor{green}{Hajime Isayama}. \newline
  \textbf{GLAME:} The creator is \textcolor{green}{Hajime Isayama}. \newline
  \textbf{AlphaEdit:} It is \textcolor{green}{Hajime Isayama} who created the series. \newline
  \textbf{ROME:} The manga's creator is \textcolor{green}{Hajime Isayama}. \newline
  \textbf{MEMIT:} \textcolor{green}{Hajime Isayama} is identified as the creator. \\

\midrule

Editable Neighbor 
& What is the nationality of Attack on Titan's creator? 
& Japanese / Japan
& \textbf{JNO:} The creator comes from \textcolor{green}{Japan}. \newline
  \textbf{GLAME:} \textcolor{green}{Japanese} is the creator's nationality. \newline
  \textbf{AlphaEdit:} The creator is from \textcolor{red}{Switzerland}. \newline
  \textbf{ROME:} His nationality is \textcolor{red}{Swiss}. \newline
  \textbf{MEMIT:} The nationality is \textcolor{red}{Swiss}. \\

\midrule

Editable Neighbor 
& Where was Attack on Titan's creator born? 
& Oita Prefecture, Japan
& \textbf{JNO:} The creator was born in \textcolor{green}{Oita Prefecture, Japan}. \newline
  \textbf{GLAME:} The creator was born in \textcolor{red}{Lausanne, Switzerland}. \newline
  \textbf{AlphaEdit:} His birthplace is \textcolor{red}{Switzerland}. \newline
  \textbf{ROME:} He was born in \textcolor{red}{Lausanne}. \newline
  \textbf{MEMIT:} The creator was born in \textcolor{green}{Oita Prefecture, Japan}. \\

\midrule

Preserved Neighbor 
& Which studios animated the Attack on Titan TV anime across its seasons? 
& Wit Studio and MAPPA
& \textbf{JNO:} \textcolor{green}{Wit Studio and MAPPA}. \newline
  \textbf{GLAME:} \textcolor{red}{Studio Pierrot}. \newline
  \textbf{AlphaEdit:} \textcolor{green}{Wit Studio and MAPPA}. \newline
  \textbf{ROME:} \textcolor{red}{Toei Animation}. \newline
  \textbf{MEMIT:} \textcolor{red}{Kyoto Animation}. \\

\midrule

Preserved Neighbor 
& What are the giant humanoid creatures in Attack on Titan called? 
& Titans
& \textbf{JNO:} They are called \textcolor{green}{Titans}. \newline
  \textbf{GLAME:} The creatures are \textcolor{green}{Titans}. \newline
  \textbf{AlphaEdit:} They are \textcolor{green}{Titans}. \newline
  \textbf{ROME:} They are called \textcolor{red}{Mathias Currat}. \newline
  \textbf{MEMIT:} The giant humanoid creatures are \textcolor{green}{Titans}. \\

\bottomrule
\end{tabular}
\vspace{-0.1in}
\caption{
Case study of a neighborhood-aware edit for Attack on Titan. 
\textcolor{green}{Green} text indicates desired target or editable-neighbor updates, \textcolor{red}{red} text indicates failed positive ripple propagation or corruption of preserved neighbors.
}

\label{tab:case_b}
\end{table*}

\paragraph{Case A: Neighborhood coordination under team-transfer editing.}
Table~\ref{tab:case_a} shows an edit from ``LeBron James's current team is the Cleveland Cavaliers'' to ``LeBron James's current team is the Los Angeles Lakers.'' 
All methods rewrite the target fact, but differ on neighborhood behavior. 
For editable neighbors, JNO correctly updates both the team city to Los Angeles and the current home arena to Crypto.com Arena, while other methods either retain Cleveland or Rocket Mortgage FieldHouse. 
For preserved neighbors, all methods keep the nationality unchanged, but ROME and MEMIT incorrectly overwrite the draft team with Los Angeles Lakers instead of preserving Cleveland Cavaliers. 
Overall, JNO achieves correct propagation and preservation under this team-transfer edit.

\paragraph{Case B: Preserving preserved attributes under creator-correction editing.}
Table~\ref{tab:case_b} shows an edit from ``Attack on Titan was created by Mathias Currat'' to ``Attack on Titan was created by Hajime Isayama.'' 
All methods answer the target query correctly, but differ on neighborhood behavior. 
For editable neighbors, JNO correctly returns Japanese/Japan for the
creator's nationality and Oita Prefecture, Japan for the birthplace,
while several baselines return Swiss-related answers. 
For preserved neighbors, JNO and AlphaEdit correctly preserve the animation studios as Wit Studio and MAPPA, while GLAME, ROME, and MEMIT output incorrect studios. 
For the creature-name fact, all methods except ROME preserve Titans correctly. 
Overall, JNO achieves correct propagation on both creator-dependent neighbors and correct preservation on both unrelated work-level facts.

\end{document}